\documentclass[11pt]{article}
\usepackage{authblk}
\providecommand{\keywords}[1]
{
\small	
\textbf{\textit{Keywords: }} #1
}
\usepackage
[
        a4paper,% other options: a3paper, a5paper, etc
        left=2cm,
        right=2cm,
        top=4cm,
        bottom=4cm,
]
{geometry}
\usepackage{microtype}
\usepackage{subcaption}
\usepackage{booktabs} % for professional tables

\usepackage{hyperref}

\usepackage{bm}
\usepackage{amsmath}
\usepackage{amsfonts, verbatim}
\usepackage{amsthm}            % better theorem environments
\usepackage{amssymb}          % no idea
\usepackage{bbm}                 %new
\usepackage{booktabs}          %new
\usepackage{graphicx}
\usepackage{wrapfig}
\usepackage{float}
\usepackage[normalem]{ulem}
\DeclareMathAlphabet{\mathpzc}{OT1}{pzc}{m}{it}
\usepackage[normalem]{ulem} % for underline and strikethrough \sout
\newcommand{\stkout}[1]{\ifmmode\text{\sout{\ensuremath{#1}}}\else\sout{#1}\fi}
\usepackage{mathtools}
\usepackage{color}
\usepackage{setspace}

\usepackage{tabularx,ragged2e,booktabs,caption}
\newcolumntype{C}[1]{>{\Centering}m{#1}}

\usepackage{bbm}                 %new
\usepackage{booktabs}          %new
\usepackage{graphicx}
\usepackage{wrapfig}
\usepackage{float}
\usepackage{algorithm}
\usepackage{algorithmic}
\DeclareMathAlphabet{\mathpzc}{OT1}{pzc}{m}{it}
\usepackage[dvipsnames]{xcolor}
\usepackage[normalem]{ulem} % for underline and strikethrough \sout
\usepackage{setspace}
\usepackage{tabularx,ragged2e,booktabs}
\usepackage[font=small,labelfont=bf]{caption}
\usepackage{enumitem}
\newtheorem{theorem}{Theorem}[section]
\newtheorem{lemma}[theorem]{Lemma}

\newtheorem{corollary}[theorem]{Corollary}

\newtheorem{definition}{Definition}[section]
\newcommand{\mbf}[1]{\boldsymbol{#1}}

\newcommand{\inprod}[1]{\langle #1 \rangle}
\newcommand{\dbinprod}[1]{\langle\hspace{-0.5mm}\langle{#1}\rangle\hspace{-0.5mm}\rangle}
\newcommand{\abs}[1]{\big| #1 \big|}

\newcommand{\norm}[1]{\left\| #1 \right\|}
\newcommand{\rhoTL}{\rho_T^L}
\newcommand{\rhoTLM}{\rho_T^{L,M}}
\newcommand{\R}{\mathbb{R}}
\newcommand{\bo}{\mbf{o}}

\newcommand{\bu}{\mbf{u}}

\newcommand{\bv}{\mbf{v}}

\newcommand{\bw}{\mbf{w}}

\newcommand{\bx}{\mbf{x}}
\newcommand{\bX}{\mbf{X}}
\newcommand{\by}{\mbf{y}}

\newcommand{\bz}{\mbf{z}}

\newcommand{\mB}{\mathcal{B}}

\newcommand{\mD}{\mathcal{D}}
\newcommand{\mE}{\mathcal{E}}

\newcommand{\mH}{\mathcal{H}}
\newcommand{\mI}{\mathcal{I}}
\newcommand{\mK}{\mathcal{K}}
\newcommand{\mL}{\mathcal{L}}
\newcommand{\mM}{\mathcal{M}}
\newcommand{\mN}{\mathcal{N}}

\newcommand{\mU}{\mathcal{U}}
\newcommand{\mX}{\mathcal{X}}

\newcommand{\IR}{R_{\mM}}
\newcommand{\idxcl}{k}
\newcommand{\sdim}{d} % state dimension
\newcommand{\intkernel}{\phi}
\newcommand{\lintkernel}{\widehat{\intkernel}}
\newcommand{\intkernelvar}{\varphi}
\newcommand{\basis}{\psi}
\newcommand{\hypspace}{\mathcal{H}}
\newcommand{\E}{\mathbb{E}}

\newcommand{\probIC}{\mu_0}
\newcommand{\muX}{\probIC(\mM^N)}
\newcommand{\argmin}[1]{\underset{#1}{\operatorname{arg}\operatorname{min}}\;}
\newcommand{\mand}{\quad \text{and} \quad}

\newtheorem{assumption}{Assumption}

\newtheorem{prop}{Proposition}

\newcommand{\bXtl}{\bX_{t_l}}
\newcommand{\dotbXtl}{\dot{\bX}_{t_l}}
\newcommand{\bXmtl}{\bX^{m}_{t_l}}
\newcommand{\dotbXmtl}{\dot{\bX}^{m}_{t_l}}
\title{Learning Interaction Kernels for Agent Systems on Riemannian Manifolds}
\author[a,b]{Mauro Maggioni}
\author[a]{Jason Miller}
\author[a]{Hongda Qiu}
\author[a]{Ming Zhong\footnote{mzhong5@jhu.edu}}
\affil[a]{Department of Applied Mathematics $\&$ Statistics}
\affil[b]{Department of Mathematics}
\affil[ ]{Johns Hopkins University, Baltimore, MD $21218$, USA}
\begin{document}
\maketitle
\begin{abstract}
Interacting agent and particle systems are extensively used to model complex phenomena in science and engineering.  
We consider the problem of learning interaction kernels in these dynamical systems constrained to evolve on Riemannian manifolds from given trajectory data.  
The models we consider are based on interaction kernels depending on pairwise Riemannian distances between agents, with agents interacting locally along the direction of the shortest geodesic connecting them.
We show that our estimators converge at a rate that is independent of the dimension of the state space, and derive bounds on the trajectory estimation error, on the manifold, between the observed and estimated dynamics.
We demonstrate the performance of our estimator on two classical first order interacting systems: Opinion Dynamics and a Predator-Swarm system, with each system constrained on two prototypical manifolds, the $2$-dimensional sphere and the Poincar\'e disk model of hyperbolic space.
\end{abstract}
\newcommand{\theKeywords}{
\small{\begin{tabular}{c c c c c}
Interacting Agent Systems & $|$ & Collective Dynamics & $|$ & Riemannian Geometry \\
Nonparametric Inference  & $|$ & Machine Learning & $|$ & Inverse Problems
\end{tabular}}}
\keywords{\theKeywords} 
\section{Introduction}\label{sec:intro}
Dynamical systems of interacting agents, where ``agents'' may represent atoms, particles, neurons, cells, animals, peoples, robots, planets, etc..., are a fundamental modeling tool in many disciplines including Physics, Biology, Chemistry, Economics and Social Sciences.   It is a fundamental challenge to learn the governing equations of these systems.   Often, agents are either associated with state variables which belong to non-Euclidean spaces, e.g., phase variables considered in various Kuramoto models \cite{kuramoto1975, Strogatz2000}, or constrained to move on non-Euclidean spaces, for example  \cite{ahn2020emergent}.  This has motivated a growing body of research considering interacting agent systems on various manifolds \cite{LLM2018, CLP2014, Sarlette}, including opinion dynamics \cite{OpinionDynamicsAylin2017}, flocking models \cite{ahn2020emergent} and a classical aggregation model \cite{AggregationModelC_Fetecau_2019}.   Further recent approaches for interacting agents on manifolds include \cite{yang2020inference,soize2020probabilistic}.  

In this work,  we offer a nonparametric and inverse-problem-based learning approach to infer the governing structure of interacting agent dynamics, in the form of $\dot\bX = \mbf{f}(\bX)$, constrained on Riemannian manifolds.  Our method is different from others introduced to learn ODEs and PDEs from observations, that aim at inferring $\mbf{f}$, and would be cursed by the high-dimension of the state space of $\bX$. Instead, we exploit the form of the function $\mbf{f}$, special to interacting agent systems, which is determined by an underlying interaction kernel function $\phi$ of one variable only, and learn $\intkernel$, with minimal assumptions on $\intkernel$.  By exploiting invariance of the equations under permutation of the agents as well as the radial symmetry of $\intkernel$, we are able to overcome the curse of dimensionality, while most other approaches (Bayesian, sparse, neural networks) are cursed by the dimension of the state space.  We also demonstrate how our approach can perform transfer learning in section \ref{sec:num_results}.

Let $(\mM, g)$ be a connected, smooth, and geodesically-complete $\sdim$-dimensional Riemannian manifold, with the Riemannian distance denoted by $d_{\mM}$.   Consider $N$ interacting agents, each represented by a state vector $\bx_i(t) \in \mM$.   Their dynamics is governed by the following first order dynamical system, where $\intkernel$, the \textit{interaction kernel}, is the object of our inference: for each $i=1,\ldots,N$,
\begin{equation}\label{eq:newmodel}
\scalebox{0.95}{$
\dot\bx_i(t) = \frac{1}{N}\sum\limits_{i' = 1}^{N}\intkernel(d_{\mM}(\bx_i(t), \bx_{i'}(t)))\bw(\bx_i(t),\bx_{i'}(t))\,.
$}
\end{equation}
Here $\bw(\bz_1, \bz_2)$, for $\bz_1, \bz_2 \in \mM$, is a weight vector pointing in the tangent direction at $\bz_1$ to the shortest geodesic from $\bz_1$ to $\bz_2$. For this to make sense, we restrict our attention to local interactions, e.g. by assuming that $\intkernel$ is compactly supported in a sufficiently small interval $[0,R]$, so that length-minimizing geodesics exist uniquely.  We discuss the well-posedness of this model in greater detail in section \ref{sec:mainmodel}, where we emphasize that this model is derived naturally as a gradient system with a special potential energy depending on pairwise Riemannian distances.  

Our observations consist of states along multiple trajectories, namely $\{\bx_i^{m}(t_l)\}_{i, l, m = 1}^{N, L, M}$ with $L$ being the number of observations made in time and $M$ being the number of trajectories.   We construct an estimator $\lintkernel_{L, M, \mH}$ of $\intkernel$ that is both close to $\intkernel$ in an appropriate $L^2$ sense, and generates a system in the form of \eqref{eq:newmodel} with accurate trajectories when compared to the observed trajectories (generated by $\intkernel$) with the same initial condition.
The estimator, $\lintkernel_{L, M, \mH}$, is defined as the solution to the minimization problem
\[
\lintkernel_{L, M, \mH} = \argmin{\intkernelvar \in \hypspace}\mE_{L, M, \mM}(\intkernelvar)
\]
Here $\hypspace$ is a special function space containing suitable approximations to $\intkernel$ and $\mE_{L, M, \mM}$ is a least squares loss functional built from the trajectory data,  which also takes into account the underlying geometry of $(\mM, g)$. Having established a geometry-based coercivity condition that ensures, among other things, the recoverability of $\intkernel$ by a suitable sequence of $\lintkernel_{L, M, \mH}$'s, our theory shows the convergence rate (in $M$) of our estimator to the true interaction kernel is independent of the dimension of the observation data, i.e. $N\sdim$, and is the same as the minimax rate for $1$-dimensional nonparametric regression:
\[
 \E_{\bX_0 \sim \muX}\Big[\norm{\lintkernel_{L, M, \mH}(\cdot)\cdot - \intkernel(\cdot)\cdot}_{L^2(\rho_{T, \mM}^L)}\Big]  \leq C_1(\mM)  \bigg(\frac{\log M}{M}\bigg)^{\frac{s}{2s + 1}}.  
\]
Here $\bX_0 \in \mM^N$ is an initial system state, $\muX$ is a distribution of initial system states on $\mM^N$, $\rho_{T, \mM}^L$ is a dynamics-adapted probability measure which captures the distribution of pairwise Riemannian distances, and $C_1(\mM)$ is a constant depending the geometry of $\mM$ (see sec. \ref{sec:learningrate}). 

We also establish bounds on the error between the trajectories evolved using our estimators and the true trajectories. Let $\hat\bX_{[0, T]}, \bX_{[0, T]}$ be trajectories evolved with the interaction kernels $\lintkernel_{L, M, \mH}$ and $\intkernel$ respectively, started at the same initial condition, then:
\[
  \E_{\bX_0 \sim \muX}\Big[d_{\text{trj}}(\bX_{[0, T]},\hat\bX_{[0, T]})^2\Big]  \leq C_2(\mM)\norm{\intkernel(\cdot)\cdot - \hat\intkernel_{L, M, \mH}(\cdot)\cdot}_{L^2(\rho_{T, \mM})}^2,
\]
where $d_{\text{trj}}$ is a natural geometry-based distance on trajectories and $C_2(\mM)$ is a constant depending on the manifold's geometry.   As $M$ grows, the norm on the right hand side converges at the  rate above, yielding convergence of the trajectories; full details are given in section \ref{sec:trajanalysis}.

The numerical details of the algorithms for learning the estimator and computing trajectories on manifolds are presented in the Appendix.  The essential differences,  compared to the algorithms presented for Euclidean spaces, are the use of a geometric numerical integrator for computing the evolution of the manifold-constrained dynamics, and that at every time step we need to compute Riemannian inner products of tangent vectors, geodesics and Riemannian distances.  
We demonstrate the performances of our estimators on an opinion dynamics and a predator-swarm model, each constrained on two model spaces: the two dimensional sphere $\mathbb{S}^2$ and the Poincar\'e disk.
\subsection{Connections and Related Work}  
The research on inferring a suitable dynamical system of interacting agents from observation data has been a longstanding problem in science and engineering; see \cite{LLEK2010, KTIHC2011, CMW2014, TW2017} and references therein.  Many recent approaches in machine learning have been developed for inferring general dynamical systems, including multistep methods \cite{keller2019discovery}, optimization \cite{wrobel2013evolutionary}, sparse regression \cite{Brunton3932,RBPK2017,Schaeffer6634}, Bayesian regression \cite{zhang2018robust}, and deep learning \cite{raissi2018multistep,rudy2019deep}. 
In a different direction,  the generalization of traditional machine learning algorithms in Euclidean settings to Riemannian manifolds, and the development of new algorithms designed to work on Riemannian manifolds, has been attracting increasing attention; for example in variational calculus \cite{soize2020probabilistic}, reinforcement learning \cite{DBLP:journals/corr/abs-1803-08501}, deep learning \cite{chen2020doubly} and theoretical CS \cite{montealto2020multiagent}. 
\section{Model Equations}
In this section we introduce the governing equations which we use to model interacting agents constrained on Riemannian manifolds, and discuss the properties of the dynamics.
Table \ref{tab:common_def} shows a list of definitions of the common terms used throughout this paper.
\begin{table}[H]
\vskip-0.25cm
\centering
\small{
\small{\begin{tabular}{ c | l }
\hline
Variable                         & Definition \\
\hline\hline
$(\mM, g)$                       & Riemannian Manifold with metric $g$ \\
\hline
$T_{\bx}\mM$                     & Tangent plane to $\mM$ at $\bx$\\
\hline
$\inprod{\cdot, \cdot}_{g(\bx)}$, $\inprod{\cdot, \cdot}_{g}$ & Inner product on $T_{\bx}\mM$ \\
\hline
$\norm{\bv}_{T_{\bx}\mM}$, $\norm{\bv}_{g}$         & Length of $\bv \in T_{\bx}\mM$ induced by $g(\bx)$\\
\hline
$d_{\mM}(\cdot, \cdot)$          & Riemannian distance induced by $g$\\
\hline
$C^1(\mX)$                       & Set of cont. diff. functions on $\mX \subset \R^d$ \\
\hline
\end{tabular}}  
}
\caption{Notation for first-order models, also see the Appendix.}
\label{tab:common_def} 
\vskip-0.5cm
\end{table}
\subsection{Main model}\label{sec:mainmodel}
In order to motivate the choice of the model equations we use, we begin with a geometric gradient flow model of an interacting agent system.  Consider a system of $N$ interacting agents, with each agent described by a state vector $\bx_i(t)$ on a $\sdim$-dimensional connected, smooth, and geodesically complete Riemannian manifold $\mM$ with metric $g$.  The change of the state vectors seeks to decrease a system energy $E$:
\[
\frac{d\bx_i(t)}{dt} = -\partial_{\bx_i} E(\bx_1(t), \ldots, \bx_N(t)), \qquad i = 1, \ldots, N.
\]
Our first key assumption is that $E$ takes the special form
\[
E(\bx_1(t), \ldots, \bx_N(t)) = \frac{1}{N}\sum_{i' = 1}^N U(d_{\mM}(\bx_i(t), \bx_{i'}(t))^2), 
\]
for some $U:\R^+ \rightarrow \R$ with $U(0) = 0$, and $d_{\mM}(\cdot, \cdot)$ the geodesic distance on $(\mM,g)$.  
Simplifying, and omitting from the notation the dependency on $t$ of $\dot\bx_i$ and $\bx_i$, we obtain the first-order geometric evolution equation, 
\begin{equation}\label{eq:mainmodel}
\dot\bx_i = \frac{1}{N}\sum_{i' = 1}^{N}\intkernel(d_{\mM}(\bx_i, \bx_{i'}))\bw(\bx_i,\bx_{i'}),
\end{equation}
for $i = 1, \ldots, N$.  We call $\intkernel(r) \coloneqq 2U'(r^2)$ the \textit{interaction kernel}.  We have let $\bw(\bz_1, \bz_2) \coloneqq d_{\mM}(\bz_1, \bz_2)\bv(\bz_1, \bz_2)$ for $\bz_1, \bz_2 \in \mM$, with $\bv(\bz_1,\bz_2)$ being, for $\bz_2 \neq \bz_1$, the unit vector (i.e.  $\smash{\norm{\bv}_{T_{\bz_1}\mM} = 1}$) tangent at $\bz_1$ to the minimizing geodesic from $\bz_1$ to $\bz_2$ if $\bz_2$ not in the cut locus of $\bz_1$, and equal to $\mbf{0}$ otherwise.  In order to guarantee existence and uniqueness of a solution for \eqref{eq:mainmodel} over the time interval $[0, T]$, we make a further assumption that $\intkernel$ belongs to the admissible space
\begin{equation*}
\mK_{R, S} \coloneqq \{ \intkernelvar \in C^1([0,R]) \Big| \norm{\intkernelvar}_{L^{\infty}} + \norm{\intkernelvar'}_{L^{\infty}} \leq S \},
\end{equation*}
for some constant $S > 0$.  Here, $R$ is smaller than the global injectivity radius of $\mM$, and $L^\infty= L^{\infty}([0, R])$. 
With this assumption, the possible discountinuity of $\bv(\bz_1,\bz_2)$ due to either $\bz_2\rightarrow \bz_1$ or $\bz_2$ tends to a point in the cut locus of $\bz_1$ is canceled by the multiplication by $d_{\mM}(\bz_1,\bz_2)\rightarrow0$ in the former case, and $\phi(d_{\mM}(\bz_1,\bz_2))\rightarrow0$ in the latter case.  Therefore, the ODE system in \eqref{eq:mainmodel} has a Lipschitz right hand side,  thus it has a unique solution existing for $t \in [0, T]$ \cite{HLW2006}. 

With this geometric gradient flow point of view,  the from of the equations and the radial symmetry of the interaction kernels are naturally pre-determined by the energy potential.  This approach seems to us natural and geometric; for different approaches see \cite{OpinionDynamicsAylin2017, CLP2014}.  Note that in the case of $\mM=\mathbb{R}^{d}$ with the Euclidean metric, we have $d_{\mM}(\bx_i, \bx_{i'}) = \norm{\bx_{i'} - \bx_i}$ and $\bv(\bx_i,\bx_{i'}) = \frac{\bx_{i'} - \bx_i}{\norm{\bx_{i'} - \bx_i}}$,  and we recover the Euclidean space models used in \cite{bongini2016inferring, Lu2019a} and the many works referenced therein.
\section{Learning Framework}
We are given a set of trajectory data of the form $\{\bx_i^{m}(t_l), \dot\bx_i^{m}(t_l)\}_{i, l, m = 1}^{N, L, M}$, for $0 = t_1 < \ldots < t_L = T$, with the initial conditions $\{\bx_i^{m}(0)\}_{i = 1}^N$  being i.i.d from a distribution $\probIC(\mM)$.  The objective is to construct an estimator $\lintkernel_{L, M, \mH}$ of the interaction kernel $\intkernel$.   

Before we describe the construction of our estimator, we introduce some notation.

We let, in $\mM^{N} \coloneqq \mM\times\dots\times\mM$,
\[
\bXmtl \coloneqq \begin{bmatrix} \vdots \\ \bx_i^{m}(t_l) \\ \vdots \end{bmatrix} \quad \text{and} \quad \bX \coloneqq \begin{bmatrix} \vdots \\ \bx_i \\ \vdots \end{bmatrix},
\]
where $(\mM^N, g_{\mM}^{N})$ is the canonical product of Riemannian manifolds with product Riemannian metric given by,
\[
\Bigg\langle\begin{bmatrix} \vdots \\ \bu_i \\ \vdots \end{bmatrix}, \begin{bmatrix} \vdots \\ \bz_i \\ \vdots \end{bmatrix}\Bigg\rangle_{g_{\mM}^{N}(\bX)} \coloneqq \frac{1}{N}\sum_{i = 1}^N\inprod{\bu_i, \bz_i}_{g(\bx_i)},
\]
for $\bu_i, \bz_i \in T_{\bx_i}\mM$.  The initial conditions, $\bX^{m}_0$ are drawn i.i.d. from $\muX$, where $\muX = \probIC(\mM) \times \cdots \times \probIC(\mM)=\probIC(\mM)^N$.   Note that all expectations will be with respect to $\bX_0 \sim \muX$. 
Finally, $\mbf{f}^{\text{c}}_{\intkernel}$ is the vector field on $\mM^N$ (i.e. $\mbf{f}^{\text{c}}_{\intkernel}(\bX) \in T_{\bX}\mM^N$ for $\bX\in\mM^N$), given by
\[
\scalebox{0.825}{$
\mbf{f}^{\text{c}}_{\intkernel}(\bXmtl) \coloneqq \begin{bmatrix} \vdots \\ \frac{1}{N}\sum_{i' = 1}^{N}\intkernel(d_{\mM}(\bx_i^{m}(t_l), \bx_{i'}^{m}(t_l)))\bw(\bx_i^{m}(t_l),\bx_{i'}^{m}(t_l))\\ \vdots \end{bmatrix},
$}
\]
The system of equations \eqref{eq:mainmodel} can then be rewritten, for each $m = 1, \ldots, M$, as $$\dot\bX^{m}_t = \mbf{f}^{\text{c}}_{\intkernel}(\bX^{m}_t)\,.$$
\subsection{Geometric Loss Functionals}
In order to simplify the presentation, we assume that the observation times, i.e. $\{t_l\}_{l = 1}^L$, are equispaced in $[0, T]$ (the general case is similar).  We begin with the definition of the hypothesis space $\mH$, over which we shall minimize an error functional to obtain an estimator of $\intkernel$.
\begin{definition}\label{assump:HS}
An {{admissible hypothesis space}} $\mH$ is a compact (in $L^{\infty}$-norm) and convex subset of $L^2([0,R])$, such that every $\intkernelvar \in \mH$ is bounded above by some constant $S_0 \geq S$, i.e. $\norm{\intkernelvar}_{L^\infty([0,R])}\le S_0$; moreover $\intkernelvar$ is smooth enough to ensure the existence and uniqueness of solutions of \eqref{eq:mainmodel} for $t \in [0, T]$, i.e. $\intkernelvar \in \mH \cap \mK_{R, S_0}$.
\end{definition}
\noindent
For a function $\varphi \in \mH$, we define the loss functional
\begin{equation}\label{error functional}
  \mE_{L, M, \mM}(\intkernelvar) \coloneqq \frac{1}{ML}\sum_{l, m = 1}^{L, M}\norm{\dotbXmtl - \mbf{f}^{\text{c}}_{\intkernelvar}(\bXmtl)}_{g}^2\,,
\end{equation}
where the norm $\norm{\cdot}_{g}$ in $T_{\bXmtl}\mM^N$ can be written as 
\[
\norm{\dotbXmtl - \mbf{f}^{\text{c}}_{\intkernelvar}(\bXmtl)}_{g}^2 = \frac{1}{N}\sum_{i = 1}^N\norm{\dot\bx_{i, t_l}^{m} - \frac{1}{N}\sum_{i' = 1}^N\intkernelvar(r_{ii', t_l}^{m})\bw_{ii', t_l}^{m}}_{T_{\bx_i^{m}(t_l)}\mM}^2,
\]
with $\dot\bx_{i, t_l}^{m} \coloneqq \dot\bx_i^{m}(t_l)$, $r_{ii', t_l}^{m} \coloneqq d_{\mM}(\bx_i^{m}(t_l), \bx_{i'}^{m}(t_l))$, and $\bw_{ii', t_l}^{m} \coloneqq \bw(\bx_i^{m}(t_l), \bx_{i'}^{m}(t_l))$.  This loss functional is nonnegative, and reaches $0$ when $\varphi$ is equal to the (true) interaction kernel $\intkernel$ if $\intkernel$ is also in $\mH$ (i.e. $ \intkernel \in \mH \cap \mK_{R, S}$).   Given that $\mH$ is compact and convex, $\mE_{L, M, \mM}$ is continuous on $\mH$, the minimizer of $\mE_{L, M, \mM}$ exists and is unique.  We define it to be out estimator: 
$$\lintkernel_{L, M, \mH} \coloneqq \argmin{\intkernelvar \in \mH}\mE_{L, M, \mM}(\intkernelvar)\,.$$  
As $M \rightarrow \infty$, by the law of large numbers, we have $\mE_{L, M, \mM} \rightarrow \mE_{L, \infty, \mM}$, with 
\begin{equation}
 \mE_{L, \infty, \mM}(\intkernelvar) \coloneqq \frac{1}{L}\sum_{l = 1}^{L}\E\Big[\norm{\dotbXtl - \mbf{f}^{\text{c}}_{\intkernelvar}(\bXtl)}_{g}^2\Big]. \label{error functional expected}
\end{equation}
Since $\mE_{L, \infty, \mM}$ is continuous on $\mH$, the minimization of $\mE_{L, \infty, \mM}$ over $\mH$ is well-posed and it has a unique minimizer $\lintkernel_{L, \infty, \mH} \coloneqq \argmin{\intkernelvar \in \mH}\mE_{L, \infty, \mM}(\intkernelvar)$.  Much of our theoretical work establishes the relationship between the estimator $\lintkernel_{L, M, \mH}$, the closely related (in the infinite sample limit $M\rightarrow\infty$) $\lintkernel_{L, \infty, \mH}$, and the true interaction kernel $\intkernel$.
\subsection{Performance Measures}

We introduce a suitable normed function space in which to compare the estimator $\lintkernel_{L, M, \mH}$ with the true interaction kernel $\intkernel$. We also measure performance in terms of trajectory estimation error based on a distance between trajectories generated from the true dynamics (evolved using $\intkernel$ with some initial condition $\bX_0 \sim \muX$) and the estimated dynamics (evolved using the estimated interaction kernel $\lintkernel_{L, M, \mH}$, and with the same initial condition, i.e. $\bX_0$).

\subsubsection{Estimation Error}
First we introduce a probability measure $\rho_{T, \mM}$ on $\mathbb{R}_+$, that is used to define a norm to measure the error of the estimator, derived from the loss functionals (given by \eqref{error functional} and \eqref{error functional expected}), that reflects the distribution of pairwise data given by the dynamics as well as the geometry of the manifold $\mM$:
\begin{equation*}
\rho_{T, \mM}(r) \coloneqq \frac{1}{\binom{N}{2}}\E\Big[ \frac1T\int_{0}^T\sum_{i,i'} \delta_{d_{\mM}(\bx_i(t), \bx_{i'}(t))}(r)\, dt\Big]\,,
\end{equation*}
where $\delta$ is the Dirac delta function.  In words, this measure is obtained by averaging $\delta$-functions having mass at any pairwise distances in any trajectory, over all initial conditions drawn from $\muX$,  over all pairs of agents and all times.  A time-discretized version is given by:
\begin{equation*}
\rho_{T, \mM}^L(r) \coloneqq \frac{1}{L\binom{N}{2}}\E\Big[\sum_{l=1}^L\sum_{1 \le i < i' \le N} \delta_{d_{\mM}(\bx_i(t_l), \bx_{i'}(t_l))}(r)\Big].
\end{equation*}
The two probability measures defined above appear naturally in the proofs for the convergence rate of the estimator.
From observational data we compute the empirical version:
\begin{equation*}
\rho_{T, \mM}^{L, M}(r) \coloneqq \frac{1}{ML\binom{N}{2}}\sum_{l, m = 1}^{L, M}\sum_{1 \le i < i' \le N}\delta_{d_{\mM}(\bx_i(t_l), \bx_{i'}(t_l))}(r).
\end{equation*}
The geometry of $\mM$ is incorporated in these three measures by the presence of geodesic distances.  
The norm
\begin{equation*}
\norm{\intkernelvar(\cdot)\cdot}_{L^2(\rho_{T,\mM})}^2 \coloneqq \int_{r = 0}^{\infty}\abs{\intkernelvar(r)r}^2\, d\rho_{T,\mathcal{M}}(r)\,
\end{equation*}
is used to define the estimation error: $||{\lintkernel_{L, M, \mH}(\cdot)\cdot - \intkernel(\cdot)\cdot}||_{L^2(\rho_{T, \mM})}$.  
We also use a relative version of this error, to enable a meaningful comparison across different interaction kernels:
\begin{equation}\label{eq:rel_L2rhoT_error}
\norm{\intkernelvar(\cdot)\cdot - \intkernel(\cdot)\cdot}_{\text{Rel.} L^2(\rho_{T,\mM})}\!\! \coloneqq \dfrac{\norm{\intkernelvar(\cdot)\cdot - \intkernel(\cdot)\cdot}_{L^2(\rho_{T,\mM})}}{\norm{\intkernel(\cdot)\cdot}_{L^2(\rho_{T,\mathcal{M}})}}.
\end{equation}
\subsubsection{Trajectory Estimation Error}
Let $\bX^{m}_{[0, T]} \coloneqq (\bX^{m}_t)_{t \in [0, T]}$ be the trajectory generated by the $m^{th}$ initial condition, $\bX^{m}_0$.  The trajectory estimation error between $\bX^{m}_{[0, T]}$ and $\hat\bX^{m}_{[0, T]}$, evolved using, the unknown interaction kernel $\intkernel$ and, respectively, the estimated one, $\lintkernel$, with the same initial condition, is given by
\begin{equation}\label{eq:traj_norm}
d_{\text{trj}}(\bX^{m}_{[0, T]},\hat\bX^{m}_{[0, T]})^2\! \coloneqq \!\!\!\sup\limits_{t\in[0,T]}\!\!\!{\frac{\sum_i d_{\mM}(\bx_i^{m}(t), \hat\bx_i^{m}(t))^2\!\!\!}{N}}.
\end{equation}
We are also interested in the performance over different initial conditions, hence we use $\text{mean}_{\text{IC}}$ and $\text{std}_{\text{IC}}$ to report the mean and std of these trajectory errors over a (large) number of initial conditions sampled i.i.d. from $\muX$.
\subsection{Algorithm}
Algorithm \ref{alg:learning} shows the detailed steps on how to construct the estimator to $\intkernel$ given the observation data.
\subsection{Computational Complexity}\label{sec:comp_complex}
Assuming a finite dimensional subspace of $\hypspace$, i.e. $\hypspace_M \subset \hypspace$ with $\dim(\hypspace_M) = n(M)$, we are able to re-write the minimization problem of \eqref{error functional} over $\hypspace_M$ as a linear system, i.e. $A_M\vec{\alpha} = \vec{b}_M$ with $A_M \in R^{n \times n}$ and $\vec{b}_M \in \R^{n \times 1}$; for details, see the Appendix.   This linear system is well conditioned, ensured by the geometric coercivity condition. 

The total computational cost for solving the learning problem is: $MLN^2 + MLdn^2 + n^3$ with $MLN^2$ for computing pairwise distances, $MLdn^2$ for assembling $A_M$ and $\vec{b}_M$, and $n^3$ for solving $A_M\vec{\alpha} = \vec{b}_M$.   When choosing the optimal $n = n_* \approx (\frac{M}{\log M})^{\frac{1}{2s + 1}} \approx M^{\frac{1}{3}}$ ($s = 1$ for $C^1$ functions) as per Thm. \ref{thm:optimal rate of convergence}, we have
$\text{comp. time} = MLN^2 + MLdM^{\frac{2}{3}} + M = \mathcal{O}(M^{\frac{5}{3}})$.
The computational bottleneck comes from the assembly of $A_M$ and $\vec{b}_M$.  However, since we can parallelize our learning approach in $m$, the updated computing time in the parallel regime is
$\text{comp. time} = \mathcal{O}\Big(\Big(\frac{M}{\text{num. cores}}\Big)^{\frac{5}{3}}\Big)$. 
The total storage for the algorithm is $MLNd$ floating-point numbers for the trajectory data,  albeit one does not need to hold all of the trajectory data in memory.  The algorithm can process the data from one trajectory at a time, requiring $LNd$.  Once the linear system, $A_M\vec{\alpha} = \vec{b}_M$, is assembled, the algorithm just needs to hold roughly $n^2$ floating-point numbers in memory.   When we use the optimal number of basis functions, i.e. $n_* = M^{\frac{1}{3}}$, the memory used is $O(M^{\frac{2}{3}})$.
\begin{algorithm}[H] %, this is placed here for fitting in the document purposes. 
   \caption{Learning Algorithm}
   \label{alg:learning}
\small{
\begin{algorithmic}
   \STATE {\bfseries Input:} data $\{\bx_i^{m}(t_l), \dot\bx_i^{m}(t_l)\}_{i, l, m = 1}^{N, L, M}$
   \STATE Compute \scalebox{0.8}{$R^{\text{obs}}_{\{\min,\max\}} = \{\min,\max\}_{i, i', l, m}d_{\mM}(\bx_i^{m}(t_l), \bx_{i'}^{m}(t_l))$}
   \STATE Choose a type of basis functions, e.g., clamped B-spline
   \STATE Construct basis of $\hypspace_M$, e.g. $\{\basis_{\eta}\}_{\eta = 1}^n$, on the uniform partition of $[R^{\text{obs}}_{\min}, R^{\text{obs}}_{\max}]$
   \STATE Choose either a local chart $\mU:\mM \rightarrow \R^{\sdim}$ or a natural embedding $\mI:\mM \rightarrow \R^{d'}$
   \STATE Construct $\Psi^{m} \in (T_{\bX^{m}_{t_1}}\mM^N \times \cdots \times T_{\bX^{m}_{t_L}}\mM^N)^n$ and $\vec{d}^{m} \in T_{\bX^{m}_{t_1}}\mM^N \times \cdots \times T_{\bX^{m}_{t_L}}\mM^N$:
   \begin{align*}
   \Psi^{m}(:, \eta) &\coloneqq \Psi^{m}_{\eta} = \frac{1}{\sqrt{N}}\begin{bmatrix} \mbf{f}^{\text{c}}_{\basis_{\eta}}(\bX^{m}_{t_1}) \\ \vdots \\ \mbf{f}^{\text{c}}_{\basis_{\eta}}(\bX^{m}_{t_L}) \end{bmatrix}\,,\,
   \vec{d}^{m} \coloneqq \frac{1}{\sqrt{N}}\begin{bmatrix} \dot\bX^{m}_{t_1} \\ \vdots \\ \dot\bX^{m}_{t_L} \end{bmatrix}
   \end{align*}
   \STATE Define $\inprod{\cdot, \cdot}_G$ on $\Psi^{m}_{\eta} \in T_{\bX^{m}_{t_1}}\mM^N \times \cdots \times T_{\bX^{m}_{t_L}}\mM^N$ as
   \[
   \scalebox{0.8}{$
     \inprod{\Psi^{m}_{\eta}, \Psi^{m}_{\eta'}}_{G} = \sum_{l = 1}^L\inprod{\mbf{f}^{\text{c}}_{\psi_{\eta}}(\bX^{m}_{t_l}), \mbf{f}^{\text{c}}_{\psi_{\eta'}}(\bX^{m}_{t_l})}_{g^{\mM^N}(\bX_l^{m})} 
   $}
   \]
   \STATE Assemble $A_M(\eta, \eta') = \frac{1}{LM}\sum_{m = 1}^M\inprod{\Psi^{m}_{\eta}, \Psi^{m}_{\eta'}}_G\in \R^{n \times n}$.
   \STATE Assemble $\vec{b}_M(\eta) = \frac{1}{LM}\sum_{m = 1}^M\inprod{\vec{d}, \Psi^{m}_{\eta}}_G\in\R^{n \times 1}$.
   \STATE Solve $A_M\vec{\alpha}=\vec{b}_M$ for $\vec{\hat{\alpha}}\in\R^{n}$.
   \STATE Assemble $\lintkernel = \sum_{\eta = 1}^n \hat{\alpha}_{\eta}\basis_{\eta}$.
\end{algorithmic}
}
\end{algorithm}

\section{Learning Theory} \label{sec:learningtheory}
We present in this section the major results establishing the convergence of the estimator $\lintkernel_{L, M, \mH}$ to $\intkernel$, at the optimal learning rate, and bounding the trajectory estimation error between the true and estimated dynamics  (evolved using $\lintkernel_{L, M, \mH}$), with their corresponding proofs in the Appendix.  

\subsection{Learnability: geometric coercivity condition}\label{sec:gcc}
We establish a geometry-adapted coercivity condition, extending that of \cite{bongini2016inferring,Lu2019a} to the Riemannian setting, which will guarantee the uniqueness of the minimizer of $\mE_{L, \infty, \mM}(\intkernelvar)$, and that $\mE_{L, \infty, \mM}(\intkernelvar)$ controls the $\norm{\cdot}_{L^2(\rho_{T,\mM})}$ distance between the minimizer and the true interaction kernel.
\begin{definition}[Geometric Coercivity condition]\label{CoercivityCondition}
The geometric evolution system in \eqref{eq:mainmodel} with initial condition sampled from $\muX$ on $\mM^N$ is said to satisfy the geometric coercivity condition on the admissible hypothesis space $\mH$ if there exists a constant $c \equiv c_{L, N, \mH, \mM} > 0$ such that for any $\intkernelvar \in \mH$ with $\intkernelvar(\cdot)\cdot \in L^2(\rho_{T, \mM}^L)$ we have
\begin{align}
  &c\norm{\intkernelvar(\cdot)\cdot}_{L^2(\rho_{T,\mM}^L)}^2 \leq \nonumber
  \frac{1}{L}\sum_{l = 1}^{L}\E\Big[\norm{\mbf{f}^{\text{c}}_{\intkernelvar}(\bXtl)}_{T_{\bXtl}\mM^N}^2\Big].\label{ccineq}
\end{align}
\end{definition}
\noindent
In order to simplify the argument on how this geometric coercivity condition controls the distance between $\lintkernel_{L, \infty, \mH}$ and $\intkernel$, we introduce an inner product on $L^2 = L^2(\rho_{T, \mM}^L)$:
\[
\dbinprod{\intkernelvar_1, \intkernelvar_2}_{L^2} \coloneqq \frac{1}{L}\sum_{l = 1}^{L}\E\Big[\inprod{\mbf{f}^{\text{c}}_{\intkernelvar_1}(\bXtl), \mbf{f}^{\text{c}}_{\intkernelvar_2}(\bXtl)}_{T_{\bXtl}\mM^N}\Big].
\]
Then the geometric coercivity condition can be rewritten as
\[
  c_{L,N, \mH, \mM}\norm{\intkernelvar(\cdot)\cdot}_{L^2(\rho_{T,\mM}^L)}^2 \leq \dbinprod{\intkernelvar, \intkernelvar}_{L^2(\rho_{T, \mM}^L)},
\]
and since the loss function from \eqref{error functional expected} can be written as
$
\mE_{L, \infty, \mH}(\intkernelvar) = \dbinprod{\intkernelvar - \intkernel, \intkernelvar - \intkernel} 
$, this implies		
\[
c_{L, N, \mH, \mM}\norm{\intkernelvar(\cdot)\cdot - \intkernel(\cdot)\cdot}_{L^2(\rho_{T,\mM}^L)}^2 \le \mE_{L, \infty, \mH}(\intkernelvar).
\]
Hence when $\mE_{L, \infty, \mH}(\intkernelvar)$ is small, $\norm{\intkernelvar(\cdot)\cdot - \intkernel(\cdot)\cdot}_{L^2(\rho_{T,\mM}^L)}$ is also small; hence if we construct a sequence of minimizers of $\mE_{L, \infty, \mH}$ over increasing $\mH$ with decreasing $\mE_{L, \infty, \mH}$ values, the convergence of $\lintkernel_{L, \infty, \mH}$ to $\intkernel$ can be established.  
\subsection{Concentration and Consistency}
The first theorem bounds, with high probability, the difference between the estimator $\lintkernel_{L, M, \mH}$ and the true interaction kernel $\intkernel$, which makes apparent the trade-off between the $L^2(\rho_{T, \mM}^L)$-distance between $\intkernel$ and $\mH$ (approximation error), and $M$ the number of trajectories needed for achieving the desired accuracy.
Here $\mN(\mU, \epsilon)$ is the covering number of a set $\mU$ with open balls of radius $\epsilon$ w.r.t the $L^{\infty}$-norm.
\begin{theorem}\label{control bias}
Let $\intkernel \in L^2([0,R])$, and $\mH$ an admissible hypothesis space such that the geometric coercivity condition holds with a constant $c_{L, N,\mH, \mM}$. 
Then, $\lintkernel_{L, M, \mH}$, minimizer of \eqref{error functional} on the trajectory data generated by \eqref{eq:mainmodel}, satisfies
\[
\norm{\lintkernel_{L, M, \mH}(\cdot)\cdot - \intkernel(\cdot)\cdot}_{L^2(\rho_{T, \mM}^L)}^2 \le \frac{2}{c_{L, N,\mH, \mM}}\Big(\epsilon + \inf\limits_{\intkernelvar \in \mH}\norm{\intkernelvar(\cdot)\cdot - \intkernel(\cdot)\cdot}_{L^2(\rho_{T, \mM}^L)}^2\Big)
\]
with probability at least $1 - \tau$, when $M \geq \frac{1152S_0^2R^2}{\epsilon c_{L, N, \mH, \mM}}(\ln\mN(\mH, \frac{\epsilon}{48S_0R^2}) + \ln\frac{1}{\tau})$. 
\end{theorem}
This quantifies the usual bias-variance tradeoff in our setting:  on the one hand, with a large hypothesis space, the quantity $\smash{\inf_{\intkernelvar \in \mH}\norm{\intkernelvar(\cdot)\cdot - \intkernel(\cdot)\cdot}_{L^2(\rho_{T, \mM}^L)}}$ could be made small.  On the other hand, we wish to have the right number of samples to make the variance of the estimator small, by controlling the covering number of the hypothesis space $\mH$.
\subsection{Convergence Rate} \label{sec:learningrate}
Next we establish the convergence rate of $\lintkernel_{L, M, \mH}$ to $\intkernel$ as $M$ increases. 		
\begin{theorem}\label{optimal rate of convergence}
Let $\muX$ be the distribution of the initial conditions of trajectories, and $\mH_M = \mB_n$ with $n \asymp ({M}/{\log M})^{\frac{1}{2s + 1}}$, where $\mB_{n}$ is the central ball of $\mL_n$ with radius $c_1 + S$, and the linear space $\mL_n \subseteq L^{\infty}([0,R])$ satisfies
  \[
    dim(\mL_n) \leq c_0n \mand \inf\limits_{\intkernelvar \in \mathcal{L}_n} \norm{\intkernelvar - \intkernel}_{L^{\infty}} \leq c_1n^{-s}
  \]
  for some constants $c_0, c_1, s > 0$. Suppose that the geometric coercivity condition holds on $\mathcal{L}\coloneqq\cup_n \mL_n$ with constant $c_{L, N, \mathcal{L}, \mM}$.  Then there exists some constant $C(S, R, c_0, c_1)$ such that %_{\bX_0 \sim \muX}
  \[
  \E\Big[\norm{\lintkernel_{L, M, \mH_M}(\cdot)\cdot - \intkernel(\cdot)\cdot}_{L^2(\rho_{T, \mM}^L)}\Big] \le \frac{C(S, R, c_0, c_1)}{c_{L, N, \mathcal{L}, \mM}}\Big(\frac{\log M}{M}\Big)^{\frac{s}{2s + 1}}\,.
  \]
\end{theorem}
The constant $s$ is tied closely to the regularity of $\intkernel$, and it plays an important role in the convergence rate.  For example, when $\intkernel \in C^1$, we can take $s = 1$ with linear spaces of first degree piecewise polynomials, we end up with a $M^{\frac{1}{3}}$ learning rate.
The rate is the same as the minimax rate for nonparametric regression in one dimension (up to the logarithmic factor), and is independent of the dimension $D = N\sdim$ of the state space.   Empirical results suggest that at least in some cases, when $L$ grows, i.e. each trajectory is sampled at more points, then the estimators improve; this is however not captured by our bound.  
\subsection{Trajectory Estimation Error} \label{sec:trajanalysis}
We have established the convergence of the estimator $\lintkernel_{L, M, \mH}$ to the true interaction kernel $\intkernel$.   
We now establish the convergence of the trajectories of the estimated dynamics, evolved using $\lintkernel_{L, M, \mH}$, to the observed trajectories.
\begin{theorem}\label{Traj Acc Bound}
Let $\intkernel \in \mK_{R, S}$ and $\lintkernel \in \mK_{R, S_0}$, for some $S_0\ge S$.  Suppose that $\bX_{[0, T]}$ and $\hat\bX_{[0, T]}$ are solutions of \eqref{eq:mainmodel} w.r.t to $\intkernel$ and $\lintkernel$, respectively, for $t \in [0, T]$, with $\hat\bX_0=\bX_0$.  Then we have the following inequality,
\[
  \E\Big[d_{\text{trj}}\Big(\bX_{[0, T]},\hat\bX_{[0, T]}\Big)^2\Big] \leq 4T^2C(\mM,T)\exp(64T^2S_0^2)\norm{\intkernel(\cdot)\cdot - \lintkernel(\cdot)\cdot}_{L^2(\rho_{T, \mM})}^2,
\]
where $C(\mM,T)$ is a positive constant depending only on geometric properties of $\mM$ and $T$, but may be chosen independent of $T$ if $\mM$ is compact.
\end{theorem}
While these bounds are mainly useful for small times $T$,  given the exponential dependence on $T$ of the bounds, they can be overly pessimistic. 
It may also happen that the predicted trajectories are not accurate in terms of agent positions, but they maintain, and even predict from initial conditions, large-scale, emergent properties of the original system, such as flocking of birds of milling of fish \cite{Zhong20}. We suspect this can hold also in the manifold setting, albeit in ways that are affected by geometric properties of the manifold.
\section{Numerical Experiments}\label{sec:num_results}
We consider two prototypical first order dynamics, Opinion Dynamics (OD) and Predator-Swarm dynamics (PS$1$), each on two different manifolds, the $2D$ sphere $\mathbb{S}^2$, centered at the origin with radius $\frac{5}{\pi}$, and the Poincar\'{e} disk $\mathbb{PD}$ (unit disk centered at the origin, with the hyperbolic metric).  These are model spaces with constant positive and negative curvature, respectively.  We conduct extensive experiments on these four scenarios to demonstrate the performance of the estimators both in terms of the estimation errors (approximating $\intkernel$'s) and trajectory estimator errors (estimating the observed dynamics) over $[0, T]$.

For each type of dynamics, on each of the two model manifolds, we visualize trajectories of the system, with a random initial condition (i.e. not in the training set), driven by $\intkernel$ and $\lintkernel$. We also augment the system by adding new agents: without any re-learning, we can transfer $\lintkernel$ to drive this augmented system (with $N=40$ in our examples), for which will also visualize the trajectories (again, started from a new random initial condition. We also report on the (relative) estimation error of the interaction kernel, as defined in \eqref{eq:rel_L2rhoT_error}, and on the trajectory errors, defined in \eqref{eq:traj_norm}.

\begin{figure}[H]
\centering
\begin{subfigure}[b]{0.49\textwidth} 
\centering
\includegraphics[width=\textwidth]{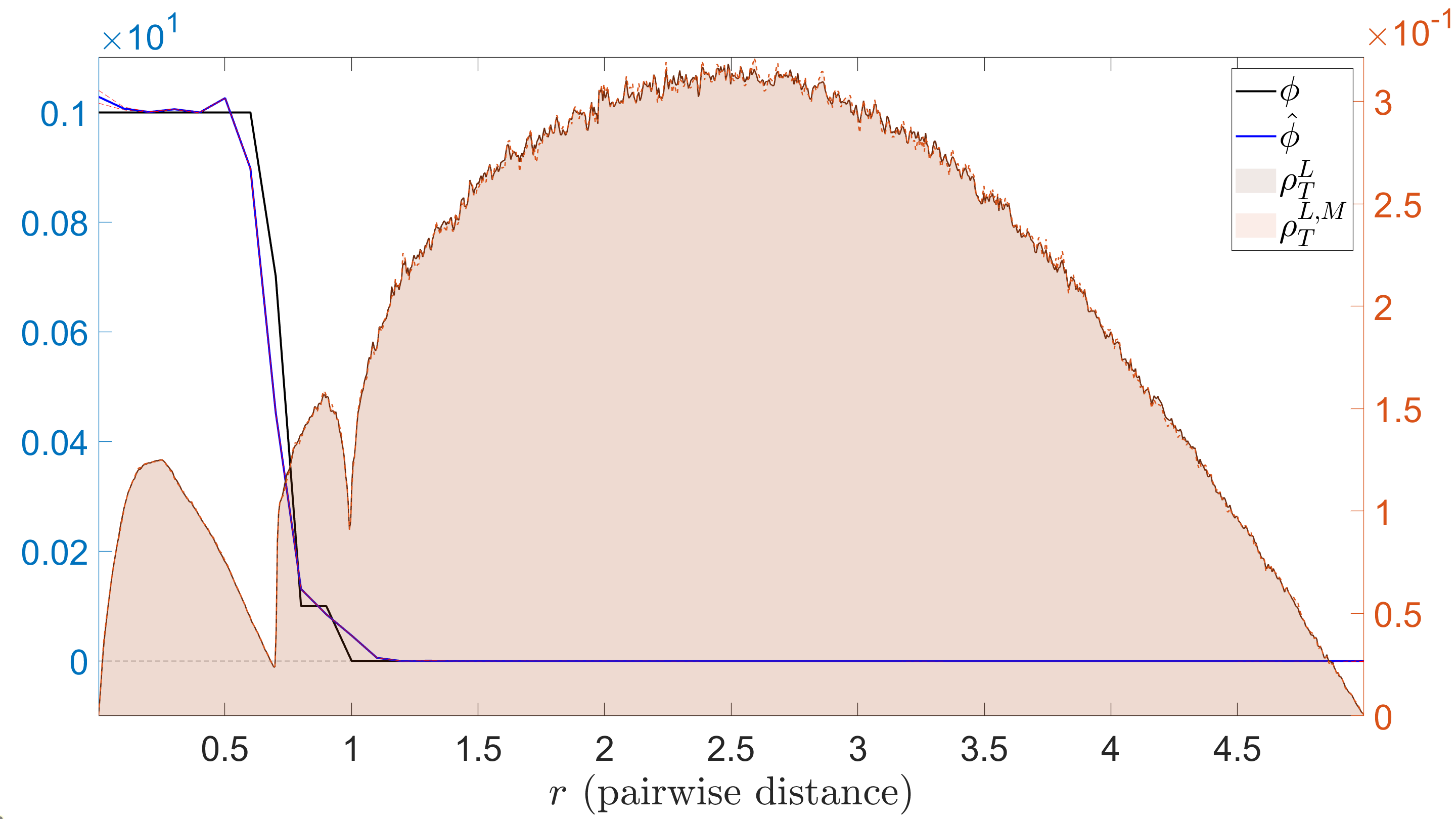} 
\caption{OD on $\mathbb{S}^2$}
\end{subfigure}
\begin{subfigure}[b]{0.49\textwidth}
\includegraphics[width=\textwidth]{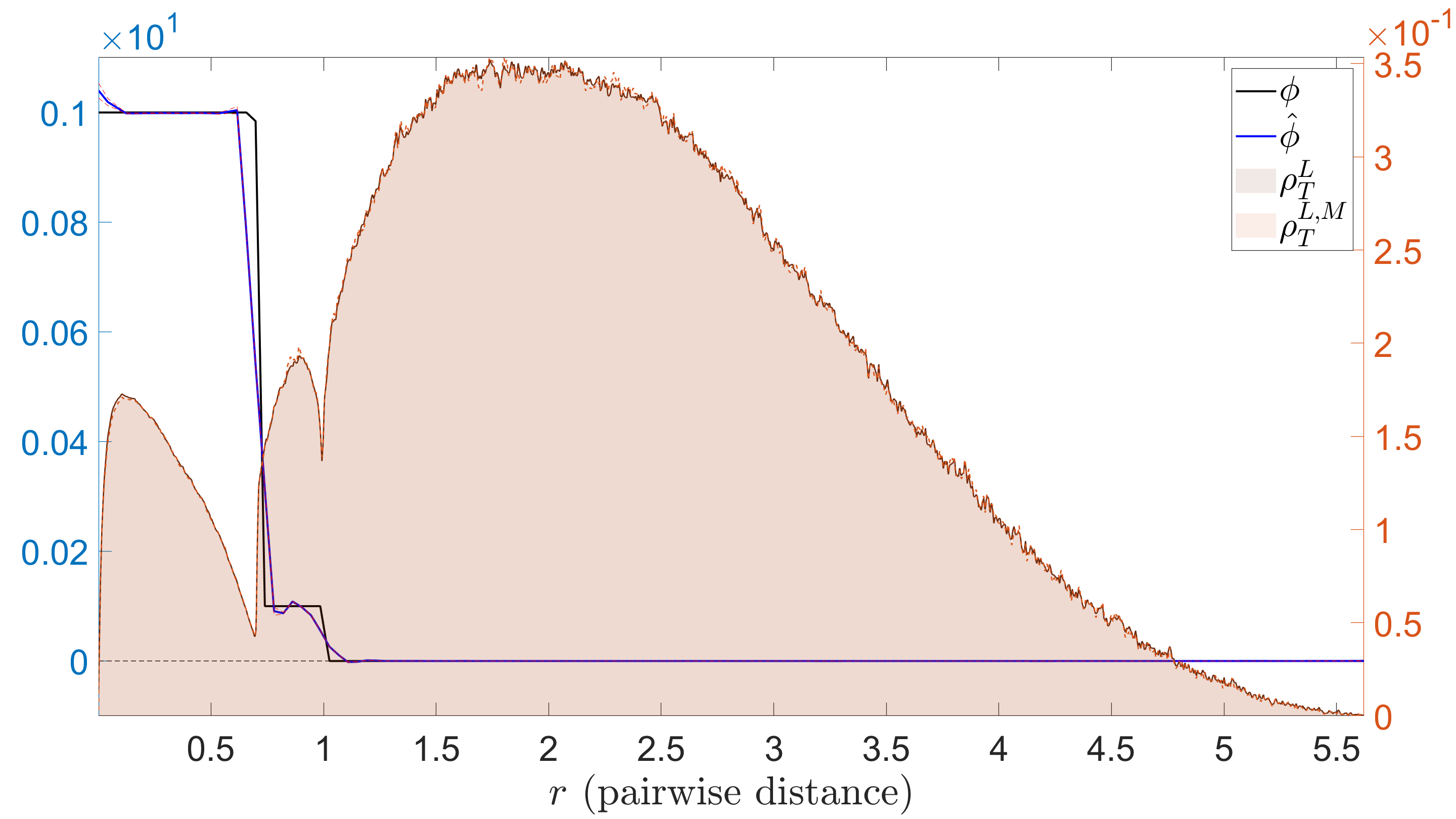}
\caption{OD on $\mathbb{PD}$}
\end{subfigure}
\centering
\begin{subfigure}[b]{0.49\textwidth}
\centering
\includegraphics[width=\textwidth]{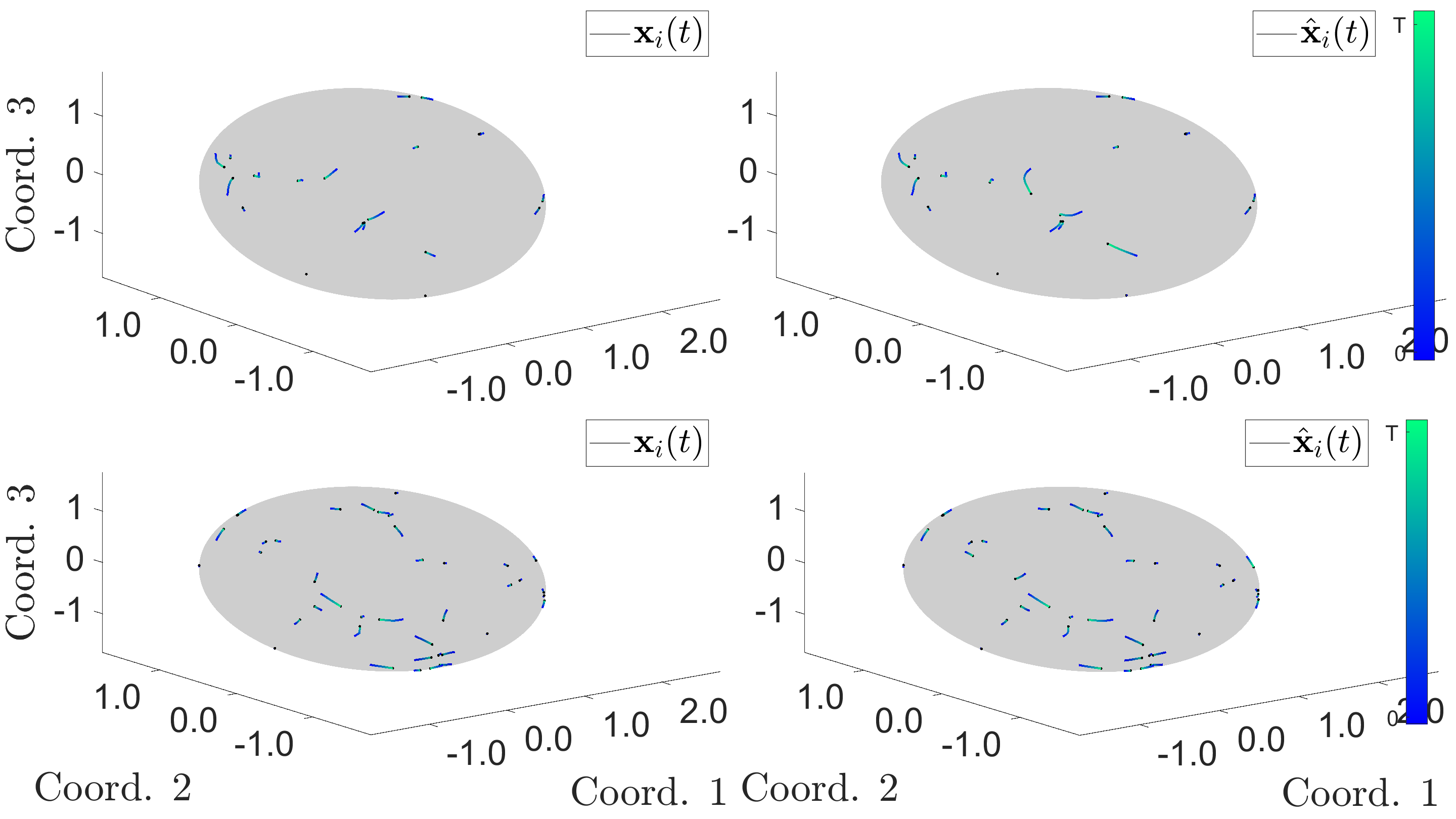}
\caption{OD on $\mathbb{S}^2$}
\end{subfigure} \,
\begin{subfigure}[b]{0.49\textwidth}
\centering
\includegraphics[width=\textwidth]{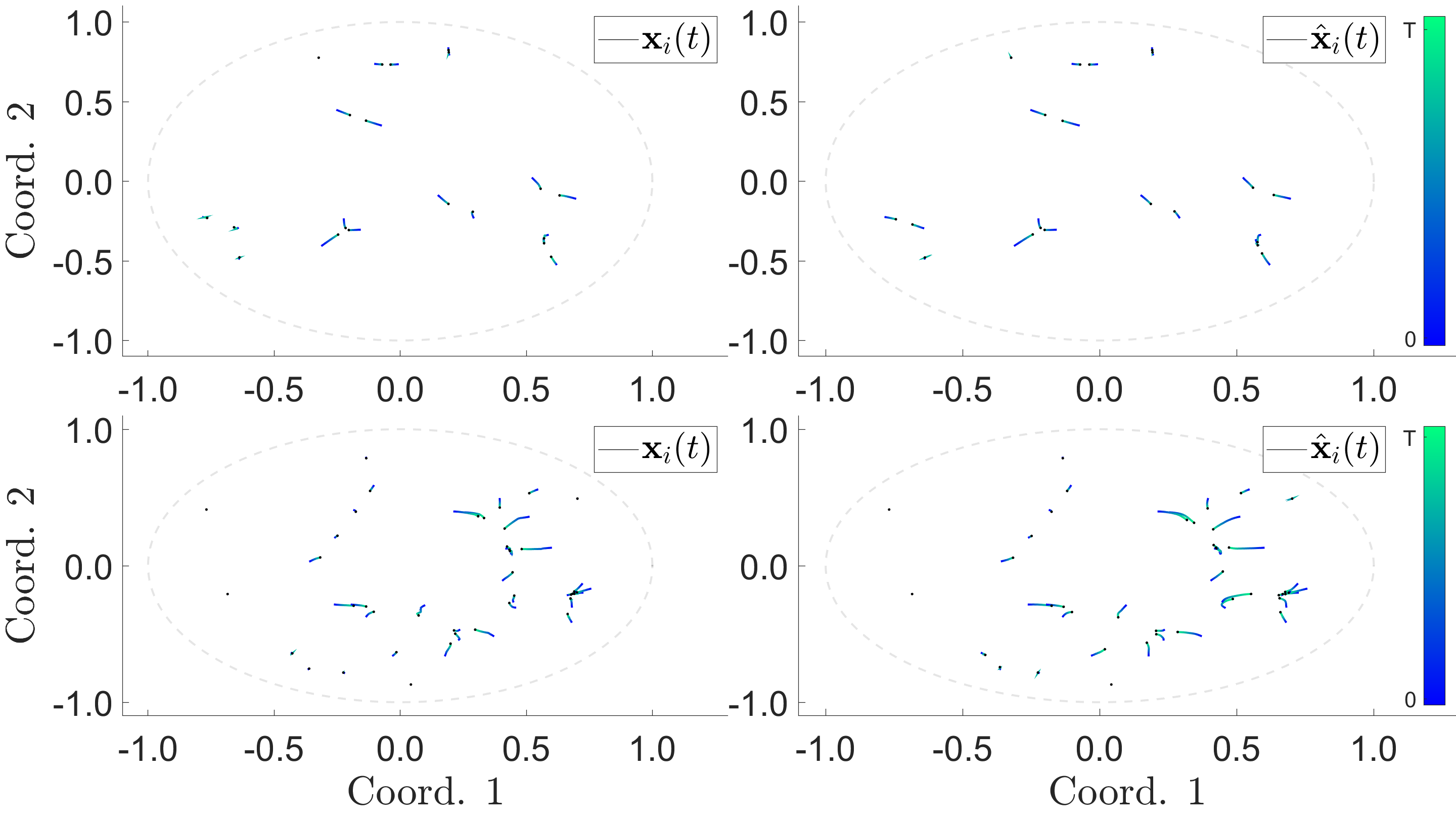}
\caption{OD on $\mathbb{PD}$}
\end{subfigure}
\caption{{\bf{Top}}: comparison of $\intkernel$ and $\lintkernel$. The true interaction kernel is shown with a black solid line, whereas the mean estimated interaction kernel is shown with a blue solid line with its confidence interval shown in red dotted lines.  Shown in the background is the comparison of the approximate $\rho_{T, \mM}^L$ versus the empirical $\rho_{T, \mM}^{L, M}$.
{\bf{Bottom}}: comparison of trajectories $\bX_{[0,T]}$ and $\hat\bX_{[0,T]}$. The trajectories $\bX_{[0,T]}$'s generated by the true interaction kernel $\intkernel$; whereas $\smash{\hat\bX_{[0,T]}}$'s are trajectories generated by the estimator $\smash{\lintkernel}$, with the same initial conditions. In the first row, trajectories are started from a randomly chosen initial condition. In the second row, trajectories are generated for a new system, with $N = 40$ agents. The colors along the trajectories indicate time, from deep blue (at $t = 0$) to light green (at $t = T$).}
\label{fig:OD_results}
\end{figure}
For each system of $N = 20$ agents, we take $M = 500$ and $L = 500$ to generate the training data.  For each $\hypspace_M$, we use first-degree clamped B-splines as the basis functions with $\dim(\hypspace_M) = \mathcal{O}(n_*) = (\frac{ML}{\log(ML)})^{\frac{1}{3}}N^{\frac{1}{d}}$.  We use a geometric numerical integrator \cite{Hairer2001} ($4^{th}$ order Backward Differentiation Formula with a projection scheme) for the evolution of the dynamics.  For details, see the Appendix. 

\textbf{Opinion Dynamics (OD)} is used to model simple interactions of opinions \cite{OpinionDynamicsAylin2017, continuousOD} as well as choreography \cite{CLP2014}. 
In fig.\ref{fig:OD_results} we display trajectories of the system on the two model manifolds.
The results are summarized in fig.\ref{fig:OD_results}.
The relative error of the estimator $\lintkernel$ for OD on $\mathbb{S}^2$ is $1.894 \cdot 10^{-1} \pm 3.1 \cdot 10^{-4}$, whereas for OD on $\mathbb{PD}$ is $1.935 \cdot 10^{-1} \pm 9.5 \cdot 10^{-4}$, both are calculated using \eqref{eq:rel_L2rhoT_error}.
The errors for trajectory prediction are reported in table \ref{tab:traj_err}.
\begin{table}[H]
\vskip-1cm
\centering
\small{\begin{tabular}{| c || c |} 
\hline
        OD                                             & $[0, T]$                                \\
\hline
$\text{mean}_{\text{IC}}^{\mathbb{S}^2}$: Training ICs & $8.8 \cdot 10^{-2} \pm 1.7 \cdot 10^{-3}$ \\ 
\hline         
$\text{mean}_{\text{IC}}^{\mathbb{S}^2}$: Random ICs   & $9.0 \cdot 10^{-2} \pm 1.6 \cdot 10^{-3}$ \\
\hline   
\hline
$\text{mean}_{\text{IC}}^{\mathbb{PD}}$: Training ICs & $1.08 \cdot 10^{-1} \pm 1.6 \cdot 10^{-3}$ \\
\hline         
$\text{mean}_{\text{IC}}^{\mathbb{PD}}$: Random ICs   & $1.08 \cdot 10^{-1} \pm 2.6 \cdot 10^{-3}$ \\
\hline 
\end{tabular}}
\caption{(OD on $\mathbb{S}^2$ or $\mathbb{PD}$) $\text{mean}_{\text{IC}}$ is the mean of the trajectory errors over $M$ initial conditions (ICs), as defined in eq.\eqref{eq:traj_norm}.}
\label{tab:traj_err}
\end{table}
\begin{figure}[H]
\centering
\begin{subfigure}[b]{0.49\textwidth}
\centering
\includegraphics[width=\textwidth]{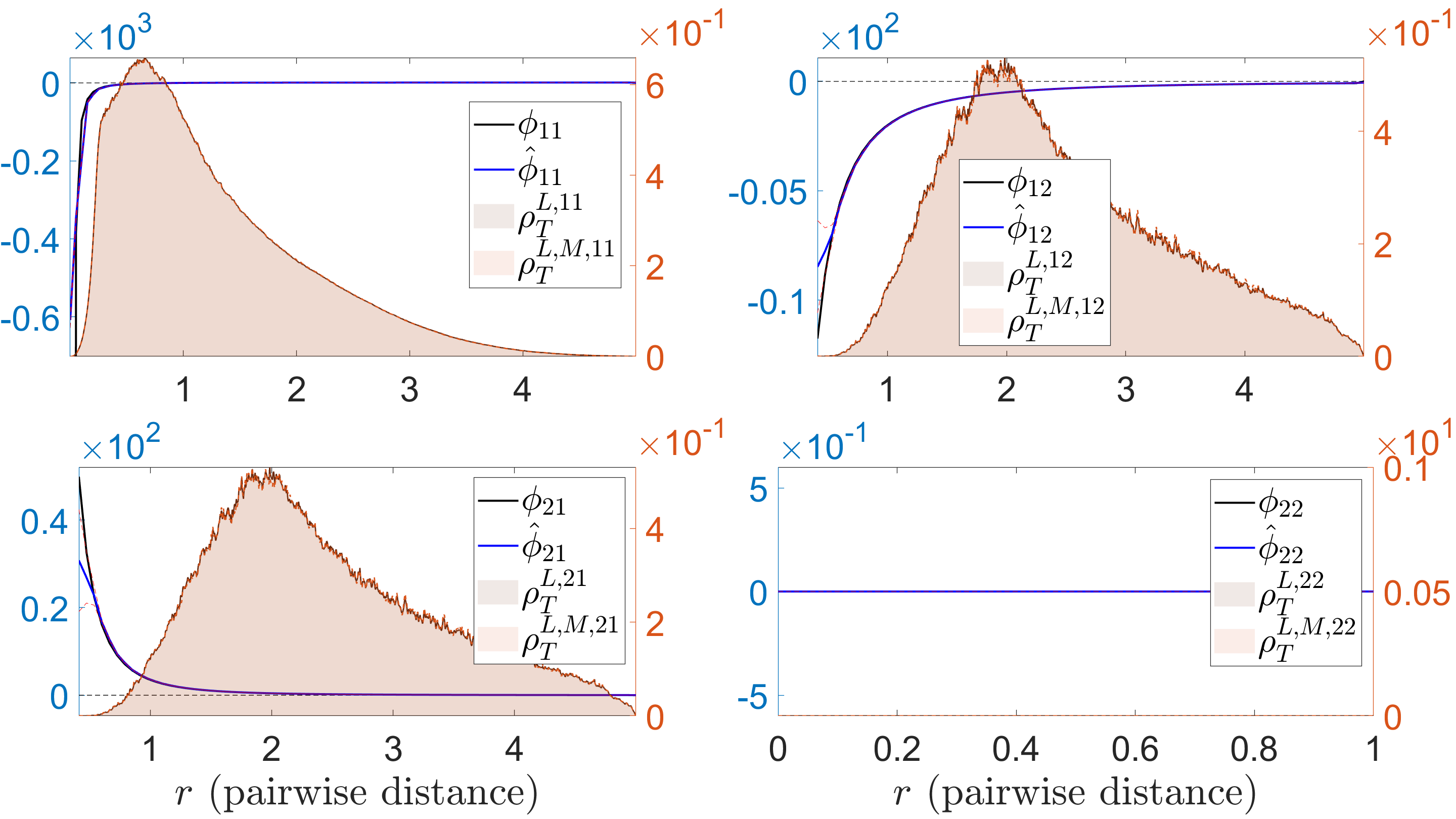}
\caption{PS$1$ on $\mathbb{S}^2$}
\end{subfigure}
\begin{subfigure}[b]{0.49\textwidth}
\centering
\includegraphics[width=\textwidth]{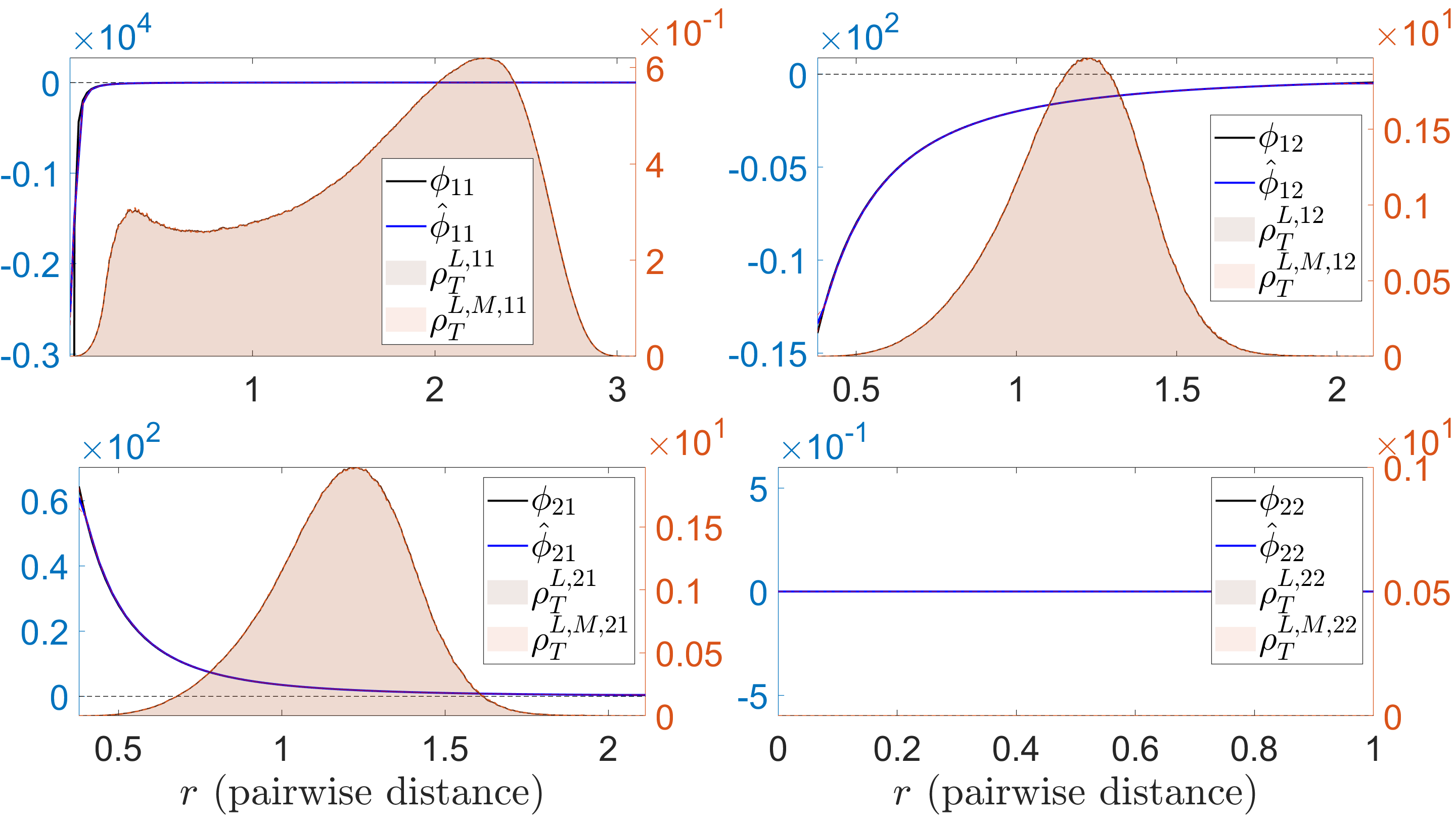}
\caption{PS$1$ on $\mathbb{PD}$}
\end{subfigure}
\begin{subfigure}[b]{0.49\textwidth}
\centering
\includegraphics[width=\textwidth]{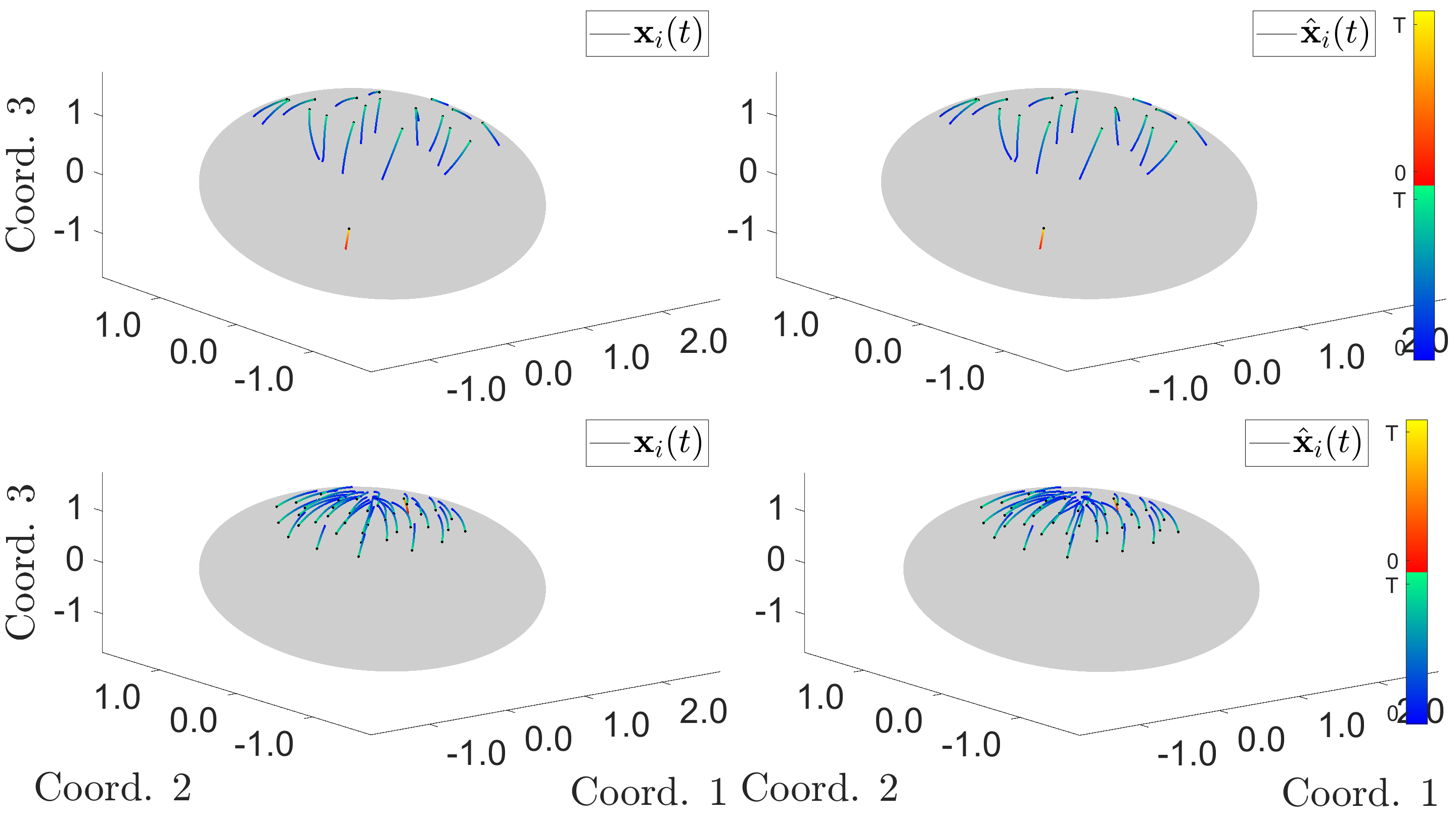}
\caption{PS$1$ on $\mathbb{S}^2$}
\end{subfigure}
\begin{subfigure}[b]{0.49\textwidth}
\centering
\includegraphics[width=\textwidth]{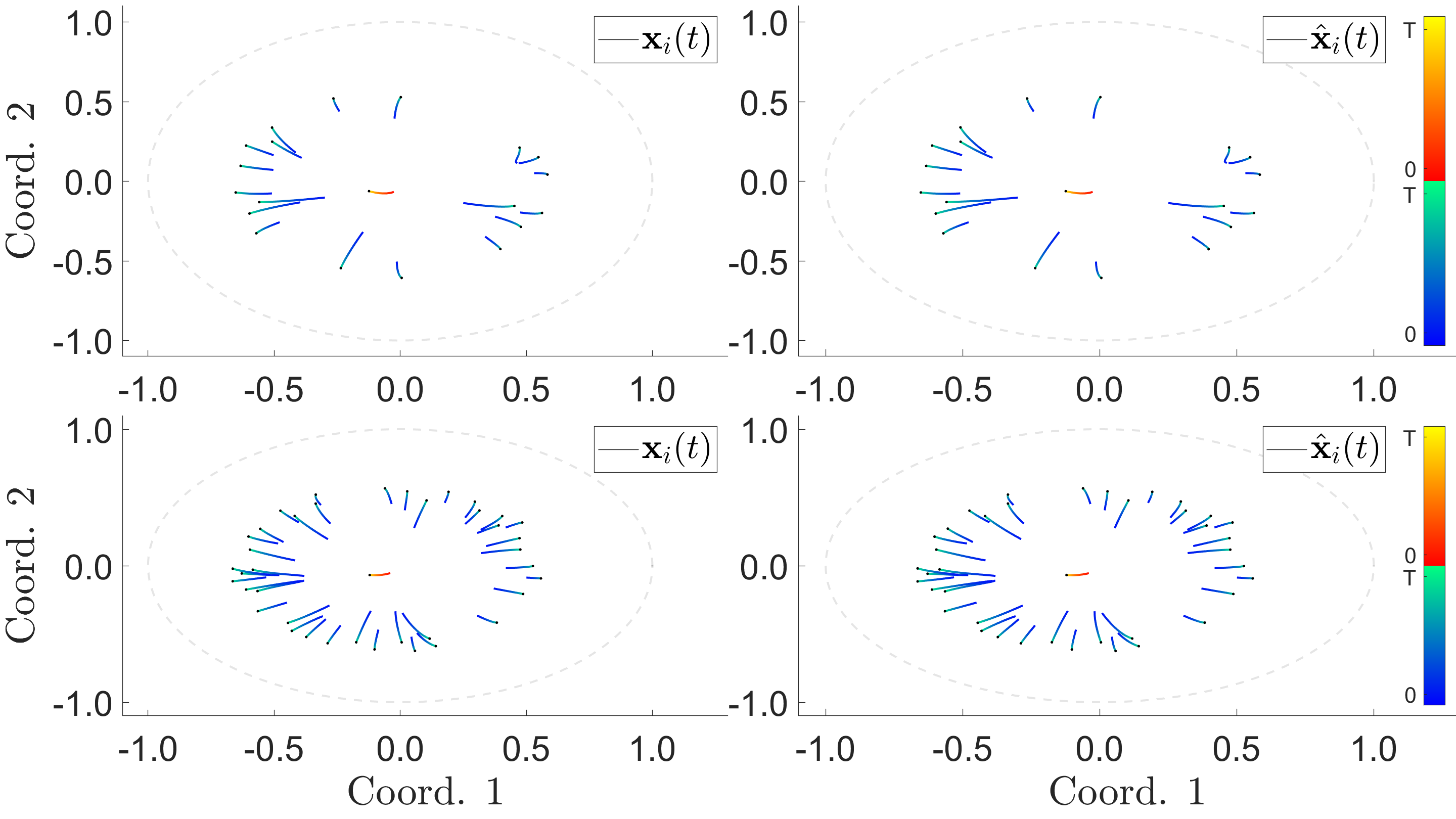}
\caption{PS$1$ on $\mathbb{PD}$}
\end{subfigure}
\centering
\caption{{\bf{Top}}: comparison of $\intkernel_{\idxcl, \idxcl'}$ and $\smash{\lintkernel_{\idxcl, \idxcl'}}$. The true interaction kernels are shown with a black solid line, whereas the mean estimated interaction kernels are shown with a blue solid line with their confidence intervals shown in red dotted lines.  Shown in the background is the comparison of the approximate $\rho_{T, \mM}^{L, \idxcl\idxcl'}$ versus the empirical $\rho_{T, \mM}^{L, M, \idxcl\idxcl'}$.  Notice that $\rho_T^{L, 12}$/$\rho_T^{L, M, 12}$ and $\rho_T^{L, 12}$/$\rho_T^{L, M, 21}$ are the same distributions.
{\bf{Bottom}}: comparison of trajectories $\bX_{[0,T]}$ and $\hat\bX_{[0,T]}$. The trajectories $\bX_{[0,T]}$'s generated by the true interaction kernel $\intkernel_{\idxcl, \idxcl'}$; whereas $\smash{\hat\bX_{[0,T]}}$'s are trajectories generated by the estimator $\smash{\lintkernel_{\idxcl, \idxcl'}}$, with the same initial conditions. In the first row, trajectories are started from a randomly chosen initial condition. In the second row, trajectories are generated for a new system, with $N = 40$ agents. The colors along the trajectories indicate time, from deep blue/bright red (at $t = 0$) to light green/light yellow (at $t = T$).  The blue/green combo is assigned to tge preys; whereas red/yellow combo to the predator.}
\label{fig:PS1_results}
\end{figure}
\textbf{Predator-Swarm System (PS1)}: this is a heterogeneous agent system, which is used to model interactions between multiple types of animals \cite{CK2013, Olson_2016}.  The learning theory presented in section \ref{sec:learningtheory} is described for homogeneous agent systems, but the theory and the corresponding algorithms extend naturally to heterogeneous agent systems in a manner analogous to \cite{Tang2019, miller2020learning}.  In this case, there are different interaction kernels, $\intkernel_{k,k'}$, one for each (directed) interaction between agents of type $k$ and agents of type $k'$. In our example here there are two types, \{prey,predator\}, and therefore $4$ interaction kernels; however there is only one predator, so the interaction kernel predator-predator is $0$. 
The results are visualized in fig.\ref{fig:PS1_results}.
The (relative) errors of the estimators are in table \ref{tab:PS1_phiE_errs}.
\begin{table}[H]
\centering
\tiny{
\begin{tabular}{ c | c }
\hline
$\text{Err}_{1, 1}^{\mathbb{S}^2}=2.98 \cdot 10^{-1} \pm 5.9 \cdot 10^{-3}$ & $\text{Err}_{1, 2}^{\mathbb{S}^2}=8.4 \cdot 10^{-3} \pm 3.0 \cdot 10^{-4}$ \\
\hline
$\text{Err}_{2, 1}^{\mathbb{S}^2}=2.5 \cdot 10^{-2} \pm 1.6 \cdot 10^{-3}$  & $\text{Err}_{2, 2}^{\mathbb{S}^2}=0$ \\
\hline
\hline
$\text{Err}_{1, 1}^{\mathbb{PD}}=6.2 \cdot 10^{-2} \pm 3.0 \cdot 10^{-3}$  & $\text{Err}_{1, 2}^{\mathbb{PD}}=9.1 \cdot 10^{-4} \pm 4.8 \cdot 10^{-5}$ \\
\hline
$\text{Err}_{2, 1}^{\mathbb{PD}}=2.7 \cdot 10^{-3} \pm 1.4 \cdot 10^{-4}$  & $\text{Err}_{2, 2}^{\mathbb{PD}}=0$ \\
\hline
\end{tabular}  
}
\caption{(PS$1$ on $\mathbb{S}^2$ or $\mathbb{PD}$) Relative estimation errors for $\lintkernel$.}
\label{tab:PS1_phiE_errs} 
\end{table}
\begin{table}[H]
\vskip-0.5cm
\centering
\small{\begin{tabular}{| c || c |} 
\hline
        PS$1$                                          & $[0, T]$                                \\
\hline
$\text{mean}_{\text{IC}}^{\mathbb{S}^2}$: Training ICs & $2.36 \cdot 10^{-2} \pm 9.8 \cdot 10^{-4}$ \\
\hline          
$\text{mean}_{\text{IC}}^{\mathbb{S}^2}$: Random ICs   & $2.40 \cdot 10^{-2} \pm 8.1 \cdot 10^{-4}$ \\
\hline
\hline 
$\text{mean}_{\text{IC}}^{\mathbb{PD}}$: Training ICs & $6.3 \cdot 10^{-3} \pm 2.0 \cdot 10^{-4}$ \\
\hline
$\text{mean}_{\text{IC}}^{\mathbb{PD}}$: Random ICs   & $6.4 \cdot 10^{-3} \pm 2.2 \cdot 10^{-4}$ \\
\hline  
\end{tabular}}
\caption{As in table \ref{tab:traj_err}, but for the PS1 system.}
\label{tab:traj_errPS1}
\vskip-0.5cm
\end{table}
\textbf{Discussion}: As shown in the figures and tables in this section, the estimators not only provide close approximation to their corresponding interaction kernels $\intkernel$'s, but also capture additional information about the true interaction laws, e.g. the support. The accuracy on the trajectories is consistent with the theory, and the lack of overfitting and the ability to generalize well to predicting trajectories started at new random initial conditions, which in general are very far from any of the initial conditions in the training data, given the high-dimensionality of the state space. This is truly made possible because we have taken advantage of the symmetries in the system, in particular invariance of the governing equations under permutations of the agents (of the same type, in the case of heterogeneous agent systems, such as PS1), and radial symmetry of the interaction kernels. Further invariances, when the number of agents increases, make it possible to re-use the interaction kernel estimated on a system of $N$ agents to predict trajectories of a system with the same interaction kernel, but a different number of agents, which of course has a state space of different dimension. This admittedly simple example of transfer would not possible for general-purpose techniques that directly estimate the r.h.s. of the system of ODEs.
\section{Conclusion}
We have considered the problem of estimating the dynamics of a special yet widely used set of dynamical systems, consisting of interacting agents on Riemannian manifolds. 
These are driven by a first-order system of ODEs on the manifold, with a typically very high-dimensional state space $\mM^{N}$, where $N$ is the (typically large) number of agents. 
We constructed estimators that are optimal and avoid the curse of dimensionality, but exploiting the multiple symmetries in these systems, and the simplicity of the underlying interaction laws.
Extensions to more complex systems of interacting agents may be considered, in particular to second-order systems, which will require the use of parallel transport on $\mM$, to more general interaction kernels, depending on other variables beyond pairwise distances, as well as to systems interacting with a varying environment.
\section{Acknowledgment}
MM is grateful for partial support from NSF-1837991, NSF-1913243, NSF-1934979, NSF-Simons-2031985, FA9550-20-1-0288, ARO W911NF-18-C-0082, and to the Simons Foundation for the Simons Fellowship for the year '20-'21; please direct any correspondence to MZ\footnote{mzhong5@jhu.edu}.   Prisma Analytics, Inc. provided computing equipment and support.

MM and MZ designed the research; all authors jointly wrote the manuscript; HQ derived theoretical results together with JM and MZ; MZ developed algorithms and applications; JM and MZ analyzed the data.
\appendix
\section{Preliminaries}\label{sec:manifold}
In this work, $\mM$ is a connected, smooth, and geodesically complete $\sdim$-dimensional Riemannian manifold with Riemannian metric $g$.  For details regarding the basic definitions of Riemannian manifolds, geodesics, Riemannian distances, exponential maps, cut loci, and injectivity radii, please see \cite{Lee2003, docarmo1976}.  We will discuss how to find the minimal geodesic and the Riemannian distance between any two points on the two prototypical manifolds used in our numerical algorithms: the two-dimensional sphere ($\mathbb{S}^2$) and the Poincar\'{e} Disk ($\mathbb{PD}$).
\subsection{Riemannian Geometry on the $2D$ Sphere}
The $2D$ Sphere ($\mathbb{S}^2$) of radius $r$ and centered at the origin can be isometrically embedded in $\R^3$ in the natural way, i.e., $\bx, \by \in \mathbb{S}^2 \subset \R^3$.  Then for any $\bx, \by \in \mathbb{S}^2$, the Riemannian distance between $\bx$ and $\by$ is given by
\[
d_{\mM}(\bx, \by) = r\cdot\theta, \quad \theta = \text{acos} \bigg(\frac{\inprod{\bx, \by}}{\norm{\bx}\cdot\norm{\by}}\bigg).
\]
The minimal geodesic between $\bx$ and $\by$ is the piece of the arc on the great circle of $\mathbb{S}^2$ with the smallest length, assuming $\bx$ and $\by$ are not in each others' cut locus, i.e. diametrically opposed.  The unit vector on the minimal geodesic from $\bx$ to $\by$, denoted as $\bv(\bx, \by)$, can be computed as follows
\[
\bv(\bx, \by) = \frac{\by - \bx - \text{Proj}_{-\bx}(\by - \bx)}{\norm{\by - \bx - \text{Proj}_{-\bx}(\by - \bx)}}.
\]
Here $\text{Proj}_{\bu}(\bw)$ is the projection of $\bw$ onto $\bu$.  
\subsection{Riemannian Geometry on the Poincar\'{e} Disk}
For any two points $\bx, \by \in \mathbb{PD}$ on the Poincar\'{e} Disk ($\mathbb{PD}$) where $\mathbb{PD} \coloneqq \{\bx \in \R^2  \text{ s.t. } \norm{\bx} < 1\}$, the Riemannian metric, written in the standard coordinates of $\R^2$, is given by
\[
g_{i, j}(\bx) = \frac{4\delta_{i, j}}{(1 - \norm{\bx}^2)^2}, \quad \bx \in \mathbb{PD}\,,
\]
with $\delta_{i, j}$ being the Kronecker delta, and the corresponding Riemannian distance between $\bx$ and $\by$ is
\[
d_{\mM}(\bx, \by) = \text{acosh}\bigg(1 + \frac{\norm{\bx - \by}^2}{(1 - \norm{\bx}^2)(1 - \norm{\by}^2)}\bigg)\,.
\]
The minimal geodesics between $\bx$ and $\by$ are either straight line segments if $\bx$ and $\by$ are on a line through the origin or circular arc perpendicular to the boundary.  For the straight line segment case, we have the unit vector on the minimal geodesic from $\bx$ to $\by$, denoted as $\bv(\bx, \by)$, computed as follows: we identify the vector $\by - \bx$, computed in $\R^2$ as a tangent vector in $T_{\bx}\mM$, then normalize it to obtain $\bv(\bx, \by) = \frac{\by - \bx}{\norm{\by - \bx}_{T_{\bx}\mM}}$.  For the perpendicular arc case, we first find the inverse $\by'$ of $\by$ w.r.t to the unit disk (in $\R^2$); then we use the three points $\bx, \by, \by'$ to find the center $\bo'$ of the circle passing through $\bx, \by$ and $\by'$. Then the unit tangent vector on the geodesic from $\bx$ to $\by$ is computed as follows: , we compute $\by - \bx - \text{Proj}_{\bo' -\bx}(\by - \bx)$ in $\R^2$ (with the Euclidean metric), then identify it as a tangent vector in $T_{\bx}\mM$, and normalize it:
\[
\bv(\bx, \by) = \frac{\by - \bx - \text{Proj}_{\bo' -\bx}(\by - \bx)}{\norm{\by - \bx - \text{Proj}_{\bo' -\bx}(\by - \bx)}_{T_{\bx}\mM}}\,.
\]
\section{Learning Theory: Foundation}\label{sec:learn_proof}
In this section, we present the theoretical foundation needed to prove the theorems presented in the main body.  We follow the ideas presented in \cite{Lu2019a} with similar strategies presented in \cite{cucker2002mathematical, gyorfi2006distribution}.  We begin with the following assumption.
\begin{assumption}
$\mH$ is a compact (in $L^{\infty}$-norm) and convex subset of $L^2([0,R])$, such that every $\intkernelvar \in \mH$ is bounded above by some constant $S_0 \geq S$, i.e. $\norm{\intkernelvar}_{L^\infty([0,R])}\le S_0$; moreover $\intkernelvar$ is smooth enough to ensure the existence and uniqueness of solutions of 
\begin{equation}\label{eq:mainmodel SI}
\dot\bx_i(t) = \frac{1}{N}\sum_{i' = 1}^{N}\intkernel(d_{\mM}(\bx_i(t), \bx_{i'}(t)))\bw(\bx_i(t),\bx_{i'}(t)), \qquad i = 1, \ldots, N.
\end{equation}
for $t \in [0, T]$, i.e. $\intkernelvar \in \mH \cap \mK_{R, S_0}$.
\end{assumption}
Another important observation is that since $\intkernel \in \mK_{R, S}$ and $T$ is finite, the distribution of $\bx_i(t)$'s does not blow up over $[0,T]$ ensuring that the $\bx_i(t)$'s have bounded distance from the $\bx_i(0)$'s.  
In fact, let $R_0$ be the maximum Riemannian distance between any pair of agents at $t = 0$, then 
\[
\max_{i, i' = 1, \ldots, N} r_{i, i'}(t) = \max_{i, i' = 1, \ldots, N} d_{\mM}(\bx_i(t), \bx_{i'}(t)) \le R_0 + TRS, \quad \text{for $t \in [0, T]$}.
\]
Hence the $\bx_i(t)$'s live in a compact (w.r.t to the $d_{\mM}$ metric) ball around the $\bx_i(0)$'s, denoted as $\mB_{\mM}(\bX_0, R_1)$ where $R_1 = R_0 + TRS$.  Recall the definition of the loss functional used to find the estimator, namely $\lintkernel_{L, M, \mH}$ to the unknown interaction kernel $\intkernel$, give by
\begin{equation}\label{error functional SI}
  \mE_{L, M, \mM}(\intkernelvar) := \frac{1}{ML}\sum_{l, m = 1}^{L, M}\norm{\dotbXmtl - \mbf{f}^{\text{c}}_{\intkernelvar}(\bXmtl)}_{T_{\bXmtl}\mM^N}^2\,.
\end{equation}
Further recall that the estimator is defined as 
$\lintkernel_{L, M, \mH} \coloneqq \argmin{\varphi \in \mH}\mE_{L, M, \mM}(\intkernelvar)$.  When $M \rightarrow \infty$, we obtain the following loss functional (by the law of large numbers).
\begin{equation}\label{error functional expected SI}
  \mE_{L, \infty, \mM}(\intkernelvar) := \frac{1}{L}\sum_{l = 1}^{L}\E_{\bX_0 \sim \muX}\Big[\norm{\dotbXtl - \mbf{f}^{\text{c}}_{\intkernelvar}(\bXtl)}_{T_{\bXtl}\mM^N}^2\Big].
\end{equation}
The minimizer of $\mE_{L, \infty, \mM}$ over $\mH$ is defined as $\lintkernel_{L, \infty, \mH}$, which is closely related to $\lintkernel_{L, M, \mH}$ (in the $M \rightarrow \infty$ sense).  And they are close to $\intkernel$, when we establish the following condition on $\mH$.
\begin{definition}[Geometric Coercivity condition]\label{CoercivityCondition SI}
The geometric evolution system in \eqref{eq:mainmodel SI} with initial condition sampled from $\muX$ on $\mM^N$ is said to satisfy the geometric coercivity condition on the admissible hypothesis space $\mH$ if there exists a constant $c_{L, N, \mH, \mM} > 0$ such that for any $\intkernelvar \in \mH$ with $\intkernelvar(\cdot)\cdot \in L^2(\rho_{T, \mM}^L)$, the following inequality holds:
\begin{equation}\label{ccineq SI}
  c_{L, N, \mH, \mM}\norm{\intkernelvar(\cdot)\cdot}_{L^2(\rho_{T,\mM}^L)}^2 \leq \frac{1}{L}\sum_{l = 1}^{L}\E_{\bX_0 \sim \muX}\Big[\norm{\mbf{f}^{\text{c}}_{\intkernelvar}(\bXtl)}_{T_{\bXtl}\mM^N}^2\Big].
\end{equation}
\end{definition}
From this condition, we can derive the following theorem. \begin{theorem}\label{control bias SI}
Let $\intkernel \in L^2([0,R])$, and $\mH$ a compact (w.r.t the $L^{\infty}$ norm) and convex subset of $ L^2([0,R])$ such that the geometric coercivity condition \eqref{ccineq SI} holds with a constant $c_{L, N,\mH, \mM}$. Then, for $\lintkernel_{L, M, \mH}$, estimated by minimizing \eqref{error functional SI} on the trajectory data generated by \eqref{eq:mainmodel SI}, the following inequality
\begin{equation}\label{estimate of norm difference}
\norm{\lintkernel_{L, M, \mH}(\cdot)\cdot - \intkernel(\cdot)\cdot}_{L^2(\rho_{T, \mM}^L)}^2 \le \frac{2}{c_{L, N,\mH, \mM}}\Big(\epsilon + \inf\limits_{\intkernelvar \in \mH}\norm{\intkernelvar(\cdot)\cdot - \intkernel(\cdot)\cdot}_{L^2(\rho_{T, \mM}^L)}^2\Big)
\end{equation}
holds with probability at least $1 - \tau$, when $M \geq \frac{1152S_0^2R^2}{\epsilon c_{L, N, \mH, \mM}}\Big(\ln(\mN(\mH, \frac{\epsilon}{48S_0R^2})) + \ln(\frac{1}{\tau})\Big)$. Here $\mN(\mU, \epsilon)$ is the covering number of a set $\mU$ with open balls of radius $\epsilon$ w.r.t the $L^{\infty}$-norm.
\end{theorem}
%We would like to present one more theorem on consistency of the estimator, namely $\lintkernel_{L, M, \mH_M}$.
Using this concentration result, we can get the strong consistency of our estimators under mild hypotheses.  
\begin{theorem}\label{consistency SI}
For a family of compact (w.r.t. the $L^{\infty}$ norm) convex subsets, $\{\mH_M\}_{M = 1}^{\infty}$, of $L^2([0,R])$, when the following conditions hold, (i) $\cup_M \mH_M$ is compact in $L^{\infty}$; (ii) the geometric coercivity condition, \eqref{CoercivityCondition SI}, holds on $\cup_M \mH_M$;
(iii) $\inf\limits_{\intkernelvar \in \mH_M} \norm{\intkernelvar(\cdot)\cdot - \intkernel(\cdot)\cdot}_{L^2(\rho_{T, \mM}^L)} \stackrel{M\to\infty}{\longrightarrow} 0$,
then 
\begin{equation}
\lim_{M \to \infty}\norm{\lintkernel_{L, M, \mH_M}(\cdot)\cdot - \intkernel(\cdot)\cdot}_{L^2(\rho_{T, \mM}^L)}  = 0 \hspace{0.5 cm} a.s.
\end{equation}  
This theorem establishes the almost sure convergence of our estimator to the true interaction kernel as $M \to \infty$. 
\end{theorem}
\subsection{Concentration and Consistency}
Our first step is to establish the consistency of the estimator for the true kernel $\intkernel$ of the system.  Note that $\mathcal{H}$ can be embedded as a compact (in $L^{\infty}$ sense) set of $L^2(\rho_{T, \mM}^L)$. We establish a strong consistency result on our estimators of the form,
\[
\lim\limits_{M \to \infty}\norm{\lintkernel_{L, M}(\cdot)\cdot - \intkernel(\cdot)\cdot}_{L^2(\rho_{T, \mM}^L)} = 0, \, a.s.
\] 
Our discussions of consistency under the $L^2-$norm on manifolds can be regarded as a natural extension from the case on Euclidean Space in \cite{Lu2019a}.  We define the following loss functional of the vectorized system, $\bX_t$ 
\begin{align}\label{loss at t}
\mE_{\bX_t}(\intkernelvar) &:= \frac{1}{N}\sum_{i = 1}^N\norm{\frac{1}{N}\sum_{i' = 1}^N(\intkernel_{ii', t} - \intkernelvar_{ii', t})\bw_{ii', t}}^2_{T_{\bx_i(t)}\mM} \nonumber \\
                            &= \frac{1}{N}\sum_{i = 1}^N\inprod{\frac{1}{N}\sum_{i' = 1}^N(\intkernel_{ii', t} - \intkernelvar_{ii', t})\bw_{ii', t}, \frac{1}{N}\sum_{i'' = 1}^N(\intkernel_{ii'', t} - \intkernelvar_{ii'', t})\bw_{ii'', t}}_{g(\bx_i(t))}.
\end{align}
Here we take $\bw_{ii', t} = d_{\mM}(\bx_i(t), \bx_{i'}(t))\bv(\bx_i(t), \bx_{i'}(t))$ and $\intkernel_{ii', t} = \intkernel(d_{\mM}(\bx_i(t), \bx_{i'}(t))$; similarly for $\intkernelvar_{ii', t}$.
Now we can see that 
\[
\mE_{L, M,\mM}(\intkernelvar) = \frac{1}{LM}\sum_{l, m = 1}^{L, M}\mE_{\bX^{m}_{t_l}}(\intkernelvar).
\]
When $M \to \infty$, we have, by the law of large numbers,
\[
\mE_{L, \infty, \mM}(\intkernelvar) = \frac{1}{L}\sum_{l = 1}^L\E_{\bX_0 \sim \muX}\mE_{\bX_{t_l}}(\intkernelvar)\,.
\]
We are ready to summarize some basic properties of $\mE_{\bX_t}(\intkernelvar)$. 
\begin{prop}\label{prop:continuity}
For $\intkernelvar_1, \intkernelvar_2 \in \mH$, we have
\begin{equation}\label{eq:continuity}
\abs{\mE_{\bX_t}(\intkernelvar_1) - \mE_{\bX_t}(\intkernelvar_2)} \leq \norm{\intkernelvar_1(\cdot)\cdot - \intkernelvar_2(\cdot)\cdot}_{L^2(\hat\rho^t_{\mM})}\norm{2\intkernel(\cdot)\cdot - \intkernelvar_1(\cdot)\cdot - \intkernelvar_2(\cdot)\cdot}_{L^2(\hat\rho^t_{\mM})}.
\end{equation}
Here we define the probability measure, $\hat{\rho}^t_{\mM}(r) \coloneqq \frac{1}{N^2}\sum_{i, i' =1}^{N}\delta_{d_{\mM}(\bx_i(t), \bx_{i'}(t))}(r)$.
\end{prop}
\begin{proof}
Let $\intkernelvar_1, \intkernelvar_2 \in \mH$, and define $\intkernelvar_{ii', t}^1 \coloneqq \intkernelvar_1(d_{\mM}(\bx_{i}(t), \bx_{i'}(t)))$, similarly for $\intkernelvar_{ii', t}^2$.  Moreover, let $r_{ii', t} \coloneqq d_{\mM}(\bx_i(t), \bx_{i'}(t))$ and $\bw_{ii', t} \coloneqq d_{\mM}(\bx_i(t), \bx_{i'}(t))\bv(\bx_i(t), \bx_{i'}(t))$.  Immediately, we have
\[
\norm{\bw_{ii', t}}_{T_{\bx_i(t)}\mM} \le r_{ii', t},
\]
since $\bv(\bx_i(t), \bx_{i'}(t))$ has either length $1$ or $0$.  Next, using Jensen's inequality, we have 
\begin{align*}
\abs{\mE_{\bX_t}(\intkernelvar_1) - \mE_{\bX_t}(\intkernelvar_2)}
&= \abs{\frac{1}{N}\sum_{i = 1}^N\inprod{\frac{1}{N}\sum_{i' = 1}^N(\intkernelvar_{ii', t}^1 - \intkernelvar_{ii', t}^2)\bw_{ii', t}, \frac{1}{N}\sum_{i'' = 1}^N(2\intkernel_{ii'', t} - \intkernelvar_{ii'', t}^1 - \intkernelvar_{ii'', t}^2)\bw_{ii'', t}}_{g(\bx_i(t))}}\\
&\le \scalebox{0.9}{$\frac{1}{N}\sum_{i = 1}^N\norm{\frac{1}{N}\sum_{i' = 1}^N(\intkernelvar_{ii', t}^1 - \intkernelvar_{ii', t}^2)\bw_{ii', t}}_{T_{\bx_i(t)}\mM}  \norm{\frac{1}{N}\sum_{i'' = 1}^N(2\intkernel_{ii'', t} - \intkernelvar_{ii'', t}^1 - \intkernelvar_{ii'', t}^2)\bw_{ii'', t}}_{T_{\bx_i(t)}\mM}$} \\
&\le \sqrt{\frac{1}{N^2}\sum_{i, i' = 1}^N(\intkernelvar_{ii', t}^1 - \intkernelvar_{ii', t}^2)r_{ii', t}^2}\sqrt{\frac{1}{N^2}\sum_{i, i'' = 1}^N(2\intkernel_{ii'', t} - \intkernelvar_{ii', t}^1 - \intkernelvar_{ii', t}^2)r_{ii'', t}^2} \\
&\le\norm{\intkernelvar_1(\cdot)\cdot - \intkernelvar_2(\cdot)\cdot)}_{L^2(\hat\rho^t_{\mM})}\norm{2\intkernel(\cdot)\cdot - \intkernelvar_1(\cdot)\cdot - \intkernelvar_2(\cdot)\cdot}_{L^2(\hat\rho^t_{\mM})}, 
\end{align*}
where 
$
\hat\rho^t_{\mM}(r) = \frac{1}{N^2}\sum_{i, i' = 1}^N\delta_{r_{ii', t}}(r)
$.
\end{proof}
With Proposition \ref{prop:continuity} proven, we get the following proposition establishing the continuity of our error functionals.
\begin{prop}
For $\intkernelvar_1, \intkernelvar_2 \in \mH$, we have the inequalities
\begin{equation}\label{ELM}
\begin{aligned}
\abs{\mE_{L, M,\mM}(\intkernelvar_1) - \mE_{L, M,\mM}(\intkernelvar_2)} &\le \norm{\intkernelvar_1(\cdot)\cdot - \intkernelvar_2(\cdot)\cdot}_{L^\infty}\norm{2\intkernel(\cdot)\cdot - \intkernelvar_1(\cdot)\cdot - \intkernelvar_2(\cdot)\cdot}_{L^\infty}\\
\abs{\mE_{L, \infty,\mM}(\intkernelvar_1) - \mE_{L, \infty,\mM}(\intkernelvar_2)} &\leq \norm{\intkernelvar_1(\cdot)\cdot - \intkernelvar_2(\cdot)\cdot}_{L^2(\rho_{T,\mM}^L)}\norm{2\intkernel(\cdot)\cdot - \intkernelvar_1(\cdot)\cdot - \intkernelvar_2(\cdot)\cdot}_{L^2(\rho_{T,\mathcal{M}}^L)}\,.
\end{aligned}
\end{equation}
\end{prop}
\begin{proof}
Using the results from Prop. \ref{prop:continuity}, and defining $\hat\rho_{T, \mM}^L \coloneqq \frac{1}{L}\sum_{l = 1}^L\hat\rho^{t_l}_{\mM}$, we have
\begin{align*}
\abs{\frac{1}{L}\sum_{l = 1}^L \mE_{\bX_{t_l}}(\intkernelvar_1) - \frac{1}{L}\sum_{l = 1}^L \mE_{\bX_{t_l}}(\intkernelvar_2)} 
&\le \frac{1}{L}\sum_{l = 1}^L\abs{\mE_{\bX_{t_l}}(\intkernelvar_1) - \mE_{\bX_{t_l}}(\intkernelvar_2)} \\
&< \frac{1}{L}\sum_{l = 1}^L\norm{\intkernelvar_1(\cdot)\cdot - \intkernelvar_2(\cdot)\cdot}_{L^2(\hat\rho^t_{\mM})}\norm{2\intkernel(\cdot)\cdot - \intkernelvar_1(\cdot)\cdot - \intkernelvar_2(\cdot)\cdot}_{L^2(\hat\rho^t_{\mM})} \\
&\le \sqrt{\frac{1}{L}\sum_{1 = 1}^L\norm{\intkernelvar_1(\cdot)\cdot - \intkernelvar_2(\cdot)\cdot}_{L^2(\hat\rho^t_{\mM})}}\sqrt{\frac{1}{L}\sum_{l = 1}^L\norm{2\intkernel(\cdot)\cdot - \intkernelvar_1(\cdot)\cdot - \intkernelvar_2(\cdot)\cdot}_{L^2(\hat\rho^t_{\mM})}} \\
&= \norm{\intkernelvar_1(\cdot)\cdot - \intkernelvar_2(\cdot)\cdot}_{L^2(\hat\rho_{T, \mM}^L)}\norm{2\intkernel(\cdot)\cdot - \intkernelvar_1(\cdot)\cdot - \intkernelvar_2(\cdot)\cdot}_{L^2(\hat\rho_{T, \mM}^L)}
\end{align*}
Next, we have
\begin{align*}
\abs{\mE_{L, M,\mM}(\intkernelvar_1) - \mE_{L, M,\mM}(\intkernelvar_2)} &\le \frac{1}{M}\sum_{m = 1}^M\abs{\frac{1}{L}\sum_{l = 1}^L \mE_{\bX_{t_l}^{m}}(\intkernelvar_1) - \frac{1}{L}\sum_{l = 1}^L \mE_{\bX_{t_l}^{m}}(\intkernelvar_2)}\\
&\le \frac{1}{M}\sum_{m = 1}^M\norm{\intkernelvar_1(\cdot)\cdot - \intkernelvar_2(\cdot)\cdot}_{L^2(\hat\rho_{T, \mM}^L)}\norm{2\intkernel(\cdot)\cdot - \intkernelvar_1(\cdot)\cdot - \intkernelvar_2(\cdot)\cdot}_{L^2(\hat\rho_{T, \mM}^L)} \\
&\le \norm{\intkernelvar_1(\cdot)\cdot - \intkernelvar_2(\cdot)\cdot}_{L^{\infty}}\norm{2\intkernel(\cdot)\cdot - \intkernelvar_1(\cdot)\cdot - \intkernelvar_2(\cdot)\cdot}_{L^{\infty}}\\
&\le R^2\norm{\intkernelvar_1 - \intkernelvar_2}_{L^{\infty}}\norm{2\intkernel - \intkernelvar_1 - \intkernelvar_2}_{L^{\infty}}\,.
\end{align*}
Meanwhile,  taking $M \rightarrow \infty$ for $\abs{\mE_{L, M,\mM}(\intkernelvar_1) - \mE_{L, M,\mM}(\intkernelvar_2)}$, we obtain
\[
\abs{\mE_{L, \infty,\mM}(\intkernelvar_1) - \mE_{L, \infty,\mM}(\intkernelvar_2)} \le \norm{\intkernelvar_1(\cdot)\cdot - \intkernelvar_2(\cdot)\cdot}_{L^2(\rho_{T,\mathcal{M}}^L)}\norm{2\intkernel(\cdot)\cdot - \intkernelvar_1(\cdot)\cdot - \intkernelvar_2(\cdot)\cdot}_{L^2(\rho_{T,\mathcal{M}}^L)},
\]
where $\rho_{T,\mathcal{M}}^L = \E_{\bX_0 \sim \muX}[\hat\rho_{T, \mM}^L]$.
\end{proof}
As a further derivation, we observe that for any $\intkernelvar \in \mH \subset L^2([0,R])$, we have that $\max_{r \in [0, R]}\abs{\intkernelvar(\cdot)\cdot} \le R\max_{r \in [0, R]}\abs{\intkernelvar(\cdot)}$, so we obtain the following Corollary:
\begin{corollary}\label{coro1}
For $\intkernelvar \in \mH$, define
\[
\mathcal{L}_M(\psi) \coloneqq \mE_{L, \infty, \mM}(\intkernelvar) - \mE_{L, M, \mM}(\intkernelvar),
\]
then for any $\intkernelvar_1, \intkernelvar_2 \in \mH$, we have
\[
\abs{\mathcal{L}_M(\intkernelvar_1) - \mathcal{L}_M(\intkernelvar_2)} \leq 2R^2\norm{\intkernelvar_1 - \intkernelvar_2}_{L^\infty}\norm{2\intkernel - \intkernelvar_1 - \intkernelvar_2}_{L^\infty}.
\]
\end{corollary}
Now we can consider the distance between the minimizer of the error functional $\mE_{L,\infty, \mM}$ over $\mH$ and any other $\intkernelvar \in \mH$.  Let 
\[
\lintkernel_{L,\infty,\mH} = \argmin{\intkernelvar \in \mathcal{H}} \mE_{L,\infty, \mM}(\intkernelvar).
\]
From the geometric coercivity condition and the convexity of $\mH$, we obtain
\begin{prop}\label{uniqueness}
For any $\intkernelvar \in \mH$, 
\begin{equation}\label{uniqueness inequality}
\mE_{L,\infty, \mM}(\intkernelvar) - \mE_{L,\infty, \mM}(\lintkernel_{L,\infty,\mH}) \geq c_{L, N, \mH, \mM}\norm{\intkernelvar(\cdot)\cdot - \lintkernel_{L,\infty,\mH}(\cdot)\cdot}_{L^2(\rho_{T, \mM}^L)}.
\end{equation}
\end{prop}			
We now define the defect function $\mD_{L,M,\mH}(\intkernelvar) \coloneqq \mE_{L,M, \mM}(\intkernelvar) - \mathcal{E}_{L, M, \mM}(\lintkernel_{L,\infty,\mH})$, and define
\[
\mD_{L,\infty,\mH}(\intkernelvar) \coloneqq \lim\limits_{M \to \infty} \mD_{L,M, \mH}(\intkernelvar) = \mE_{L, \infty, \mH}(\intkernelvar) - \mE_{L, \infty, \mM}(\lintkernel_{L,\infty,\mH}).
\]
Then, we show that we can uniformly bound $\frac{\mD_{L, \infty, \mH}(\cdot) - \mD_{L, M, \mH}(\cdot)}{\mD_{L, \infty, \mH}(\cdot) + \epsilon}$ on $\mH$ with high probability,
\begin{prop}\label{uniform bound}
For any $\epsilon > 0$ and $\alpha \in (0,1)$, we have
\[
\mathbb{P}_{\muX} \bigg(\sup\limits_{\intkernelvar \in \mH} \frac{\mD_{L, \infty, \mH}(\intkernelvar) - \mD_{L, M, \mH}(\intkernelvar)}{\mD_{L,\infty,\mH}(\intkernelvar) + \epsilon} \geq 3\alpha \bigg) \leq \mN\Big(\mathcal{H}, \frac{\alpha\epsilon}{8S_0R^2}\Big)\exp\bigg(-\frac{c_{L,N,\mH, \mM}\alpha^2M\epsilon}{32S_0^2}\bigg)
\]
where $\mN(U, r)$ is the covering number of set $U$ with open balls of radius $r$ w.r.t the $L^\infty-$norm.
\end{prop}
The proof of Proposition \ref{uniform bound} uses the following Lemma similar to Lemma $19$ in \cite{Lu2019a},
\begin{lemma}\label{lemma defect}
For any $\epsilon > 0$ and $\alpha \in (0,1)$, if $\intkernelvar_1 \in \mH$ satisfies
\[
\frac{\mD_{L, \infty, \mH}(\intkernelvar_1) - \mD_{L, M, \mH}(\intkernelvar_1)}{\mD_{L,\infty,\mH}(\intkernelvar_1) + \epsilon} < \alpha
\]
then for any $\intkernelvar_2 \in \mH$ s.t. $\norm{\intkernelvar_1 - \intkernelvar_2}_{L^\infty} \leq r_0 = \frac{\alpha\epsilon}{8S_0R^2}$, we have
\[
\frac{\mD_{L, \infty, \mH}(\intkernelvar_2) - \mD_{L, M, \mH}(\intkernelvar_2)}{\mD_{L, \infty, \mH}(\intkernelvar_2) + \epsilon} < 3\alpha
\]
\end{lemma}
Using the results we have just established, the proofs of theorems \ref{control bias SI} and \ref{consistency SI} now follow similarly to the analogous results in \cite{Lu2019a, Tang2019, miller2020learning}.
\subsection{Rate of Convergence}
Using these results, we establish the convergence rate of $\lintkernel_{L, M, \mH}$ to $\intkernel$ as $M$ increases. 		
\begin{theorem}\label{thm:optimal rate of convergence}
Let $\muX$ be the distribution of the initial conditions of trajectories, and $\mH_M = \mB_n$ with $n \asymp ({M}/{\log M})^{\frac{1}{2s + 1}}$, where $\mB_{n}$ is the central ball of $\mL_n$ with radius $c_1 + S$, and the linear space $\mL_n \subseteq L^{\infty}([0,R])$ satisfies the dimension and approximation conditions below,
  \[
    dim(\mL_n) \leq c_0n \mand \inf\limits_{\intkernelvar \in \mathcal{L}_n} \norm{\intkernelvar - \intkernel}_{L^{\infty}} \leq c_1n^{-s}
  \]
  for some constants $c_0, c_1, s > 0$. Suppose that the geometric coercivity condition holds on $\mathcal{L}\coloneqq\cup_n \mL_n$ with constant $c_{L, N, \mathcal{L}, \mM}$.  Then there exists some constant $C(S, R, c_0, c_1)$ such that %_{\bX_0 \sim \muX}
  \begin{equation*}
  \E\Big[\norm{\lintkernel_{L, M, \mH_M}(\cdot)\cdot - \intkernel(\cdot)\cdot}_{L^2(\rho_{T, \mM}^L)}\Big] \le \frac{C(S, R, c_0, c_1)}{c_{L, N, \mathcal{L}, \mM}}\Big(\frac{\log M}{M}\Big)^{\frac{s}{2s + 1}}\,.
  \end{equation*}
 \end{theorem}
The proof of the theorem closely follows the ideas in \cite{Lu2019a} and their further development in \cite{Tang2019,miller2020learning}, and is therefore omitted. 
\subsection{Trajectory Estimation Error}
Recall the following theorem on the trajectory estimator error:
\begin{theorem}\label{Traj Acc Bound SI}
Let $\intkernel \in \mK_{R, S}$ and $\lintkernel \in \mK_{R, S_0}$, for some $S_0\ge S$.  Suppose that $\bX_{[0, T]}$ and $\hat\bX_{[0, T]}$ are solutions of \eqref{eq:mainmodel SI} w.r.t to $\intkernel$ and $\lintkernel$, respectively, for $t \in [0, T]$, with $\hat\bX_0=\bX_0$.  Then the following inequalities hold:
\begin{equation}\label{Traj Acc Inequality 1}
  d_{\text{traj},\mM^N}\Big(\bX_{[0, T]},\hat\bX_{[0, T]}\Big)^2 \le 4TC(\mM,T)\exp(64T^2S_0^2)\norm{\dot\bX_t - \mbf{f}^{\text{c}}_{\lintkernel}(\bX_t)}_{T_{\bX_t}M^N}^2 \, ,
\end{equation}
and 
\begin{equation}\label{Traj Acc Inequality 2}
  \E_{\bX_0\sim\muX}\Big[d_{\text{traj},\mM^N}\Big(\bX_{[0, T]},\hat\bX_{[0, T]}\Big)^2\Big] \leq 4T^2C(\mM,T)\exp(64T^2S_0^2)\norm{\intkernel(\cdot)\cdot - \lintkernel(\cdot)\cdot}_{L^2(\rho_{T, \mM})}^2,
\end{equation}
where $C(\mM,T)$ is a positive constant depending only on geometric properties of $\mM$ and $T$, but may be chosen independent of $T$ if $\mM$ is compact.
\end{theorem}
It states two different estimates of the trajectory estimation error. First, it bounds the system trajectory error for any one single initial condition; second, it bounds the expectation of the worst trajectory estimation error on time interval $[0,T]$ among all different initial conditions.
\begin{proof}[Proof of Theorem \ref{Traj Acc Bound SI}]
Assume that $\intkernel \in \mK_{R, S}$, $\lintkernel \in \mK_{R, S_0}$, and $\bX_t, \hat\bX_t$ are two system states generated by $\intkernel, \lintkernel$ with the same initial conditions at some $t \in [0, T]$.  Next, we assume that $\mM$ is isometrically embedded in $\R^{d'}$ (at least one such embedding exists, by Nash's embedding theorem), via a map $\mI:\mM \rightarrow \R^{d'}$.   From now on, we will identify $\bx_i$ with $\mI\bx_i$.
Then for any $t\in [0,T]$, we have
\begin{align*}
\frac{1}{N}\sum_{i = 1}^N\norm{\bx_i(t) - \hat\bx_i(t)}^2_{\R^{d'}} &= \frac{1}{N}\sum_{i = 1}^N\norm{\int_{s = 0}^t(\dot\bx_i(s) - \dot{\hat{\bx}}_i(s))\, ds}^2_{\R^{d'}} \le  \frac{1}{N}\sum_{i = 1}^Nt\int_{s = 0}^t\norm{\dot\bx_i(s) - \dot{\hat{\bx}}_i(s)}^2_{\R^{d'}}\, ds \\
&\le \frac{T}{N}\sum_{i = 1}^N\int_{s = 0}^t\norm{\dot\bx_i(s) - \dot{\hat{\bx}}_i(s)}^2_{\R^{d'}}\, ds.
\end{align*}
Define the function $F^{\mM}_{\intkernelvar}(\bx, \cdot): \mM \rightarrow T_{\bx}\mM$ for every $\bx \in \mM$ as $F^{\mM}_{\intkernelvar}(\bx,  \cdot) \coloneqq \intkernelvar(d_{\mM}(\bx, \cdot))\bw(\bx, \cdot)$.  Let $F^{\mM}_{\intkernelvar, ii', t} = F^{\mM}_{\intkernelvar}(\bx_i(t), \bx_{i'}(t))$ and $F^{\mM}_{\intkernelvar, \hat{i}\hat{i'}, t} = F^{\mM}_{\intkernelvar}(\hat\bx_i(t), \hat\bx_{i'}(t))$.  Then
\begin{align*}
\sum_{i = 1}^N\int_{s = 0}^t\norm{\dot\bx_i(s) - \dot{\hat{\bx}}_i(s)}^2_{\R^{d'}}\, ds &= \sum_{i = 1}^N\int_{s = 0}^t\norm{\dot\bx_i(s) - \frac{1}{N}\sum_{i' = 1}^NF^{\mM}_{\lintkernel, \hat{i}\hat{i'}, s}}^2_{\R^{d'}}\, ds \\
&\le \scalebox{0.85}{$2\sum_{i = 1}^N\int_{s = 0}^t\Big(\norm{\dot\bx_i(s) - \frac{1}{N}\sum_{i' = 1}^NF^{\mM}_{\lintkernel, ii', s}}_{\R^{d'}}^2 + 
\norm{\frac{1}{N}\sum_{i' = 1}^NF^{\mM}_{\lintkernel, ii', s} - \frac{1}{N}\sum_{i' = 1}^NF^{\mM}_{\lintkernel, \hat{i}\hat{i'}, s}}_{\R^{d'}}^2\Big) \, ds$} \\
&= 2\sum_{i = 1}^N\int_{s = 0}^t\Big(\norm{\dot\bx_i(s) - \frac{1}{N}\sum_{i' = 1}^NF^{\mM}_{\lintkernel, ii', s}}_{\R^{d'}}^2 + I(s)\Big) \, ds\,.
\end{align*}
Next,
\begin{align*}
I(s) &= \norm{\frac{1}{N}\sum_{i' = 1}^NF^{\mM}_{\lintkernel, ii', s} - \frac{1}{N}\sum_{i' = 1}^NF^{\mM}_{\lintkernel, \hat{i}\hat{i'}, s}}_{\R^{d'}}^2 
= \frac{1}{N^2}\norm{\sum_{i' = 1}^N\big(F^{\mM}_{\lintkernel, ii', s} - F^{\mM}_{\lintkernel, i\hat{i'}, s} + F^{\mM}_{\lintkernel, i\hat{i'}, s} - F^{\mM}_{\lintkernel, \hat{i}\hat{i'}, s}\big)}_{\R^{d'}}^2 \\
&\le \frac{2}{N^2}\Big(\norm{\sum_{i' = 1}^N\big(F^{\mM}_{\lintkernel, ii', s} - F^{\mM}_{\lintkernel, i\hat{i'}, s}\big)}_{\R^{d'}}^2 + \norm{\sum_{i' = 1}^N\big(F^{\mM}_{\lintkernel, i\hat{i'}, s} - F^{\mM}_{\lintkernel, \hat{i}\hat{i'}, s}\big)}_{\R^{d'}}^2\Big)\,.
\end{align*}
Since $\lintkernel \in \mK_{R, S_0}$, $F^{\mM}_{\lintkernel}$ is Lipschitz in each of its arguments; moreover, $\max_{r \in [0, R]}\abs{\lintkernel} \leq S_0$, so that $\text{Lip}(F^{\mM}_{\lintkernel}(\bx, \cdot))$, $\text{Lip}(F^{\mM}_{\lintkernel}(\cdot, \bx)) \le 2S_0$. Therefore,
\begin{align*}
I(s) &\le \frac{2}{N^2}\Big(2\text{Lip}(F^{\mM}_{\lintkernel}(\bx_i(s), \cdot))^2\sum_{i' = 1}^N\norm{\bx_{i'}(s) - \hat\bx_{i'}(s)}_{\R^{d'}}^2 + 2\sum_{i' = 1}^N\text{Lip}(F^{\mM}_{\lintkernel}(\cdot, \hat\bx_{i'}(s)))^2\norm{\bx_i(s) - \hat\bx_i(s)}_{\R^{d'}}^2\Big)\\
&\le  \frac{4}{N^2}\text{Lip}(F^{\mM}_{\lintkernel}(\bx_i(s), \cdot))^2\sum_{i' = 1}^N\norm{\bx_{i'}(s) - \hat\bx_{i'}(s)}^2_{\R^{d'}}  + \frac{4}{N^2}\sum_{i' = 1}^N\text{Lip}(F^{\mM}_{\lintkernel}(\cdot, \hat\bx_{i'}(s)))^2\norm{\bx_{i'}(s) - \hat\bx_{i'}(s)}^2_{\R^{d'}} \\
&\le \frac{16S_0^2}{N^2}\sum_{i' = 1}^N\norm{\bx_{i'}(s) - \hat\bx_{i'}(s)}^2_{\R^{d'}} + \frac{16S_0^2}{N^2}\sum_{i' = 1}^N\norm{\bx_{i'}(s) - \hat\bx_{i'}(s)}^2_{\R^{d'}}\\
&\le \frac{32S_0^2}{N^2}\sum_{i' = 1}^N\norm{\bx_{i'}(s) - \hat\bx_{i'}(s)}^2_{\R^{d'}}\,.
\end{align*}
Putting these results together, we have
\begin{align*}
\frac{1}{N}\sum_{i = 1}^N\norm{\bx_i(t) - \hat\bx_i(t)}^2_{\R^{d'}} &\le \frac{2T}{N}\sum_{i = 1}^N\int_{s = 0}^t\Big(\norm{\dot\bx_i(s) - \frac{1}{N}\sum_{i' = 1}^NF^{\mM}_{\lintkernel, ii', s}}_{\R^{d'}}^2 + \frac{32S_0^2}{N^2}\sum_{i' = 1}^N\norm{\bx_{i'}(s) - \hat\bx_{i'}(s)}^2_{\R^{d'}}\Big) \, ds \\
&= \frac{64TS_0^2}{N}\sum_{i = 1}^N\norm{\bx_i(t) - \hat\bx_i(t)}^2_{\R^{d'}} + \frac{2T}{N}\sum_{i = 1}^N\int_{s = 0}^t\norm{\dot\bx_i(s) - \frac{1}{N}\sum_{i' = 1}^NF^{\mM}_{\lintkernel, ii', s}}_{\R^{d'}}^2\, ds.
\end{align*}
By Gr\"{o}nwall's inequality, we have
\[
\frac{1}{N}\sum_{i = 1}^N\norm{\bx_i(t) - \hat\bx_i(t)}^2_{\R^{d'}} \le \frac{2T}{N}\exp(64T^2S_0^2)\sum_{i = 1}^N\int_{s = 0}^t\norm{\dot\bx_i(s) - \frac{1}{N}\sum_{i' = 1}^NF^{\mM}_{\lintkernel, ii', s}}_{\R^{d'}}^2\, ds.
\]
\noindent
Recall that $T$ is small, hence the solution $\bX_t$ and $\hat\bX_t$ live in a compact neighborhood of the initial condition, $\bX_0 = \hat\bX_0 \in \mM^N$; i.e. $\bX_t, \hat\bX_t \in \mB_{\mM}(\bX_0, R_2)$ with $R_2 = R_0 + TRS_0$.  From the compactness of (the closure of) this set, and via the embedding $\mI$, we deduce that there exists a constant $C_1(\mM,\mI,T)$ such that
\[
d_{\mM}(\bx_i(t), \hat\bx_i(t)) \le C_1(\mM, \mI,T)\norm{\bx_i(t) - \hat\bx_i(t)}_{\mathbb{R}^{d'}}, \quad \text{for $t \in [0, T]$.}
\]			
Since $\mI$ is isometric, for $\bu \in T_{\bx}\mM$ we have $\norm{d\mI(\bu)}_{\mathbb{R}^{d'}} =\norm{\bu}_{T_{\bx}\mM}$.  Using both the bounds above, we have
\begin{align*}
d_{\mM}(\bX_t, \hat\bX_t)^2 &= \frac{1}{N}\sum_{i = 1}^Nd_{\mM}(\bx_i(t), \hat\bx_i(t))^2 \le \frac{C_1(\mM, \mI,T)^2}{N}\sum_{i = 1}^N\norm{\bx_i(t) - \hat\bx_i(t)}^2_{\R^{d'}}  \\
&\le \frac{2C_1(\mM, \mI,T)^2T\exp(64T^2S_0^2)}{N}\sum_{i = 1}^N\int_{s = 0}^t\norm{\dot\bx_i(s) - \frac{1}{N}\sum_{i' = 1}^NF^{\mM}_{\lintkernel, ii', s}}_{\R^{d'}}^2\, ds.\\
&= \frac{2C_1(\mM, \mI,T)^2T\exp(64T^2S_0^2)}{N}\sum_{i = 1}^N\int_{s = 0}^t\norm{\dot\bx_i(s) - \frac{1}{N}\sum_{i' = 1}^NF^{\mM}_{\lintkernel, ii', s}}_{T_{\bx_i(s)}\mM}^2\, ds \\
&= 2C_1(\mM, \mI,T)^2T\exp(64T^2S_0^2)\int_{s = 0}^t\norm{\dot\bX_s - \mbf{f}^{\text{c}}_{\lintkernel}(\bX_s)}_{T_{\bX_s}\mM^N}^2\, ds
\end{align*}
Letting 
\[
C(\mM,T) := \inf_{\text{all isometric embeddings $\mI$}}C_1(\mM, \mI,T)^2\,,
\]
and choosing an isometric embedding $\mathcal{I}$ which gives a value at most twice the infimum, we obtain
\[
d_{\mM}(\bX_t, \hat\bX_t)^2 \le 4TC(\mM,T)\exp(64T^2S_0^2)\int_{s = 0}^t\norm{\dot\bX_s - \mbf{f}^{\text{c}}_{\lintkernel}(\bX_s)}_{T_{\bX}\mM^N}^2 \, ds.
\]	
Now, take $\intkernel$ to be the true interaction kernel, and $\lintkernel$ the estimator of $\intkernel$ by our learning approach, by Prop. \ref{prop:continuity} we have that
%\MM{: by [add ref here to the inequality used]}
\[
\frac{1}{T}\int_{t = 0}^{T}\norm{\dot\bX_s - \mbf{f}^{\text{c}}_{\lintkernel}(\bX_s)}_{T_{\bX}\mM^N}^2 \, dt \le \norm{\intkernel(\cdot)\cdot - \lintkernel(\cdot)\cdot}_{L^2(\rho_{T, \mM})}^2.
\]
Together with \eqref{Traj Acc Inequality 1}, recalling that $\hat\bX_0 = \bX_0$ and $\bX_0\sim \muX$, we have the desired result that
\[
\E_{\bX_0\sim\muX}\Big[d_{\text{traj}, \mM}(\bX_{[0, T]}, \hat\bX_{[0, T]})^2\Big] \le 4T^2C(\mM,T)\exp(64T^2S_0^2) \E_{\bX_0\sim\muX} \norm{\intkernel(\cdot)\cdot - \lintkernel(\cdot)\cdot}_{L^2(\rho_{T, \mM})}^2.
\]
\end{proof}
\section{Numerical Implementations}\label{sec:numerics}
If the trajectory data, $\{\bx_i^{m}(t_l), \dot\bx_i^{m}(t_l)\}_{i, l, m = 1}^{N, L, M}$, is given by the user, we use the following geometry-based algorithm to find the minimizer of \eqref{error functional SI}.  First, we construct a finite dimensional subspace of the hypothesis space, i.e. $\hypspace_M \subset \hypspace$, where $\hypspace_M$ with dimension $\dim(\mH_M) = n = n(M) \approx \mathcal{O}(M^{\frac{1}{3}})$ is a space of clamped B-spline functions\footnote{Other type of basis functions can be considered, such as piecewise polynomials, Fourier, etc.} supported on $[R_{\min}^{\text{obs}}, R_{\max}^{\text{obs}}]$ with $R_{\min}^{\text{obs}}$/$R_{\max}^{\text{obs}}$ being the minimum/maximum interaction radius computed from the observation data.  Hence the test functions can be expressed as linear combination of the basis functions of $\hypspace_M$, i.e., $\intkernelvar(r) = \sum_{\eta = 1}^n \alpha_{\eta}\basis_{\eta}(r)$ with $\{\basis_{\eta}\}_{\eta = 1}^n$ being a basis for $\hypspace_M$.  Next, we use either a local chart $\mU: \mM \rightarrow \R^{\sdim}$ or a natural embedding $\mI: \mM \rightarrow \R^{d'}$, such that $\bx_i \in \mM$ can be expressed using either local coordinates in $\R^{\sdim}$ (as in the $\mathbb{PD}$ case) or global coordinates in $\R^{d'}$ (as in the $\mathbb{S}^2$ case).  The computation of $\inprod{\cdot, \cdot}_{g(\bx)}$ will be based on the choice of the local chart, or on the embedding, accordingly.  Then, we define a basis matrix, $\Psi^{m} \in (T_{\bX^{m}_{t_1}}\mM^N \times \cdots \times T_{\bX^{m}_{t_L}}\mM^N)^n$, whose columns are
\[
\Psi^{m}(:, \eta) = \Psi^{m}_{\eta} = \frac{1}{\sqrt{N}}\begin{bmatrix} \mbf{f}^{\text{c}}_{\basis_{\eta}}(\bX^{m}_{t_1}) \\ \vdots \\ \mbf{f}^{\text{c}}_{\basis_{\eta}}(\bX^{m}_{t_L}) \end{bmatrix} \in T_{\bX^{m}_{t_1}}\mM^N \times \cdots \times T_{\bX^{m}_{t_L}}\mM^N,
\]
recall
\[
\mbf{f}^{\text{c}}_{\intkernelvar}(\bX_t) = \begin{bmatrix} \vdots \\ \frac{1}{N}\sum_{i' = 1}^{N}\intkernelvar(d_{\mM}(\bx_i(t), \bx_{i'}(t)))\bw(\bx_i(t), \bx_{i'}(t))\\ \vdots \end{bmatrix} \in T_{\bX_t}\mM^N.
\]
Next, we define the derivative vector, $\vec{d}^{m} \in T_{\bX^{m}_{t_1}}\mM^N \times \cdots \times T_{\bX^{m}_{t_L}}\mM^N$, as follows,
\[
\vec{d}^{m} = \frac{1}{\sqrt{N}}\begin{bmatrix} \dot\bX^{m}_{t_1} \\ \vdots \\ \dot\bX^{m}_{t_L} \end{bmatrix}.
\]
Then, we define the learning matrix $A_M \in \R^{n \times n}$ as follows
\[
A_M(\eta, \eta') = \frac{1}{LM}\sum_{m = 1}^m\inprod{\Psi^{m}_{\eta}, \Psi^{m}_{\eta'}}_G, \quad \text{for $\eta, \eta' = 1, \ldots, n$}.
\]
Here the inner product $\inprod{\cdot, \cdot}_G$ on $\Psi^{m}_{\eta} \in T_{\bX^{m}_{t_1}}\mM^N \times \cdots \times T_{\bX^{m}_{t_L}}\mM^N$ is defined as
\[
\inprod{\Psi^{m}_{\eta}, \Psi^{m}_{\eta'}}_{G} = \sum_{l = 1}^L\inprod{\mbf{f}^{\text{c}}_{\psi_{\eta}}(\bX^{m}_{t_l}), \mbf{f}^{\text{c}}_{\psi_{\eta'}}(\bX^{m}_{t_l})}_{g^{\mM^N}(\bX_l^{m})}. 
\]
Next for the learning right hand side, $\vec{b}_M \in \R^{n \times 1}$, we have
\[
\vec{b}_M(\eta) = \frac{1}{LM}\sum_{m = 1}^m\inprod{\vec{d}, \Psi^{m}_{\eta}}_G, \quad \text{for $\eta = 1, \ldots, n$}
\]
Therefore, the minimization of \eqref{error functional SI} over $\hypspace_M$ can be rewritten as
\[
A_M\vec{\alpha} = \vec{b}_M, \quad \vec{\alpha} = \begin{bmatrix}\alpha_1 \\ \vdots \\ \alpha_n \end{bmatrix} \in \R^{n \times 1}.
\]
$A_M$ is symmetric positive definite (guaranteed by the geometric coercivity condition), hence we can solve the linear system to obtain $\hat{\vec{\alpha}}$, and assemble
\[
\lintkernel(r) = \sum_{\eta = 1}^n \hat{\alpha}_{\eta}\basis_{\eta}(r).
\]
In order to produce unique solution of \eqref{eq:mainmodel SI} using $\lintkernel$, we smooth out $\lintkernel$ for the evolution of the dynamics.

If the trajectory data is not given, we will generate it using a Geometric Numerical Integrator, which is a fourth order Backward Differentiation Formula (BDF) of fixed time step size $h$ combined with a projection scheme.  For details see \cite{HLW2006}.  Once a reasonable evolution of the dynamics is obtained, we observe it at $0 = t_1 < \ldots < t_L = T$ to obtain a set of trajectory data, and use it as training data to input to the learning algorithm.   The observation times do not need to be aligned with the numerical integration times, i.e. where numerical solution of $\{\bx_i^{m}(t), \dot\bx_i^{m}(t)\}_{i, m = 1}^{N, M}$ is obtained at $\{t_{l'}\}_{l' = 1}^{L'}$ (except for $t_1 = 0$ and $t_{L'} = T$).  When $t_l$ does not land on one of the numerical integration time points, a continuous extension method is used to interpolate the numerical solution at $t_l$. 
\section{Numerical Experiments}\label{sec:num_results_SI}
We consider three prototypical first order dynamics, Opinion Dynamics (OD), Lennard-Jones Dynamics (LJD), and Predator-Swarm dynamics (PS$1$), on two different manifolds, the $2D$ sphere ($\mathbb{S}^2$ centered at the origin with radius $\frac{5}{\pi}$) and the Poincar\'{e} disk ($ \mathbb{PD} $, unit disk centered at the origin, with the hyperbolic metric).  The two prototypical manifolds are chosen because $\mathbb{S}^2$ and $ \mathbb{PD} $ are model spaces with constant positive and negative curvature, respectively.  We conduct extensive experiments on the aforementioned six different scenarios to demonstrate the performance of our learning approach for dynamics evolving on manifolds.  We report the results in terms of function estimation errors and trajectory estimation errors, and discuss in detail the learning performance of the estimators.

The setup of the numerical experiments is as follows.  We generate a set of $M_{\rho}$ different initial conditions, and evolve the various dynamics of $N$ agents for $t \in [0, T]$ using a Geometric Numerical Integrator with a uniform time step $h$ (for details see section \ref{sec:numerics}); then we observe each dynamics at equidistant times, i.e. $0 = t_1 < \ldots < t_L = T$, to obtain a set of trajectory data, $\{\bx_i^{m}(t_l), \dot\bx_i^{m}(t_l)\}_{i, l, m = 1}^{N, L, M_{\rho}}$, to approximate the ``true'' probability distribution $\rho_{T, \mM}^L$.  From this set of pre-generated trajectory data, we randomly choose a subset of $M \ll M_{\rho}$ of them to be used as training data for the learning simulation.  The hypothesis space where the estimator is learned is generated as a set of $n$ first-degree clamped B-spline basis functions built on a uniform partition of the learning interval $[R^{\text{obs}}_{\min}, R^{\text{obs}}_{\max}]$, with $R^{\text{obs}}_{\min}$ and $R^{\text{obs}}_{\max}$ being the minimum and maximum interaction radii computed from the training and trajectory data, respectively.  Once an estimator, denoted as $\lintkernel$, is obtained, we report the estimation error, $\intkernel(\cdot)\cdot - \lintkernel(\cdot)\cdot$, using
\begin{equation}\label{eq:rel_L2rhoT_error SI}
\norm{\intkernelvar(\cdot)\cdot - \intkernel(\cdot)\cdot}_{\text{Rel.} L^2(\rho_{T,\mM})}\!\! \coloneqq \dfrac{\norm{\intkernelvar(\cdot)\cdot - \intkernel(\cdot)\cdot}_{L^2(\rho_{T,\mM})}}{\norm{\intkernel(\cdot)\cdot}_{L^2(\rho_{T,\mathcal{M}})}};
\end{equation}
and the trajectory estimation error 
\begin{equation}\label{eq:traj_norm SI}
d_{\text{trj}}(\bX^{m}_{[0, T]},\hat\bX^{m}_{[0, T]})^2\! \coloneqq \!\!\!\sup\limits_{t\in[0,T]}\!\!\!{\frac{\sum_i d_{\mM}(\bx_i^{m}(t), \hat\bx_i^{m}(t))^2\!\!\!}{N}}
\end{equation}
between, the true and estimated dynamics, evolved using $\intkernel$ or $\lintkernel$ with the same initial conditions for $t \in [0, T]$ respectively, and observed at the same observation times $0 = t_1 < \ldots < t_L = T$, over both the training initial conditions and another set of $M$ randomly chosen initial conditions.  Moreover, the above learning procedure is run $10$ times independently in order to generate empirical error bars.  We will report the errors in the form of $\text{mean}\pm\text{std}$.  Visual comparisons of $\intkernel$ versus $\lintkernel$, and $\bX$ versus $\hat\bX$ will be shown, and discussions of learning results will be presented in each subsection.

Table \ref{tab:common_params} shows the values of the common parameters shared by all six experiments.
\begin{table}[H]
\centering
\small{
\small{\begin{tabular}{ c | c | c | c | c | c | c}
$M_{\rho}$ & $N$   & $L$  & $M$   & Num. of Learning Trials & $\IR$ on $\mathbb{S}^2$ & $\IR$ on $ \mathbb{PD} $ \\
\hline
$3000$     & $20$ & $500$ & $500$ & $10$                    & $5$            & $\infty$\\
\end{tabular}}  
}
\caption{Values of the parameters shared by the six experiments}
\label{tab:common_params} 
\end{table}
\noindent
Moreover, section \ref{sec:manifold} shows the details on how to calculate the geodesic direction and the Riemannian distance between any two points on $\mathbb{S}^2$ and $ \mathbb{PD} $.  The distribution of the initial conditions, $\muX$, is given as follows: uniform on $\mathcal{M}=\mathbb{S}^2$; whereas uniform on an open ball (centered at origin with radius $r_0)$ for the $ \mathbb{PD} $ case with $r_0$ given as follows.
\[
r_0 = \bigg(2 + \frac{1}{\cosh(5) - 1} - \sqrt{\frac{4}{\cosh(5) - 1} + \frac{1}{(\cosh(5) - 1)^2}}\bigg)/2.
\]
This radius is used so that the maximum distance between any pair of agents on the Poincar\'{e} disk is $5$.  PS$1$ will have different setup for the initial conditions, which will be discussed in section \ref{sec:ps1_results}.
\subsection{Computing Platform}
We use a computing workstation with an AMD Ryzen $9$ $3900$X CPU (which has $12$ computing cores), and available $128$ GB memory, running CentOS $7$, provided and managed by Prisma Analytics, Inc. .  
All $6$ experiments are ran in the MATLAB (R$2020a$) environment with parallel mode enabled and a parallel pool of $12$ workers.  Such parallel mode is used in each experiment for the computation of $\rho_{T, \mM}^L$, learning, and trajectory error estimation.  Detailed report of the running time for the experiments is provided in the result section of each experiment.
\subsection{Opinion Dynamics}
We first choose opinion dynamics, which is used to model simple interactions of opinions \cite{OpinionDynamicsAylin2017, continuousOD} as well as choreography \cite{CLP2014}.  We consider the generalization of this dynamics to take place on two different manifolds: the $2D$ sphere ($\mathbb{S}^2$) and the Poincar\'{e} disk ($\mathbb{PD}$).
We consider the interaction kernel
\[
\intkernel(r) := \left\{
        \begin{array}{ll}
            1,                            & 0                         \le r < \frac{1}{\sqrt{2}} - 0.01 \\
            a_1r^3 + b_1r^2 + c_1r + d_1, & \frac{1}{\sqrt{2}} - 0.01 \le r < \frac{1}{\sqrt{2}} \\
            0.1,                          & \frac{1}{\sqrt{2}}        \le r < 0.99 \\
            a_2r^3 + b_2r^2 + c_2r + d_2, & 0.99                      \le r < 1 \\
            0,   & \text{otherwise}
        \end{array}
    \right.
\]
The parameters, i.e. $(a_1,a_2, b_1,b_2, c_1,c_2, d_1,d_2)$, are chosen so that $\intkernel \in C^1([0, 1])$.  Table \ref{tab:OD_params} shows the values of the parameters needed for the learning simulation.
\begin{table}[H]
\centering
\small{
\small{\begin{tabular}{ c | c | c | c }
$n_{\mathbb{S}^2}$ & $n_{\mathbb{PD}}$ & $T$  & $h$ \\
\hline
$51$      & $69$            & $10$ & $0.01$ \\
\end{tabular}}  
}
\caption{Test Parameters for OD.}
\label{tab:OD_params} 
\end{table}

\textbf{Results for the $\mathbb{S}^2$ case:} Fig. \ref{fig:OD_on_S2_phiE} shows the comparison between $\intkernel$ and its estimator $\lintkernel$ learned from the trajectory data.
\begin{figure}[H]  
\begin{subfigure}{\textwidth}
  \centering
  \includegraphics[width=0.8\textwidth]{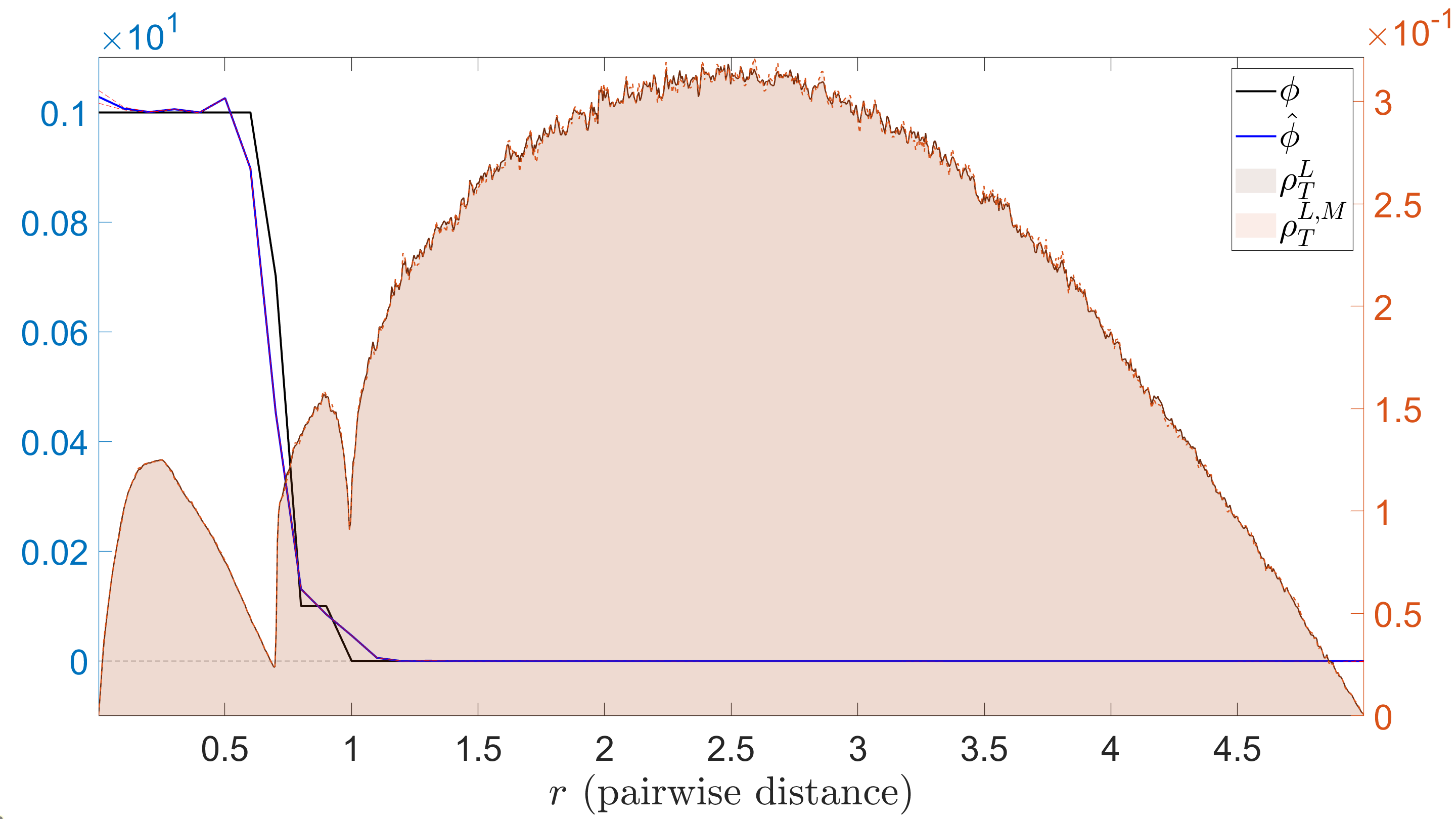}
\end{subfigure}
\caption{(OD on $\mathbb{S}^2$) Comparison of $\intkernel$ and $\lintkernel$, with the relative error being $1.894 \cdot 10^{-1} \pm 3.1 \cdot 10^{-4}$ (calculated using \eqref{eq:rel_L2rhoT_error SI}). The true interaction kernel is shown in black solid line, whereas the mean estimated interaction kernel is shown in blue solid line with its confidence interval shown in red dotted lines.  Shown in the background is the comparison of the approximate $\rho_{T, \mM}^L$ versus the empirical $\rho_{T, \mM}^{L, M}$.}
\label{fig:OD_on_S2_phiE}
\end{figure}
As it is shown in Fig. \ref{fig:OD_on_S2_phiE}, the estimator is able to capture the compact support of the $\intkernel$ from the trajectory data.  Fig. \ref{fig:OD_on_S2_trajs} shows the comparison of the trajectory data between the true dynamics and estimated dynamics.
\begin{figure}[H] 
\begin{subfigure}{\textwidth}
  \centering
  \includegraphics[width=0.8\textwidth]{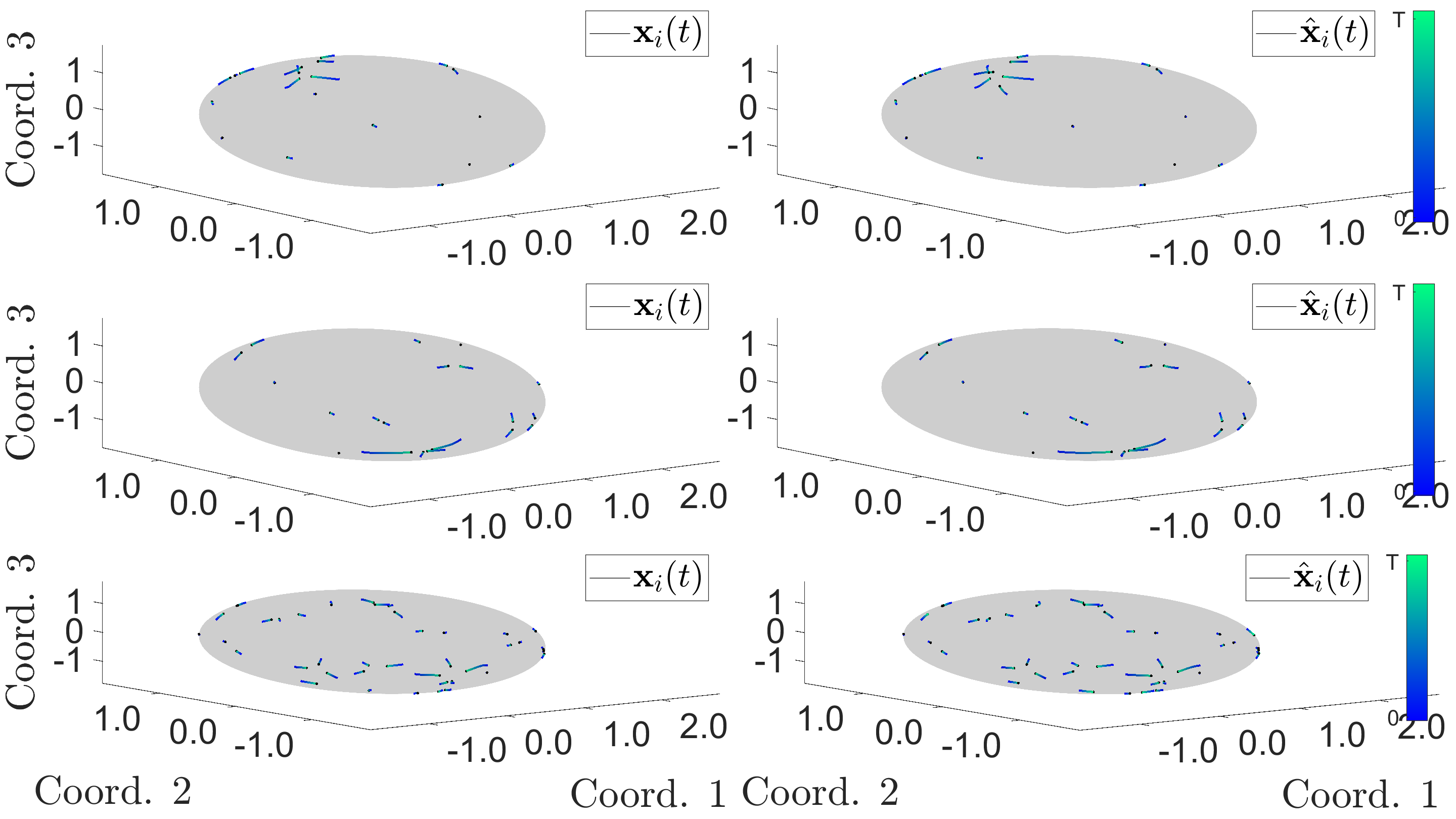} 
\end{subfigure}
\caption{(OD on $\mathbb{S}^2$) Comparison of $\bX$ (generated by $\intkernel$) and $\hat\bX$ (generated by $\lintkernel$), with the errors reported in table \ref{tab:OD_on_S2_traj_err}.  \textbf{Top}: $\bX$ and $\hat\bX$ are generated from an initial condition taken from the training data.  \textbf{Middle}: $\bX$ and $\hat\bX$ are generated from a randomly chosen initial condition.  \textbf{Bottom}: $\bX$ and $\hat\bX$ are generated from a new initial condition with bigger $N = 40$.  The color of the trajectory indicates the flow of time, from deep blue (at $t = 0$) to light green (at $t = T$).}
\label{fig:OD_on_S2_trajs}
\end{figure}
A quantitative comparison of the trajectory estimation errors is shown in Table \ref{tab:OD_on_S2_traj_err}.
\begin{table}[H]
\centering
\small{\begin{tabular}{| c || c |} 
\hline
                                        & $[0, T]$                                \\
\hline
$\text{mean}_{\text{IC}}$: Training ICs & $8.8 \cdot 10^{-2} \pm 1.7 \cdot 10^{-3}$ \\
\hline
$\text{std}_{\text{IC}}$:  Training ICs & $5.9 \cdot 10^{-2} \pm 1.5 \cdot 10^{-3}$ \\
\hline   
\hline         
$\text{mean}_{\text{IC}}$: Random ICs   & $9.0 \cdot 10^{-2} \pm 1.6 \cdot 10^{-3}$ \\
\hline
$\text{std}_{\text{IC}}$:  Random ICs   & $6.0 \cdot 10^{-2} \pm 1.7 \cdot 10^{-3}$ \\
\hline   
\end{tabular}}
\caption{(OD on $\mathbb{S}^2$) trajectory estimation errors: Initial Conditions (ICs) used in the training set (first two rows), new ICs randomly drawn from $\muX$ (second set of two rows).  $\text{mean}_{\text{IC}}$ and $\text{std}_{\text{IC}}$ are the mean and standard deviation of the trajectory errors calculated using \eqref{eq:traj_norm SI}.}
\label{tab:OD_on_S2_traj_err}
\end{table}
We also report the condition number and the smallest eigenvalue of the learning matrix $A$ to indirectly verify the geometric coercivity condition in table \ref{tab:OD_on_S2_coer}.
\begin{table}[H]
\centering
\small{\begin{tabular}{ c || c } 
Condition Number    & $1.8 \cdot 10^{5} \pm 1.4 \cdot 10^{4}$ \\ 
\hline
Smallest Eigenvalue & $1.09 \cdot 10^{-7} \pm 9.0 \cdot 10^{-9}$
\end{tabular}}
\caption{(OD on $\mathbb{S}^2$) Information from the learning matrix $A$.}
\label{tab:OD_on_S2_coer}
\end{table}
It took $1.41 \cdot 10^{4}$ seconds to generate $\rho_{T, \mM}^L$ and $4.76 \cdot 10^{4}$ seconds to run $10$ learning simulations, with $1.44 \cdot 10^{3}$ seconds spent on learning the estimated interactions (on average, it took $1.44 \cdot 10^{2} \pm 3.1$ seconds to run one estimation), and $4.61 \cdot 10^{4}$ seconds spent on computing the trajectory error estimates (on average, it took $4.61 \cdot 10^{3} \pm 20.0$ seconds to run one set of trajectory error estimation).

\textbf{Results for the $ \mathbb{PD} $ case:} Fig. \ref{fig:OD_on_PD_phiE} shows the comparison between the $C^1$ version of $\phi$ and its estimator $\hat\phi$ learned from the trajectory data.
\begin{figure}[H] %FIGFLOAT
\begin{subfigure}{\textwidth}
  \centering
  \includegraphics[width=0.8\textwidth]{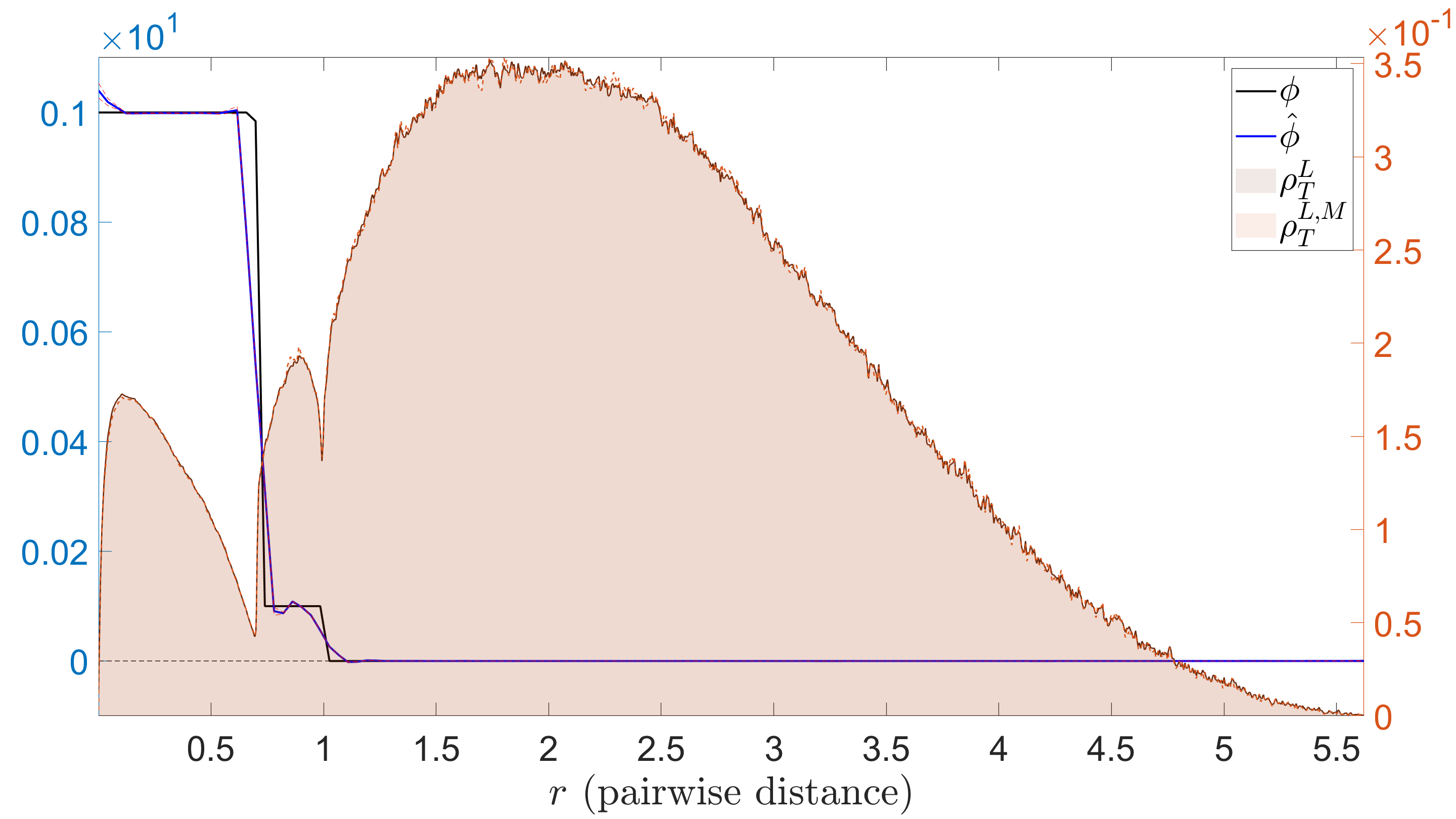}
\end{subfigure}
\caption{(OD on $ \mathbb{PD} $) Comparison of $\intkernel$ and $\lintkernel$ , with the relative error being $2.114 \cdot 10^{-1} \pm 5.0 \cdot 10^{-4}$ (calculated using \eqref{eq:rel_L2rhoT_error SI}). The true interaction kernel is shown in black solid line, whereas the mean estimated interaction kernel is shown in blue solid line with its confidence interval shown in red dotted lines.  Shown in the background is the comparison of the approximate $\rho_{T, \mM}^L$ versus the empirical $\rho_{T, \mM}^{L, M}$.}
\label{fig:OD_on_PD_phiE}
\end{figure}
As it is shown in Fig. \ref{fig:OD_on_PD_phiE}, the estimator is able to capture the compact support of the $\intkernel$ from the trajectory data.  Fig. \ref{fig:OD_on_PD_trajs} shows the comparison of the trajectory data between the true dynamics and estimated dynamics.
\begin{figure}[H] 
\begin{subfigure}{\textwidth}
  \centering
  \includegraphics[width=0.8\textwidth]{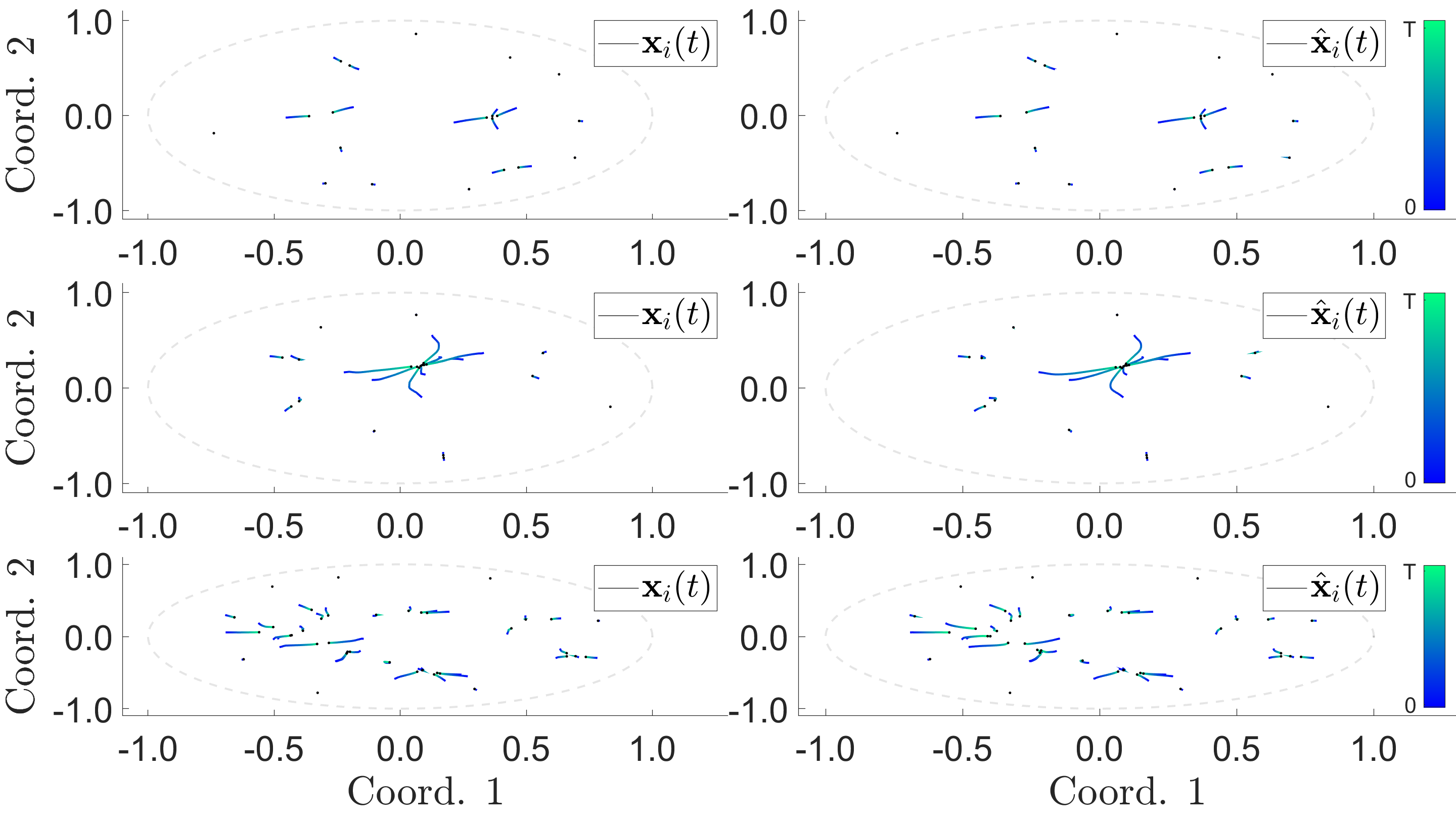} 
\end{subfigure}
\caption{(OD on  $ \mathbb{PD} $ ) Comparison of $\bX$ (generated by $\intkernel$) and $\hat\bX$ (generated by $\lintkernel$), with the errors reported in table \ref{tab:OD_on_PD_traj_err}.  \textbf{Top}: $\bX$ and $\hat\bX$ are generated from an initial condition taken from the training data.  \textbf{Middle}: $\bX$ and $\hat\bX$ are generated from a randomly chosen initial condition.  \textbf{Bottom}: $\bX$ and $\hat\bX$ are generated from a new initial condition with bigger $N = 40$.  The color of the trajectory indicates the flow of time, from deep blue (at $t = 0$) to light green (at $t = T$).}
\label{fig:OD_on_PD_trajs}
\end{figure}
As shown in Fig. \ref{fig:OD_on_PD_phiE}, around $r = \frac{1}{\sqrt{2}}$, the estimator $\lintkernel$ produces values bigger than that from $\intkernel$, leading to stronger influence, hence the merging of cluster happening in the predicted trajectories in the second row of Fig. \ref{fig:OD_on_PD_trajs}. As demonstrated by the average prediction error on trajectories, this is a relatively rare event, occurring for only certain initial conditions.  A quantitative comparison of the trajectory estimation errors is shown in Table \ref{tab:OD_on_PD_traj_err}.
\begin{table}[H]
\centering
\small{\begin{tabular}{| c || c |} 
\hline
                                        & $[0, T]$                                \\
\hline
$\text{mean}_{\text{IC}}$: Training ICs & $2.53 \cdot 10^{-1} \pm 7.2 \cdot 10^{-3}$ \\
\hline
$\text{std}_{\text{IC}}$:  Training ICs & $1.90 \cdot 10^{-1} \pm 6.5 \cdot 10^{-3}$ \\
\hline   
\hline         
$\text{mean}_{\text{IC}}$: Random ICs   & $2.55 \cdot 10^{-1} \pm 9.7 \cdot 10^{-3}$ \\
\hline
$\text{std}_{\text{IC}}$:  Random ICs   & $1.89 \cdot 10^{-1} \pm 5.9 \cdot 10^{-3}$ \\
\hline   
\end{tabular}}
\caption{(OD on  $ \mathbb{PD} $ ) trajectory estimation errors: Initial Conditions (ICs) used in the training set (first two rows), new ICs randomly drawn from $\muX$ (second set of two rows).  $\text{mean}_{\text{IC}}$ and $\text{std}_{\text{IC}}$ are the mean and standard deviation of the trajectory errors calculated using \eqref{eq:traj_norm SI}.}
\label{tab:OD_on_PD_traj_err}
\end{table}
We also report the condition number and the smallest eigenvalue of the learning matrix $A$ to indirectly verify the geometric coercivity condition in table \ref{tab:OD_on_PD_coer}.
\begin{table}[H]
\centering
\small{\begin{tabular}{ c || c } 
Condition Number    & $4.9 \cdot 10^{5} \pm 1.5 \cdot 10^{4}$ \\ 
\hline
Smallest Eigenvalue & $5.3 \cdot 10^{-6} \pm 1.2 \cdot 10^{-7}$
\end{tabular}}
\caption{(OD on  $ \mathbb{PD} $ ) Information from the learning matrix $A$.}
\label{tab:OD_on_PD_coer}
\end{table}
It took $1.33 \cdot 10^{4}$ seconds to generate $\rho_{T, \mM}^L$ and $4.06 \cdot 10^{4}$ seconds to run $10$ learning simulations, with $1.23 \cdot 10^{3}$ seconds spent on learning the estimated interactions (on average, it took $1.23 \cdot 10^{2} \pm 1.1$ seconds to run one estimation), and $3.93 \cdot 10^{4}$ seconds spent on computing the trajectory error estimates (on average, it took $3.93 \cdot 10^{3} \pm 82.1$ seconds to run one set of trajectory error estimation).
\subsection{Lennard-Jones Dynamics}
The second first-order model considered here is induced from a special energy functional, the so-called Lennard-Jones energy potential.  This first-order model, the Lennard-Jones Dynamics (LJD), is a simplified version of the second-order dynamics used in molecular dynamics.  The energy function, $U_{\text{LJ}}$, is given by
\[
U_{\text{LJ}}(r) := 4\varepsilon\Big( \Big(\frac{\sigma}{r}\Big)^{12} - \Big(\frac{\sigma}{r}\Big)^6 \Big)\,.
\]
Here $\varepsilon$ is the depth of the potential well, $\sigma$ is the distance when $U$ is zero, and $r$ is the distance between any pair of agents.  We set $\varepsilon = 10$ and $\sigma = 1$.  The corresponding interaction kernel $\phi$, derived from this potential, is
\[
\intkernel_{\text{LJ}}(r) := \frac{U_{\text{LJ}}'(r)}{r} = 24\frac{\varepsilon}{\sigma^2}\Big(\Big(\frac{\sigma}{r}\Big)^8 - 2\Big(\frac{\sigma}{r}\Big)^{14}\Big)\,.
\]
We shall use a slightly modified version of $\intkernel_{\text{LJ}}$:
\[
\intkernel(r) := \left\{
        \begin{array}{ll}
            \intkernel_{\text{LJ}}(1) - \intkernel_{\text{LJ}}'(1)/4,                                & 0           \le r < \frac{1}{2} \\
            \intkernel_{\text{LJ}}'(1)r^2 - \intkernel_{\text{LJ}}'(1)r + \intkernel_{\text{LJ}}(1), & \frac{1}{2} \le r < 1 \\
            \intkernel_{\text{LJ}}(r),                                                               & 1           \le r < 0.99\IR \\
            a_3r^3 + b_3r^2 + c_3r + d_3,                                                            & 0.99\IR     \le r < \IR \\
            0,                                                                                       & \IR         \le r.
        \end{array}
    \right.
\]
The parameters, $(a_3, b_3, c_3, d_3)$, are chosen so that $\intkernel \in C^1([0, \IR])$ when $\IR < \infty$; otherwise $\intkernel(r) = \intkernel_{\text{LJ}}(r)$ for $r \ge 1$.  Table \ref{tab:LJD_params} shows the values of the parameters needed for the learning simulation.
\begin{table}[H]
\centering
\small{
\small{\begin{tabular}{ c | c | c | c }
$n_{\mathbb{S}^2}$ & $n_{\mathbb{PD}}$ & $T$       & $h$ \\
\hline
$51$      & $69$            & $10^{-3}$ & $10^{-6}$ \\
\end{tabular}}  
}
\caption{Test Parameters for LJD.}
\label{tab:LJD_params} 
\end{table}

\textbf{Results for the $\mathbb{S}^2$ case:} Fig. \ref{fig:LJD_on_S2_phiE} shows the comparison between $\intkernel$ and its estimator $\lintkernel$ learned from the trajectory data.
\begin{figure}[H] 
\begin{subfigure}{\textwidth}
  \centering
  \includegraphics[width=0.8\textwidth]{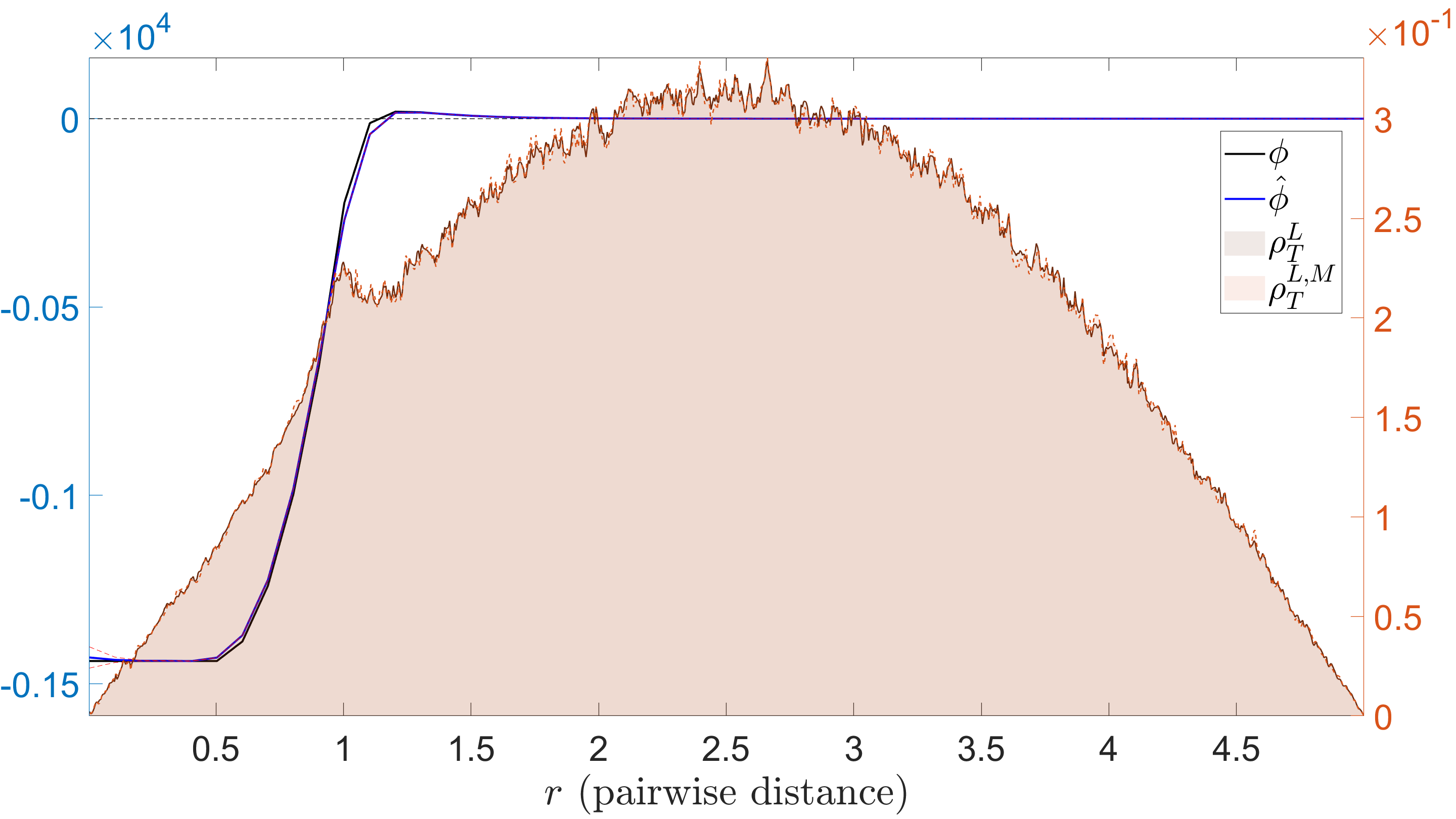}
\end{subfigure}
\caption{(LJD on $\mathbb{S}^2$) Comparison of $\intkernel$ and $\lintkernel$, with the relative error being $3.65 \cdot 10^{-2} \pm 2.7 \cdot 10^{-4}$ (calculated using \eqref{eq:rel_L2rhoT_error SI}). The true interaction kernel is shown in black solid line, whereas the mean estimated interaction kernel is shown in blue solid line with its confidence interval shown in blue dotted lines.  Shown in the background is the comparison of the approximate $\rhoTL$ versus the empirical $\rhoTLM$.}
\label{fig:LJD_on_S2_phiE}
\end{figure}
Fig. \ref{fig:LJD_on_S2_trajs} shows the comparison of the trajectory data between the true dynamics and estimated dynamics.
\begin{figure}[H]  
\begin{subfigure}{\textwidth}
  \centering
  \includegraphics[width=0.8\textwidth]{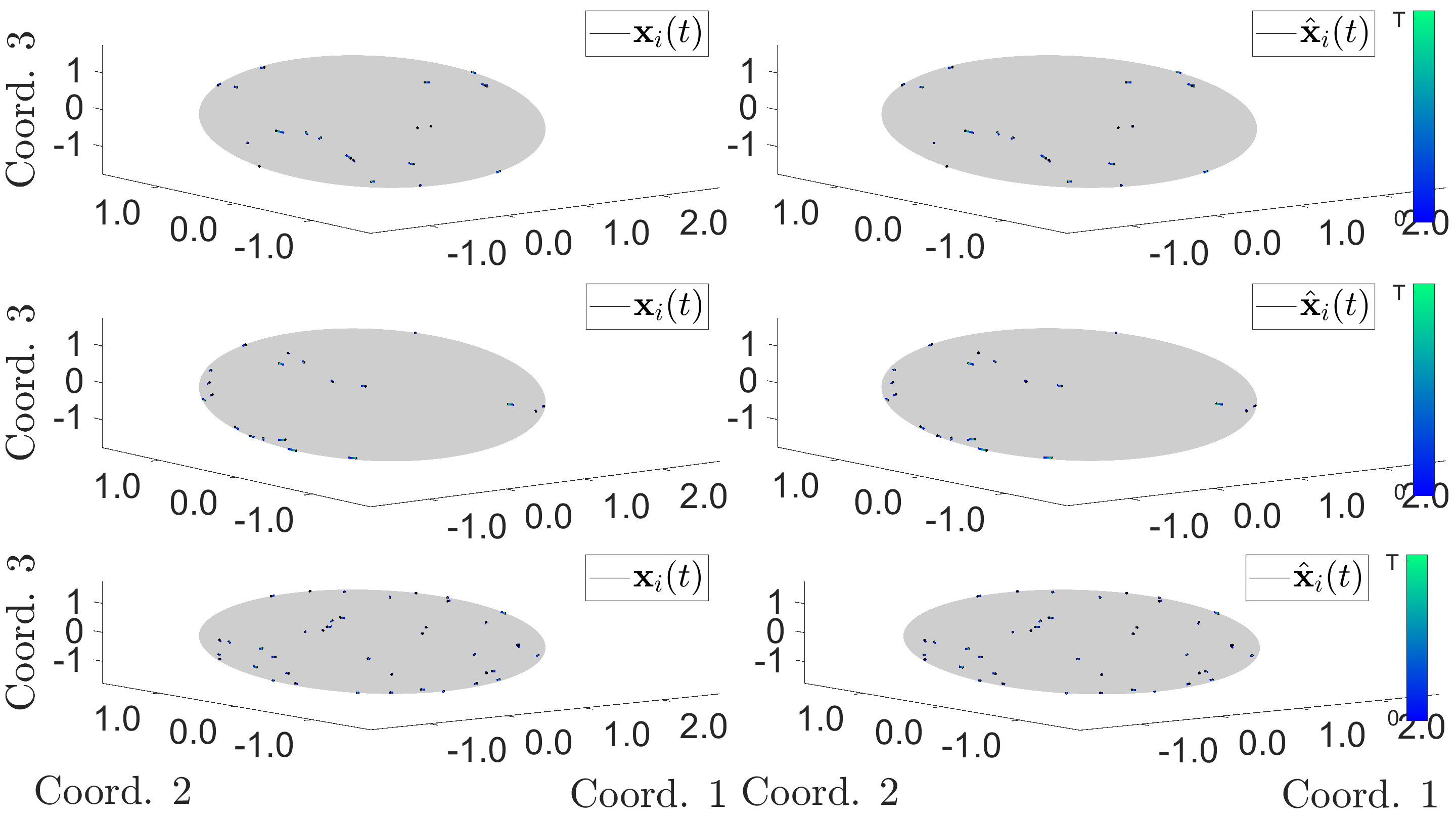} 
\end{subfigure}
\caption{(LJD on $\mathbb{S}^2$) Comparison of $\bX$ (generated by $\intkernel$) and $\hat\bX$ (generated by $\lintkernel$), with the errors reported in table \ref{tab:LJD_on_S2_traj_err}.  \textbf{Top}: $\bX$ and $\hat\bX$ are generated from an initial condition taken from the training data.  \textbf{Middle}: $\bX$ and $\hat\bX$ are generated from a randomly chosen initial condition.  \textbf{Bottom}: $\bX$ and $\hat\bX$ are generated from a new initial condition with bigger $N = 40$.  The color of the trajectory indicates the flow of time, from deep blue (at $t = 0$) to light green (at $t = T$).}
\label{fig:LJD_on_S2_trajs}
\end{figure}
A quantitative comparison of the trajectory estimation errors is shown in Table \ref{tab:LJD_on_S2_traj_err}.
\begin{table}[H]
\centering
\small{\begin{tabular}{| c || c |} 
\hline
                                        & $[0, T]$                                \\
\hline
$\text{mean}_{\text{IC}}$: Training ICs & $2.88 \cdot 10^{-3} \pm 2.5 \cdot 10^{-5}$ \\
\hline
$\text{std}_{\text{IC}}$:  Training ICs & $6.1 \cdot 10^{-4} \pm 1.8 \cdot 10^{-5}$ \\
\hline   
\hline         
$\text{mean}_{\text{IC}}$: Random ICs   & $2.88 \cdot 10^{-3} \pm 3.2 \cdot 10^{-5}$ \\
\hline
$\text{std}_{\text{IC}}$:  Random ICs   & $6.0 \cdot 10^{-4} \pm 1.8 \cdot 10^{-5}$ \\
\hline   
\end{tabular}}
\caption{(LJD on $\mathbb{S}^2$) trajectory estimation errors: Initial Conditions (ICs) used in the training set (first two rows), new ICs randomly drawn from $\muX$ (second set of two rows).  The trajectory estimation errors is calculated using \eqref{eq:rel_L2rhoT_error SI}.}
\label{tab:LJD_on_S2_traj_err}
\end{table}
We also report the condition number and the smallest eigenvalue of the learning matrix $A$ to indirectly verify the geometric coercivity condition in table \ref{tab:LJD_on_S2_coer}.
\begin{table}[H]
\centering
\small{\begin{tabular}{ c || c } 
Condition Number    & $6 \cdot 10^{5} \pm 1.5 \cdot 10^{5}$ \\ 
\hline
Smallest Eigenvalue & $2.4 \cdot 10^{-8} \pm 6.2 \cdot 10^{-9}$
\end{tabular}}
\caption{(LJD on $\mathbb{S}^2$) Information from the learning matrix $A$.}
\label{tab:LJD_on_S2_coer}
\end{table}
It took $2.43 \cdot 10^{4}$ seconds to generate $\rho_{T, \mM}^L$ and $7.14 \cdot 10^{4}$ seconds to run $10$ learning simulations, with $1.72 \cdot 10^{3}$ seconds spent on learning the estimated interactions (on average, it took $1.72 \cdot 10^{2} \pm 2.5$ seconds to run one estimation), and $6.96 \cdot 10^{4}$ seconds spent on computing the trajectory error estimates (on average, it took $6.96 \cdot 10^{3} \pm 35.9$ seconds to run one set of trajectory error estimation).

\textbf{Results for the  $ \mathbb{PD} $  case:} Fig. \ref{fig:LJD_on_PD_phiE} shows the comparison between $\intkernel$ and its estimator $\lintkernel$ learned from the trajectory data.
\begin{figure}[H]  
\begin{subfigure}{\textwidth}
  \centering
  \includegraphics[width=0.8\textwidth]{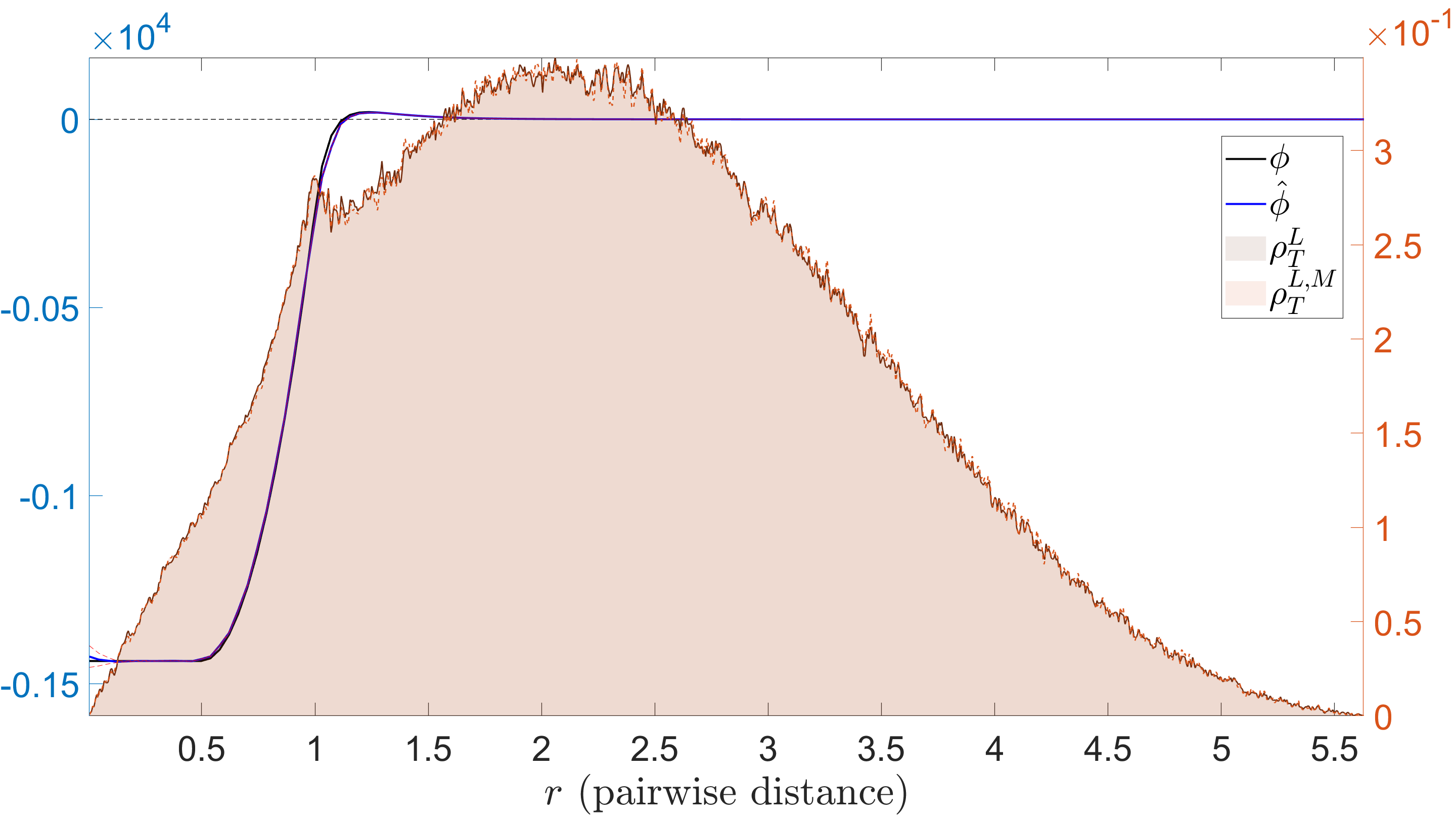}
\end{subfigure}
\caption{(LJD on  $ \mathbb{PD} $ ) Comparison of $\intkernel$ and $\lintkernel$, with the relative error being $2.52 \cdot 10^{-2} \pm 3.6 \cdot 10^{-4}$ (calculated using \eqref{eq:rel_L2rhoT_error SI}). The true interaction kernel is shown in black solid line, whereas the mean estimated interaction kernel is shown in blue solid line with its confidence interval shown in blue dotted lines.  Shown in the background is the comparison of the approximate $\rhoTL$ versus the empirical $\rhoTLM$.}
\label{fig:LJD_on_PD_phiE}
\end{figure}
Fig. \ref{fig:LJD_on_PD_trajs} shows the comparison of the trajectory data between the true dynamics and estimated dynamics.
\begin{figure}[H]  
\begin{subfigure}{\textwidth}
  \centering
  \includegraphics[width=0.8\textwidth]{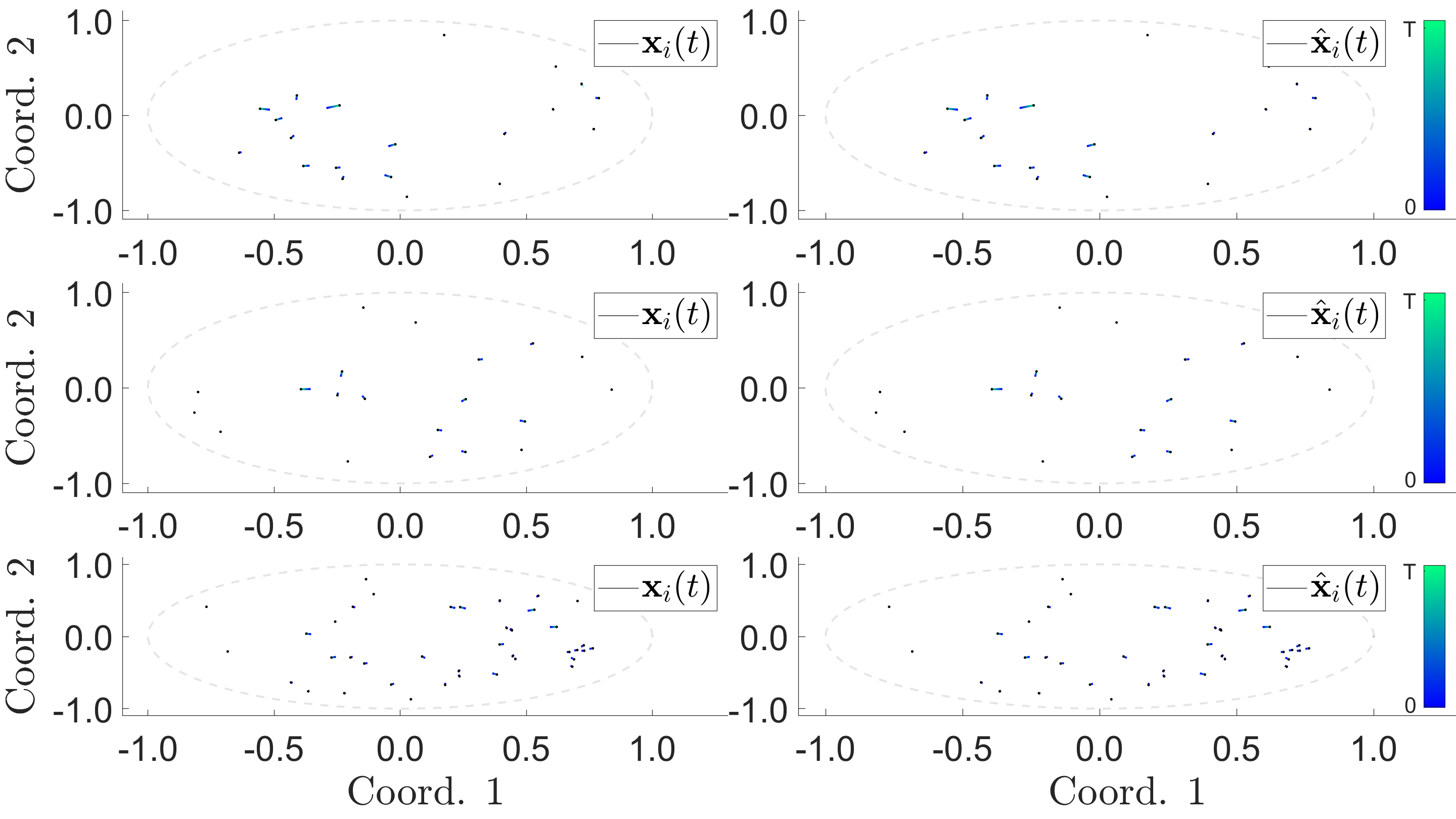} 
\end{subfigure}
\caption{(LJD on  $ \mathbb{PD} $ ) Comparison of $\bX$ (generated by $\intkernel$) and $\hat\bX$ (generated by $\lintkernel$), with the errors reported in table \ref{tab:LJD_on_PD_traj_err}.  \textbf{Top}: $\bX$ and $\hat\bX$ are generated from an initial condition taken from the training data.  \textbf{Middle}: $\bX$ and $\hat\bX$ are generated from a randomly chosen initial condition.  \textbf{Bottom}: $\bX$ and $\hat\bX$ are generated from a new initial condition with bigger $N = 40$.  The color of the trajectory indicates the flow of time, from deep blue (at $t = 0$) to light green (at $t = T$).}
\label{fig:LJD_on_PD_trajs}
\end{figure}
A quantitative comparison of the trajectory estimation errors is shown in Table \ref{tab:LJD_on_PD_traj_err}.
\begin{table}[H]
\centering
\small{\begin{tabular}{| c || c |} 
\hline
                                        & $[0, T]$                                \\
\hline
$\text{mean}_{\text{IC}}$: Training ICs & $2.27 \cdot 10^{-3} \pm 4.0 \cdot 10^{-5}$ \\
\hline
$\text{std}_{\text{IC}}$:  Training ICs & $5.6 \cdot 10^{-4} \pm 1.7 \cdot 10^{-5}$ \\
\hline   
\hline         
$\text{mean}_{\text{IC}}$: Random ICs   & $2.28 \cdot 10^{-3} \pm 3.8 \cdot 10^{-5}$ \\
\hline
$\text{std}_{\text{IC}}$:  Random ICs   & $5.6 \cdot 10^{-4} \pm 1.6 \cdot 10^{-5}$ \\
\hline   
\end{tabular}}
\caption{(LJD on  $ \mathbb{PD} $ ) trajectory estimation errors: Initial Conditions (ICs) used in the training set (first two rows), new ICs randomly drawn from $\muX$ (second set of two rows).  $\text{mean}_{\text{IC}}$ and $\text{std}_{\text{IC}}$ are the mean and standard deviation of the trajectory errors calculated using \eqref{eq:traj_norm SI}.}
\label{tab:LJD_on_PD_traj_err}
\end{table}
We also report the condition number and the smallest eigenvalue of the learning matrix $A$ to indirectly verify the geometric coercivity condition in table \ref{tab:LJD_on_PD_coer}.
\begin{table}[H]
\centering
\small{\begin{tabular}{ c || c } 
Condition Number    & $6 \cdot 10^{6} \pm 1.9 \cdot 10^{6}$ \\ 
\hline
Smallest Eigenvalue & $1.7 \cdot 10^{-8} \pm 6.6 \cdot 10^{-9}$
\end{tabular}}
\caption{(LJD on  $ \mathbb{PD} $ ) Information from the learning matrix $A$.}
\label{tab:LJD_on_PD_coer}
\end{table}
It took $1.51 \cdot 10^{4}$ seconds to generate $\rho_{T, \mM}^L$ and $6.23 \cdot 10^{4}$ seconds to run $10$ learning simulations, with $1.20 \cdot 10^{3}$ seconds spent on learning the estimated interactions (on average, it took $1.20 \cdot 10^{2} \pm 9.4$ seconds to run one estimation), and $6.10 \cdot 10^{4}$ seconds spent on computing the trajectory error estimates (on average, it took $6 \cdot 10^{3} \pm 1.3 \cdot 10^{3}$ seconds to run one set of trajectory error estimation).
\subsection{Predator-Swarm Dynamics}\label{sec:ps1_results}
The third first-order model considered here is a heterogeneous agent system, which is used to model interactions between multiple types of animals \cite{CK2013, Olson_2016} or agents (need ref.).  The learning theory presented in this work is described for homogeneous agent systems, but the theory and the corresponding algorithms extend naturally to heterogeneous agent systems in a manner analogous to \cite{Tang2019, miller2020learning}.  

We consider here a system of a single predator versus a group of preys, namely the Predator-Swarm Dynamics (PS$1$), discussed in \cite{CK2013}.  The preys are in type $1$, and the single predator is in type $2$.  We have multiple interaction kernels, depending on the types of agents in each interacting pair: $\intkernel_{\idxcl\idxcl'}$ defines the influence of agents in type $\idxcl'$ on agents in type $\idxcl$, for $\idxcl, \idxcl' = 1, 2$.  The interaction kernels are given as follows.
\[
\intkernel_{11}(r) := \left\{
        \begin{array}{ll}
            \frac{2}{0.01^3}(r - 0.01) + (1 - \frac{1}{0.01^2}) & 0 < r \le 0.01 \\
            1 - \frac{1}{r^2}                                   & 0.01 < r \le 0.99\IR \\
            a_{1, 1}r^3 + b_{1, 1}r^2 + c_{1, 1}r + d_{1, 1},   & 0.99\IR     \le r < \IR \\
            0,                                                  & \IR         \le r
        \end{array}
    \right.
\]
The parameters, $(a_{1, 1}, b_{1, 1}, c_{1, 1}, d_{1, 1})$, are chosen so that $\intkernel_{11}(r) \in C^1([0, \IR])$ when $\IR < \infty$; otherwise $\intkernel_{11}(r) = 1 - \frac{1}{r^2}$ for $r \ge 0.01$;
\[
\intkernel_{12}(r) := \left\{
        \begin{array}{ll}
            \frac{4}{0.01^3}(r - 0.01) + \frac{-2}{0.01^2})   & 0 < r \le 0.01 \\
            \frac{-2}{r^2}                                    & 0.01 < r \le 0.99\IR \\
            a_{1, 2}r^3 + b_{1, 2}r^2 + c_{1, 2}r + d_{1, 2}, & 0.99\IR     \le r < \IR \\
            0,                                                & \IR         \le r
        \end{array}
    \right.
\]
The parameters, $(a_{1, 2}, b_{1, 2}, c_{1, 2}, d_{1, 2})$, are chosen so that $\intkernel_{12}(r) \in C^1([0, \IR])$ when $\IR < \infty$; otherwise $\intkernel_{12}(r) = \frac{-2}{r^2}$ for $r \ge 0.01$;
\[
\intkernel_{21}(r) := \left\{
        \begin{array}{ll}
            \frac{-10.5}{0.01^4}(r - 0.01) + \frac{3.5}{0.01^3}) & 0 < r \le 0.01 \\
            \frac{3.5}{r^3}                                      & 0.01 < r \le 0.99\IR \\
            a_{2, 1}r^3 + b_{2, 1}r^2 + c_{2, 1}r + d_{2, 1},    & 0.99\IR     \le r < \IR \\
            0,                                                   & \IR         \le r
        \end{array}
    \right.
\]
The parameters, $(a_{2, 1}, b_{2, 1}, c_{2, 1}, d_{2, 1})$, are chosen so that $\intkernel_{21}(r) \in C^1([0, \IR])$ when $\IR < \infty$; otherwise $\intkernel_{21}(r) = \frac{3.5}{r^3}$ for $r \ge 0.01$; then $\intkernel_{22} \equiv 0$, since there is only one predator.  We set $T = 0.5$ and $h = 10^{-4}$ for the two $PS1$ models.

\textbf{Results for the  $\mathbb{S}^2$  case:} In order to produce more interesting interactions, we choose the distribution of the initial condition to be as follows.  The setting will start from $\R^2$ first.  The position of the predator is randomly chosen uniformly within a circular disk of radius $0.1$ centered at the origin of $\R^2$.  The remaining $N - 1$ agents will be prey and chosen uniformly at random within an annulus of radii $0.3$ and $0.8$, centered at the origin.  Then these positions will mapped through a stereographic projection (where the origin of $\R^2$ is the south pole of $\mathbb{S}^2$) back to $\mathbb{S}^2$.  When back on $\mathbb{S}^2$, the position of the predator is moved via parallel transport to a random location on $\mathbb{S}^2$, and the rest of the preys are moved using the same map, so that the relative position between each pair of agents is not changed.

Table \ref{tab:PS1_ns_S2} shows the number of basis functions, namely $n_{\idxcl\idxcl'}$'s, for each estimator $\lintkernel_{\idxcl\idxcl'}$ for $\idxcl, \idxcl' = 1, 2$, and their corresponding degrees, $p_{\idxcl, \idxcl'}$'s, for the Clamped B-spline basis.
\begin{table}[H]
\centering
\small{
\small{\begin{tabular}{ c | c | c | c}
$n_{1, 1}$ & $n_{1, 2}$ & $n_{2, 1}$ & $n_{2, 2}$ \\
\hline
$50$       & $37$       & $37$       & $1$ \\
\hline
$p_{1, 1}$ & $p_{1, 2}$ & $p_{2, 1}$ & $p_{2, 2}$ \\
\hline
$1$        & $1$        & $1$        & $0$ \\
\hline
\end{tabular}}  
}
\caption{(PS$1$ on  $\mathbb{S}^2$ ) Number of basis functions.}
\label{tab:PS1_ns_S2} 
\end{table}
Fig. \ref{fig:PS1_on_PD_phiE} shows the comparison between $\intkernel_{\idxcl\idxcl'}$ and its estimators $\lintkernel_{\idxcl\idxcl'}$ learned from the trajectory data.
\begin{figure}[H]  
\begin{subfigure}{\textwidth}
  \centering
  \includegraphics[width=0.8\textwidth]{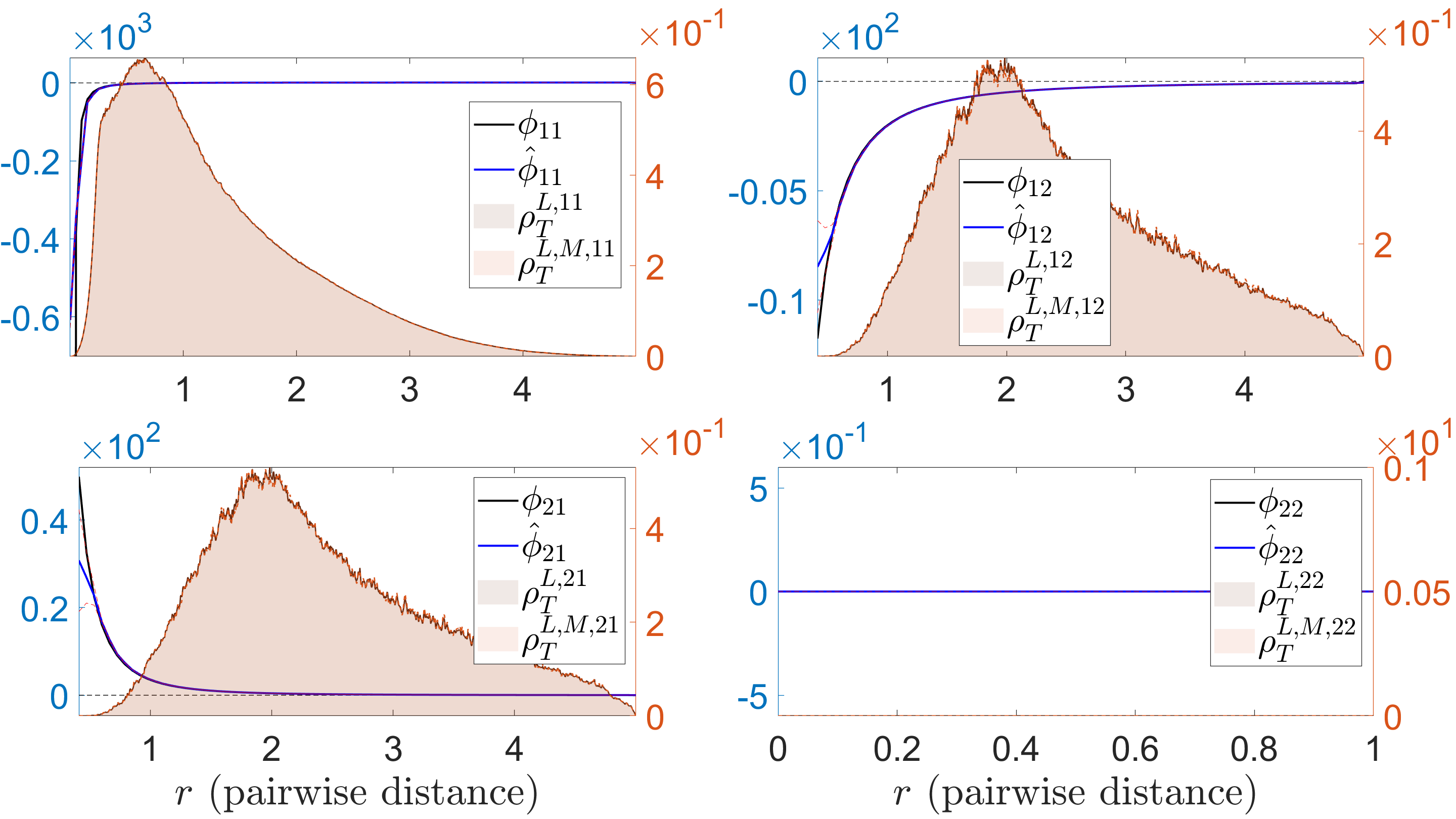}
\end{subfigure}
\caption{(PS$1$ on  $\mathbb{S}^2$ ) Comparison of $\intkernel_{\idxcl\idxcl'}$ and $\lintkernel_{\idxcl, \idxcl'}$, with the relative errors shown in table \ref{tab:PS1_errs_PD}. The true interaction kernels are shown in black solid line, whereas the mean estimated interaction kernel are shown in blue solid line with their corresponding confidence intervals shown in blue dotted lines.  Shown in the background is the comparison of the approximate $\rho_T^{L, \idxcl\idxcl'}$ versus the empirical $\rho_T^{L, M, \idxcl\idxcl'}$. Notice that $\rho_T^{L, 12}$/$\rho_T^{L, M, 12}$ and $\rho_T^{L, 12}$/$\rho_T^{L, M, 21}$ are the same distributions.}
\label{fig:PS1_on_S2_phiE}
\end{figure}
\begin{table}[H]
\centering
\small{
\small{\begin{tabular}{ c | c | c | c}
$\text{Err}_{1, 1}$                        & $\text{Err}_{1, 2}$                       & $\text{Err}_{2, 1}$                       & $\text{Err}_{2, 2}$ \\
\hline
$2.98 \cdot 10^{-1} \pm 5.9 \cdot 10^{-3}$ & $8.4 \cdot 10^{-3} \pm 3.0 \cdot 10^{-4}$ & $2.5 \cdot 10^{-2} \pm 1.6 \cdot 10^{-3}$ & $0$ \\
\end{tabular}}  
}
\caption{(PS$1$ on  $\mathbb{S}^2$ ) Relative estimation errors calculated using \eqref{eq:rel_L2rhoT_error SI}.}
\label{tab:PS1_errs_S2} 
\end{table}
Fig. \ref{fig:PS1_on_S2_trajs} shows the comparison of the trajectory data between the true dynamics and estimated dynamics. 
\begin{figure}[H] 
\begin{subfigure}{\textwidth}
  \centering
  \includegraphics[width=0.8\textwidth]{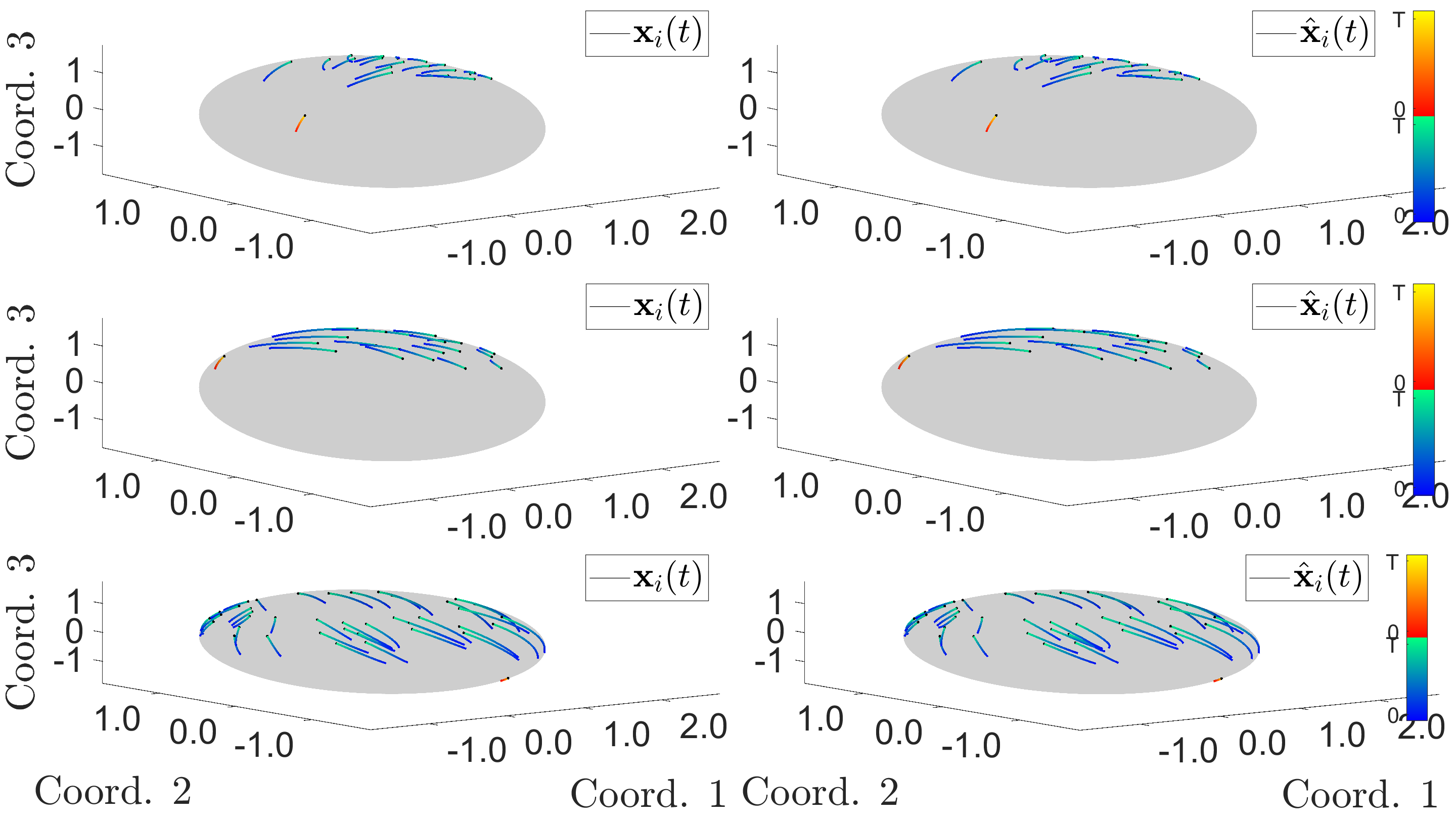} 
\end{subfigure}
\caption{(PS$1$ on  $\mathbb{S}^2$ ) Comparison of $\bX$ (generated by $\intkernel_{\idxcl, \idxcl'}$'s) and $\hat\bX$ (generated by $\lintkernel_{\idxcl, \idxcl'}$'s), with the errors reported in table \ref{tab:PS1_on_S2_traj_err}.  \textbf{Top}: $\bX$ and $\hat\bX$ are generated from an initial condition taken from the training data.  \textbf{Middle}: $\bX$ and $\hat\bX$ are generated from a randomly chosen initial condition.  \textbf{Bottom}: $\bX$ and $\hat\bX$ are generated from a new initial condition with bigger $N = 40$.  The color of the trajectory indicates the flow of time, from deep blue/bright red (at $t = 0$) to light green/light yellow (at $t = T$).  The blue/green combination is assigned to the preys; whereas the red/yellow comb for the predator.}
\label{fig:PS1_on_S2_trajs}
\end{figure}
A quantitative comparison of the trajectory estimation errors is shown in Table \ref{tab:PS1_on_PD_traj_err}.
\begin{table}[H]
\centering
\small{\begin{tabular}{| c || c |} 
\hline
                                        & $[0, T]$                                \\
\hline
$\text{mean}_{\text{IC}}$: Training ICs & $2.36 \cdot 10^{-2} \pm 9.8 \cdot 10^{-4}$ \\
\hline
$\text{std}_{\text{IC}}$:  Training ICs & $1.9 \cdot 10^{-2} \pm 1.5 \cdot 10^{-4}$ \\
\hline   
\hline         
$\text{mean}_{\text{IC}}$: Random ICs   & $2.40 \cdot 10^{-2} \pm 8.1 \cdot 10^{-4}$ \\
\hline
$\text{std}_{\text{IC}}$:  Random ICs   & $2.3 \cdot 10^{-3} \pm 6.1 \cdot 10^{-3}$ \\
\hline   
\end{tabular}}
\caption{(PS$1$ on  $\mathbb{S}^2$ ) trajectory estimation errors: Initial Conditions (ICs) used in the training set (first two rows), new ICs randomly drawn from $\muX$ (second set of two rows).  $\text{mean}_{\text{IC}}$ and $\text{std}_{\text{IC}}$ are the mean and standard deviation of the trajectory errors calculated using \eqref{eq:traj_norm SI}.}
\label{tab:PS1_on_S2_traj_err}
\end{table}
We also report the condition number and the smallest eigenvalue of the learning matrix $A$ to indirectly verify the geometric coercivity condition in table \ref{tab:PS1_on_PD_coer}.
\begin{table}[H]
\centering
\small{\begin{tabular}{ c || c } 
\hline
Condition Number for $A_1$    & $2.2 \cdot 10^{7} \pm 1.8 \cdot 10^{6}$ \\ 
\hline
Smallest Eigenvalue for $A_1$ & $1.28 \cdot 10^{-8} \pm 8.5 \cdot 10^{-10}$ \\
\hline
\hline
Condition Number for $A_2$    & $2.9 \cdot 10^{5} \pm 2.2 \cdot 10^{5}$ \\ 
\hline
Smallest Eigenvalue for $A_2$ & $9 \cdot 10^{-7} \pm 5.7 \cdot 10^{-7}$ \\
\hline
\end{tabular}}
\caption{(PS$1$ on  $\mathbb{S}^2$ ) Information from the learning matrix $A_{\idxcl}$'s.}
\label{tab:PS1_on_S2_coer}
\end{table}
The matrix $A_1$ is used to obtain the estimators, $\lintkernel_{1, 1}$ and $\lintkernel_{1, 2}$; whereas $A_2$ is used to obtain $\lintkernel_{2, 1}$ and $\lintkernel_{2, 2}$.  Since there is one single predator, we set $\lintkernel_{2, 2}$ to zero.  It took $9.77 \cdot 10^{4}$ seconds to generate $\rho_{T, \mM}^L$ and $4.01 \cdot 10^{5}$ seconds to run $10$ learning simulations, with $1.66 \cdot 10^{3}$ seconds spent on learning the estimated interactions (on average, it took $1.66 \cdot 10^{2} \pm 4.6$ seconds to run one estimation), and $4.05 \cdot 10^{5}$ seconds spent on computing the trajectory error estimates (on average, it took $4.0 \cdot 10^{4} \pm 7.1 \cdot 10^{3}$ seconds to run one set of trajectory error estimation).

\textbf{Results for the  $ \mathbb{PD} $  case:} In order to produce more interesting interactions, we choose the distribution of the initial condition to be as follows: the predator is randomly placed in a circle centered at the origin with radius $r_1$, given as follows
\[
r_0 = \bigg(2 + \frac{1}{\cosh(0.5) - 1} - \sqrt{\frac{4}{\cosh(0.5) - 1} + \frac{1}{(\cosh(0.5) - 1)^2}}\bigg)/2,
\]
so that the agents are at most $0.5$ distance away from each other; then the group of preys (Swarm) will be randomly and uniformly placed on an annulus centered at the origin with radii,$(R_1, r_1)$, given as follows
\[
r_1 = \bigg(2 + \frac{1}{\cosh(1) - 1} - \sqrt{\frac{4}{\cosh(1) - 1} + \frac{1}{(\cosh(1) - 1)^2}}\bigg)/2
\]
and
\[
R_1 = \bigg(2 + \frac{1}{\cosh(2) - 1} - \sqrt{\frac{4}{\cosh(2) - 1} + \frac{1}{(\cosh(2) - 1)^2}}\bigg)/2;
\]
so that the group of preys are surrounding the single predator.  Table \ref{tab:PS1_ns_PD} shows the number of basis functions, namely $n_{\idxcl\idxcl'}$'s, for each estimator $\lintkernel_{\idxcl\idxcl'}$ for $\idxcl, \idxcl' = 1, 2$, and their corresponding degrees, $p_{\idxcl, \idxcl'}$'s, for the Clamped B-spline basis.
\begin{table}[H]
\centering
\small{
\small{\begin{tabular}{ c | c | c | c}
$n_{1, 1}$ & $n_{1, 2}$ & $n_{2, 1}$ & $n_{2, 2}$ \\
\hline
$68$       & $43$       & $43$       & $1$ \\
\hline
$p_{1, 1}$ & $p_{1, 2}$ & $p_{2, 1}$ & $p_{2, 2}$ \\
\hline
$1$        & $1$        & $1$        & $0$ \\
\hline
\end{tabular}}  
}
\caption{(PS$1$ on  $ \mathbb{PD} $ ) Number of basis functions.}
\label{tab:PS1_ns_PD} 
\end{table}
Fig. \ref{fig:PS1_on_PD_phiE} shows the comparison between $\intkernel_{\idxcl\idxcl'}$ and its estimators $\lintkernel_{\idxcl\idxcl'}$ learned from the trajectory data.
\begin{figure}[H]  
\begin{subfigure}{\textwidth}
  \centering
  \includegraphics[width=0.8\textwidth]{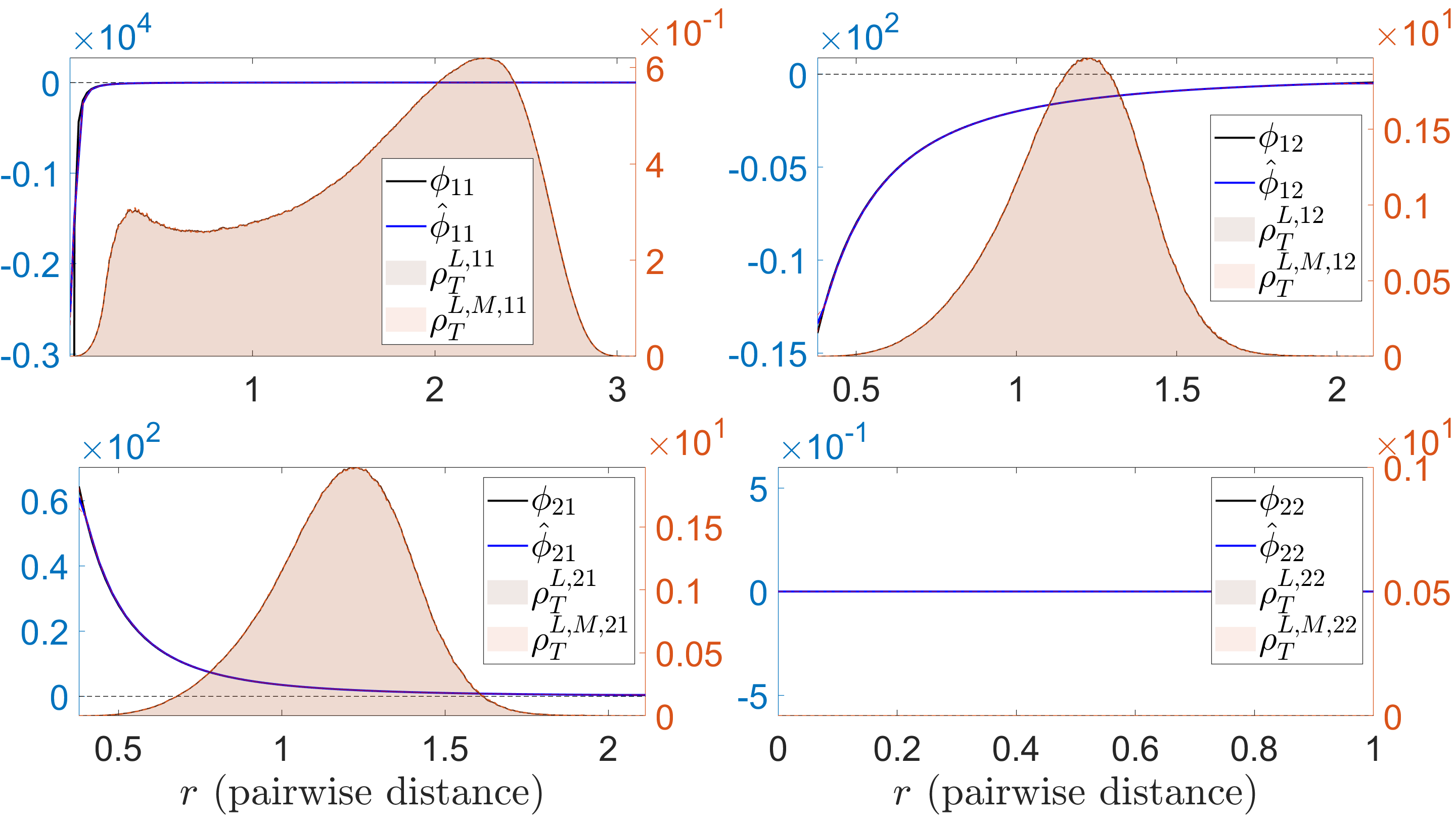}
\end{subfigure}
\caption{(PS$1$ on  $ \mathbb{PD} $ ) Comparison of $\intkernel_{\idxcl\idxcl'}$ and $\lintkernel_{\idxcl, \idxcl'}$, with the relative errors shown in table \ref{tab:PS1_errs_PD}. The true interaction kernels are shown in black solid line, whereas the mean estimated interaction kernel are shown in blue solid line with their corresponding confidence intervals shown in blue dotted lines.  Shown in the background is the comparison of the approximate $\rho_T^{L, \idxcl\idxcl'}$ versus the empirical $\rho_T^{L, M, \idxcl\idxcl'}$. Notice that $\rho_T^{L, 12}$/$\rho_T^{L, M, 12}$ and $\rho_T^{L, 12}$/$\rho_T^{L, M, 21}$ are the same distributions.}
\label{fig:PS1_on_PD_phiE}
\end{figure}
 \begin{table}[H]
\centering
\small{
\small{\begin{tabular}{ c | c | c | c}
$\text{Err}_{1, 1}$                       & $\text{Err}_{1, 2}$                       & $\text{Err}_{2, 1}$                       & $\text{Err}_{2, 2}$ \\
\hline
$9.0 \cdot 10^{-2} \pm 2.6 \cdot 10^{-3}$ & $1.34 \cdot 10^{-3} \pm 8.8 \cdot 10^{-5}$ & $3.6 \cdot 10^{-3} \pm 2.4 \cdot 10^{-4}$ & $0$ \\
\end{tabular}}  
}
\caption{(PS$1$ on  $ \mathbb{PD} $ ) Relative estimation errors calculated using \eqref{eq:rel_L2rhoT_error SI}.}
\label{tab:PS1_errs_PD} 
\end{table}
Fig. \ref{fig:PS1_on_PD_trajs} shows the comparison of the trajectory data between the true dynamics and estimated dynamics.
\begin{figure}[H]  
\begin{subfigure}{\textwidth}
  \centering
  \includegraphics[width=0.8\textwidth]{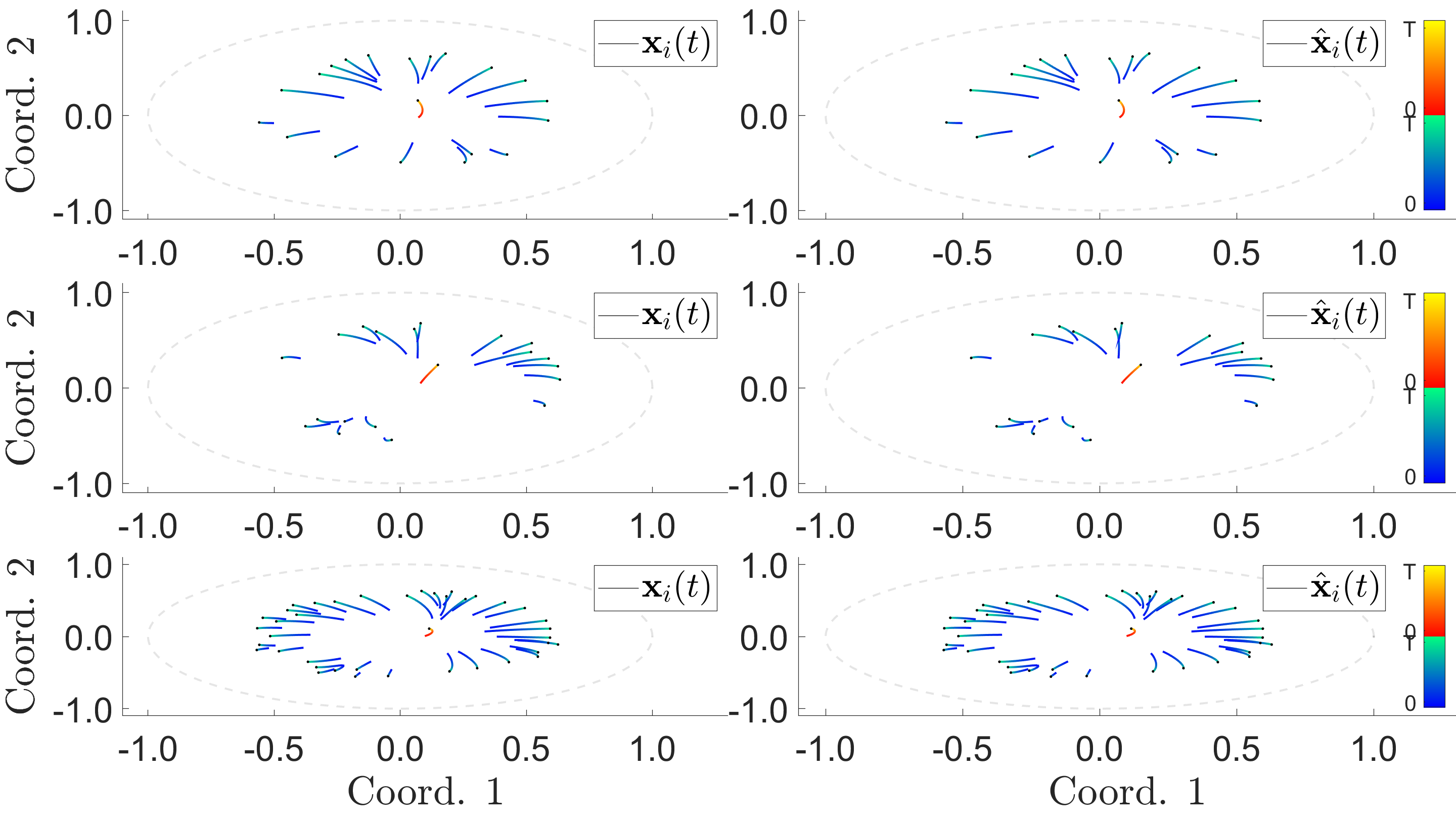} 
\end{subfigure}
\caption{(PS$1$ on  $ \mathbb{PD} $ ) Comparison of $\bX$ (generated by $\intkernel_{\idxcl, \idxcl'}$'s) and $\hat\bX$ (generated by $\lintkernel_{\idxcl, \idxcl'}$'s), with the errors reported in table \ref{tab:PS1_on_PD_traj_err}.  \textbf{Top}: $\bX$ and $\hat\bX$ are generated from an initial condition taken from the training data.  \textbf{Middle}: $\bX$ and $\hat\bX$ are generated from a randomly chosen initial condition.  \textbf{Bottom}: $\bX$ and $\hat\bX$ are generated from a new initial condition with bigger $N = 40$.  The color of the trajectory indicates the flow of time, from deep blue/bright red (at $t = 0$) to light green/light yellow (at $t = T$).  The blue/green combination is assigned to the preys; whereas the red/yellow comb for the predator.}
\label{fig:PS1_on_PD_trajs}
\end{figure}
A quantitative comparison of the trajectory estimation errors is shown in Table \ref{tab:PS1_on_PD_traj_err}.
\begin{table}[H]
\centering
\small{\begin{tabular}{| c || c |} 
\hline
                                        & $[0, T]$                                \\
\hline
$\text{mean}_{\text{IC}}$: Training ICs & $4.8 \cdot 10^{-3} \pm 1.2 \cdot 10^{-4}$ \\
\hline
$\text{std}_{\text{IC}}$:  Training ICs & $2.3 \cdot 10^{-3} \pm 3.0 \cdot 10^{-4}$ \\
\hline   
\hline         
$\text{mean}_{\text{IC}}$: Random ICs   & $4.8 \cdot 10^{-3} \pm 1.2 \cdot 10^{-4}$ \\
\hline
$\text{std}_{\text{IC}}$:  Random ICs   & $2.5 \cdot 10^{-3} \pm 3.9 \cdot 10^{-3}$ \\
\hline   
\end{tabular}}
\caption{(PS$1$ on  $ \mathbb{PD} $ ) trajectory estimation errors: Initial Conditions (ICs) used in the training set (first two rows), new ICs randomly drawn from $\muX$ (second set of two rows).  $\text{mean}_{\text{IC}}$ and $\text{std}_{\text{IC}}$ are the mean and standard deviation of the trajectory errors calculated using \eqref{eq:traj_norm SI}.}
\label{tab:PS1_on_PD_traj_err}
\end{table}
We also report the condition number and the smallest eigenvalue of the learning matrix $A$ to indirectly verify the geometric coercivity condition in table \ref{tab:PS1_on_PD_coer}.
\begin{table}[H]
\centering
\small{\begin{tabular}{ c || c } 
\hline
Condition Number for $A_1$    & $2.3 \cdot 10^{9} \pm 4.7 \cdot 10^{8}$ \\ 
\hline
Smallest Eigenvalue for $A_1$ & $7 \cdot 10^{-11} \pm 1.7 \cdot 10^{-11}$ \\
\hline
\hline
Condition Number for $A_2$    & $5 \cdot 10^{5} \pm 3.1 \cdot 10^{5}$ \\ 
\hline
Smallest Eigenvalue for $A_2$ & $4 \cdot 10^{-8} \pm 2.9 \cdot 10^{-8}$ \\
\hline
\end{tabular}}
\caption{(PS$1$ on  $ \mathbb{PD} $ ) Information from the learning matrix $A_{\idxcl}$'s.}
\label{tab:PS1_on_PD_coer}
\end{table}
The matrix $A_1$ is used to obtain the estimators, $\lintkernel_{1, 1}$ and $\lintkernel_{1, 2}$; whereas $A_2$ is used to obtain $\lintkernel_{2, 1}$ and $\lintkernel_{2, 2}$.  Since there is one single predator, we set $\lintkernel_{2, 2}$ to zero.  It took $7.37 \cdot 10^{4}$ seconds to generate $\rho_{T, \mM}^L$ and $2.49 \cdot 10^{5}$ seconds to run $10$ learning simulations, with $1.25 \cdot 10^{3}$ seconds spent on learning the estimated interactions (on average, it took $1.25 \cdot 10^{2} \pm 1.5$ seconds to run one estimation), and $2.48 \cdot 10^{5}$ seconds spent on computing the trajectory error estimates (on average, it took $2.48 \cdot 10^{4} \pm 2.3 \cdot 10^{2}$ seconds to run one set of trajectory error estimation).
\bibliography{ref}
\bibliographystyle{siam}
\end{document}